\documentclass[a4paper, 11pt]{article}
\usepackage{latexsym,amsfonts, amsmath, amssymb, amsthm}
\usepackage[utf8]{inputenc}
\usepackage[final=true,protrusion=true,expansion=true]{microtype}
\usepackage[noadjust]{cite}
\usepackage{enumitem}
\usepackage{color}
\usepackage{mathtools}

\usepackage{booktabs}
\usepackage{multirow}

\usepackage{subcaption}
\usepackage[labelfont=bf]{caption}
\usepackage[affil-it]{authblk}

\usepackage{geometry}
\usepackage{hyperref}
\hypersetup{
	final,
    colorlinks=true,
    linkcolor=blue,
    filecolor=magenta,
    urlcolor=cyan,
}

\theoremstyle{plain}
\newtheorem{theorem}{Theorem}[section]
\newtheorem{proposition}[theorem]{Proposition}
\newtheorem{lemma}[theorem]{Lemma}
\newtheorem{corollary}[theorem]{Corollary}

\newtheorem{assumption}{Assumption}

\theoremstyle{plain}
\newtheorem{definition}[theorem]{Definition}

\theoremstyle{definition}
\newtheorem{remark}[theorem]{Remark}

\usepackage{latexsym,amsfonts, amssymb}

\usepackage[utf8]{inputenc} 
\usepackage[T1]{fontenc}    
\usepackage{hyperref}       
\usepackage{url}            
\usepackage{booktabs}       
\usepackage{nicefrac}       
\usepackage{microtype}      
\usepackage{xcolor}         

\usepackage{bbm}
\usepackage[utf8]{inputenc}
\usepackage{enumitem}
\usepackage{color}
\usepackage{mathtools}
\usepackage{subcaption}
\usepackage{apptools}
\usepackage{footmisc}
\usepackage{wrapfig}
\usepackage{tikz}
\usepackage{esint}

\hypersetup{
	final,
    colorlinks=true,
    linkcolor=blue,
    filecolor=magenta,
    urlcolor=cyan,
}

\newcommand{\dummy}{\mathord{\color{black!33}\bullet}}%

\providecommand{\bbN}{\mathbb{N}}
\providecommand{\bbR}{\mathbb{R}}

\providecommand{\bbE}{\mathbb{E}}
\providecommand{\bbP}{\mathbb{P}}

\providecommand{\CE}{\mathcal{E}}
\providecommand{\CN}{\mathcal{N}}

\providecommand{\CC}{\mathcal{C}}
\providecommand{\CO}{\mathcal{O}}
\providecommand{\CP}{\mathcal{P}}
\providecommand{\CM}{\mathcal{M}}

\providecommand{\CF}{\mathcal{F}}

\providecommand{\CF}{\mathcal{F}}

\providecommand{\CI}{\mathcal{I}}
\providecommand{\CO}{\mathcal{O}}

\providecommand{\Id}{\mathrm{Id}}
\providecommand{\argmin}{\operatorname*{\arg\min}}
\providecommand{\supp}{\operatorname{supp}}

\providecommand*{\abs}[1]{\left|{#1}\right|} 
\providecommand*{\N}[1]{\left\|{#1}\right\|} 
\providecommand*{\Nnormal}[1]{\|{#1}\|} 
\providecommand*{\Nbig}[1]{\big\|{#1}\big\|} 

\newcommand{\overbar}[1]{\makebox[0pt]{$\phantom{#1}\mkern 1.5mu\overline{\mkern-1.5mu\phantom{#1}\mkern-1.5mu}\mkern 1.5mu$}#1}
\renewcommand{\underbar}[1]{\makebox[0pt]{$\phantom{#1}\mkern 1.5mu\underline{\mkern-1.5mu\phantom{#1}\mkern-1.5mu}\mkern 1.5mu$}#1}

\newcommand{\overbarscript}[1]{\mkern 1.5mu\overline{\mkern-1.5mu#1\mkern-1.5mu}\mkern 0mu}
\newcommand{\underbarscript}[1]{\mkern 1.5mu\underline{\mkern-1.5mu#1\mkern-1.5mu}\mkern 1.5mu}

\newcommand{\divergence}{\textrm{div}}

\newcommand{\globmin}{{x^*}}
\newcommand{\minobj}{\underbar \CE}

\newcommand{\omegaa}[0]{\omega_{\alpha}^\CE}

\newcommand{\conspoint}[1]{x_{\alpha}^\CE({#1})}
\newcommand{\conspointE}[2]{x^{{#1}}_{\alpha}({#2})}
\newcommand{\conspointnoarg}{x_{\alpha}^\CE}
\newcommand{\conspointnoargE}[1]{x_{\alpha}^{#1}}

\newcommand{\empmeasure}[1]{\widehat\rho_{#1}^N}

\newcommand{\empmeasuretilde}[1]{\widetilde\rho_{#1}^N}

\newcommand{\indivmeasure}[0]{\varrho} 

\usepackage{amsopn}

\makeatother

\title{\usefont{OT1}{bch}{b}{n}
	{\huge Gradient is All You Need?}\\ \Large How Consensus-Based Optimization can be Interpreted as a Stochastic Relaxation of Gradient Descent \\
}

\date{}

\author[1]{Konstantin Riedl\thanks{Email: \texttt{konstantin.riedl@maths.ox.ac.uk}}}
\affil[1]{University of Oxford, Mathematical Institute}
\author[2]{Timo Klock\thanks{Email: \texttt{tmklock@googlemail.com}}}
\affil[2]{Deeptech Consulting}
\author[3]{Carina Geldhauser\thanks{Email: \texttt{carina.geldhauser@math.ethz.ch}}}
\affil[3]{ETH Zurich, Department of Mathematics}
\author[4,5,6]{Massimo Fornasier\thanks{Email: \texttt{massimo.fornasier@ma.tum.de}}}
\affil[4]{Technical University of Munich, School of Computation, Information and Technology, Department of Mathematics}
\affil[5]{Munich Center for Machine Learning}
\affil[6]{Munich Data Science Institute}

\begin{document}
\maketitle
\vspace{-1em}
\begin{abstract}
\noindent
	In this paper, we provide a novel analytical perspective on the theoretical understanding of gradient-based learning algorithms by interpreting consensus-based optimization~(CBO), a recently proposed multi-particle derivative-free optimization method, as a stochastic relaxation of gradient descent. 
	Remarkably, we observe that through communication of the particles, CBO exhibits a stochastic gradient descent~(SGD)-like behavior
	despite solely relying on evaluations of the objective function.
    The fundamental value of such link between CBO and SGD lies in the fact that CBO is provably globally convergent to global minimizers for ample classes of nonsmooth and nonconvex objective functions.
    Hence, on the one side, we offer a novel explanation for the success of stochastic relaxations of gradient descent by furnishing useful and precise insights that explain how problem-tailored stochastic perturbations of gradient descent (like the ones induced by CBO) overcome energy barriers and reach deep levels of nonconvex functions.
    On the other side, and contrary to the conventional wisdom for which derivative-free methods ought to be inefficient or not to possess generalization abilities, our results unveil an intrinsic gradient descent nature of heuristics. 
	Instructive numerical illustrations support the theoretical insights.
\end{abstract}



\section{Introduction}

Gradient-based learning algorithms,
such as stochastic gradient descent~(SGD), AdaGrad~\cite{duchi2011adagrad}, \mbox{RMSProp} and Adam~\cite{adam2015}, just to name a few of the most known and advocated,
have undoubtedly been one of the cornerstones of the astounding successes of machine learning~\cite{collobert2008unified,graves2013speech,krizhevsky2017imagenet} in the last decades.
In particular, the efficient computation of gradients through backpropagation~\cite{rumelhart1986learning} and automatic differentiation~\cite{baydin2018automatic} has allowed practitioners to leverage nowadays enormous amounts of data to train huge models~\cite{lecun2015deep}.
Despite an ever-growing relevance of advancing our mathematical understanding concerning the behavior of gradient-based learning algorithms when employed to train neural networks,
the fundamental reasons behind their empirical successes largely remain elusive~\cite{zhang2021understanding} and defy our theoretical understanding~\cite{mei2018mean}.
Yet, over the last years, several studies have started shedding light on the peculiarities of neural network loss functions as well as the training dynamics of SGD and its variants, see, e.g., \cite{mei2018mean,chizat2018global,rotskoff2018trainability,sirignano2020mean,choromanska2015loss,soudry2016no,kawaguchi2016deep,nguyen2017loss,safran2018spurious,soltanolkotabi2018theoretical,du2018power,du2019gradient,oymak2019overparameterized,fehrman2020convergence} and references therein.
While shallow neural networks are prone to spurious local minima~\cite{safran2018spurious}, which render the optimization NP-hard in general~\cite{blum1988training}, overparameterization is widely believed to be responsible for well-behaved loss function landscapes, allowing gradient-based learning algorithms to find parameters that generalize for a variety of architectures~\cite{choromanska2015loss,kawaguchi2016deep,soudry2016no,nguyen2017loss,du2018power,soltanolkotabi2018theoretical,du2019gradient,arora2019implicit}.

In this work, we consider the more generic, ubiquitous problem of finding a global minimizer of a potentially nonsmooth and nonconvex objective function~$\CE:\bbR^d\rightarrow\bbR$, i.e., solving
$\globmin\in \argmin_{x\in\bbR^d} \CE(x)$.
Supported by illustrative numerical experiments (see Figure~\ref{fig:intuitionGiAyN}), we provide a novel analytical perspective on gradient-based learning algorithms for general global optimization problems. Specifically, we interpret the recently proposed multi-particle metaheuristic derivative-free optimization method, known as consensus-based optimization (CBO)~\cite{pinnau2017consensus}, as a stochastic relaxation of gradient descent (GD). 
The formulation of CBO is recalled below in Equation \eqref{eq:CBO_dynamics}.
The main result of this paper that formally clarifies the stochastic approximation of GD by CBO  is stated in Theorem~\ref{thm:main_informal}.
The key benefit of and the key motivation for establishing a link between CBO and (S)GD lies in the proven ability of CBO~\cite{carrillo2018analytical,carrillo2019consensus,ha2021convergence,fornasier2020consensus_sphere_convergence,fornasier2021consensus,fornasier2021convergence,fornasier2023consensus,riedl2024perspective} (see Section~\ref{sec:CBO_convergence_results} for a review of \cite{fornasier2021consensus,fornasier2021convergence,riedl2024perspective}) to achieve global convergence to global minimizers for broad classes of nonsmooth and nonconvex objective functions,
which includes in particular the setting of high-dimensional problems coming from data analysis and signal processing~\cite{fornasier2020consensus_sphere_convergence,riedl2022leveraging}, as well as machine learning~\cite{carrillo2019consensus,fornasier2021convergence,fornasier2020consensus_sphere_convergence,riedl2022leveraging,trillos2023FedCBO,trillos2024CB2O,trillos2024attack}.%
\footnote{\cite[Section~2.4]{fornasier2020consensus_sphere_convergence} applies CBO for a phase retrieval problem, robust subspace detection, and the robust computation of eigenfaces; \cite[Section~4.4]{riedl2022leveraging} solves a compressed sensing task; \cite[Section~4.3]{carrillo2019consensus}, \cite[Section~4]{fornasier2021convergence}, and \cite[Section~4.3]{riedl2022leveraging} train neural networks for image classification; \cite[Section~5.3]{trillos2024CB2O} tackles a sparse representation learning problem; \cite[Section~2.4]{trillos2023FedCBO} and \cite[Section~3]{trillos2024attack} devise FedCBO and FedCB$^2$O, respectively, to solve clustered federated learning problems while ensuring maximal data privacy in both attack-free and adversarial environments; \cite[Section~IV]{CBOrobotics} uses CBO to design optimal trajectories and policies
for robotic systems.\label{footnote:CBO_applications}}

This previously unexplored connection between mathematically explainable derivative-free optimization methods and gradient-based learning algorithms provides, on the one hand, a novel and complementary perspective on the success of stochastic relaxations of GD, and, on the other hand, unveils the intrinsic GD-like nature of heuristic methods.

\paragraph{Contributions}


To our knowledge, for the first time in the literature of CBO and related multi-particle-based heuristics, such as particle swarm optimization \cite{kirkpatrick1983optimization,grassi2021particle}, we demonstrate that, under appropriate parameter scalings, CBO\,---\,despite being a derivative-free (zero-order) optimization method\,---\,naturally approximates stochastic gradient flow dynamics and thus implicitly behaves like a gradient-based (first-order) method (see Theorem~\ref{thm:main_informal} and Figures~\ref{fig:intuitionGiAyN} and~\ref{fig:approximation}). 
To establish this connection, we employ a fully nonsmooth analysis that combines a recently developed quantitative version of the Laplace principle~\cite{fornasier2021consensus} (log-sum-exp trick) with the minimizing movement scheme~\cite{de1993new} (proximal iteration~\cite{parikh2014proximal}), a well-known tool from gradient flow theory~\cite{santambrogio2017euclideanmetricandwassersteingradientflows}. 
Our results shed light on the {\it nonlocal} mechanisms through which stochastic perturbations in GD overcome energy barriers, unlocking deeper levels of nonconvex objective functions and enabling global optimization. To the best of our knowledge, this insight is unprecedented in the literature, which has traditionally focused on interpreting (S)GD dynamics solely from a {\it local} perspective.
Moreover, while the standard global analysis of (stochastic) GD typically requires the loss function to be $L$-smooth and to satisfy the Polyak-{\L}ojasiewicz condition, the global convergence of CBO only necessitates local Lipschitz continuity and a specific growth condition near the global minimizer~\cite{fornasier2021consensus,fornasier2021convergence}.
By establishing such a link between stochastic GD on the one hand and metaheuristic black-box optimization algorithms such as CBO on the other,
we not just allow for complementing our theoretical understanding of successfully deployed optimization algorithms in machine learning and beyond,
but we also widen the scope of applications of methods which\,---\,in one way or another, be it explicitly or implicitly\,---\,estimate and exploit gradients.

\paragraph{Organization}
Section~\ref{sec:Assumptions} summarizes the assumptions under which the results of this work are valid.
In Section~\ref{sec:CBO}, after introducing CBO and describing the mechanisms behind its functioning, we present, discuss, and numerically verify the main theoretical results of this work.
Section~\ref{sec:CBO_convergence_results} recapitulates state-of-the-art global convergence results for CBO in the setting of potentially nonsmooth and nonconvex objective functions~$\CE$.
Section~\ref{sec:CBO_stochastic_relaxation_GD} is dedicated to presenting the technical details behind the main theoretical findings of this work.
We first sketch how to interpret CBO as a stochastic relaxation of GD by introducing what we call the {\it consensus hopping scheme},
which interconnects by sampling the derivative-free with the gradient-based approach to optimization.
It further highlights a connection between sampling and optimization.
Afterwards, the proof of our main result, Theorem~\ref{thm:main_informal}, is provided in Section~\ref{sec:proof:thm:main_informal} with the central technical tools being collected in Section~\ref{sec:appendix:proof_details}.
Section~\ref{sec:conclusions} eventually concludes the paper by discussing future perspectives.
In the 
\href{https://github.com/KonstantinRiedl/CBOGlobalConvergenceAnalysis}{GitHub repository}
we provide the implementation of the algorithms analyzed in this work and the code used to create the visualizations.
Python and Julia code for CBO are available in the \href{https://github.com/PdIPS}{GitHub repository}~\cite{bailo2024cbx}.

\paragraph{Notation}
We write $\CC(X)$ and $\CC^k(X)$ for the spaces of continuous and $k$-times continuously differentiable functions~$f:X\rightarrow \bbR$, respectively.
With $\nabla f$ we denote the gradient of a differentiable function~$f$.
$\CP(\bbR^d)$, respectively $\CP_p(\bbR^d)$, is the set containing all probability measures over $\bbR^d$ (with finite $p$-th moment).
$\CP_p(\bbR^d)$ is metrized by the \mbox{Wasserstein-$p$} distance~$W_p$, see, e.g., \cite{savare2008gradientflows,villani20090oldandnew}.
$\CN(m,\Sigma)$ denotes a Gaussian distribution with mean $m$ and covariance matrix~$\Sigma$.

\section{Characterization of the class of objective functions}
    \label{sec:Assumptions}

The theoretical findings of this work hold for objectives satisfying the following conditions,
which are complementary to classical assumptions under which the convergence of stochastic relaxations of gradient descent has been studied, see, e.g., \cite{karimi2016linear,chizat2018global}.

\begin{assumption}
    \label{asm:objective}
	Throughout we consider objective functions $\CE \in \CC(\bbR^d)$, 
	\begin{enumerate}[label=A\arabic*,labelsep=10pt,leftmargin=35pt]
		\item\label{asm:zero_global} for which there exists $\globmin\in\bbR^d$ such that $\CE(\globmin)=\inf_{x\in\bbR^d} \CE(x) =: \underbar\CE$,
        \item\label{asm:local_Lipschitz_quadratic_growth} for which there exist $C_1, C_2 > 0$ such that
		\begin{align}
			\abs{\CE(x)-\CE(x')}
			&\leq C_1(1 + \N{x}_2 + \Nnormal{x'}_2)\Nnormal{x-x'}_2
			\quad \text{ for all } x, x' \in \bbR^d, \label{asm:local_Lipschitz_quadratic_growth:1}\\
			\abs{\CE(x)-\underbar\CE}
			&\leq C_2(1 + \N{x}_2^2)
			\quad \text{ for all } x \in \bbR^d, \label{asm:local_Lipschitz_quadratic_growth:2}
		\end{align}
		\item\label{asm:quadratic_growth} for which either $\overbar\CE := \sup_{x \in \bbR^d}\CE(x) < \infty$, or for which there exist $C_3,C_4 > 0$ such that
		\begin{align} \label{asm:quadratic_growth:2}
			\CE(x) - \minobj
			\geq C_3\N{x}_2^2
			\quad \text{ for all } x \in \bbR^d \text{ with } \N{x}_2 \geq C_4,
		\end{align} 
	\item\label{asm:lambda-convex} which are semi-convex ($\Lambda$-convex for some $\Lambda\in\bbR$), i.e., $\CE(\!\;\dummy\!\;)-\frac{\Lambda}{2}\N{\!\;\dummy\!\;}_2^2$ is convex.
	\end{enumerate}
\end{assumption}

Assumption~\ref{asm:zero_global} requires that the continuous objective function~$\CE$ attains its globally minimal value~$\underbar\CE$ at some $\globmin\in\bbR^d$.
This does not exclude objectives with multiple global minimizers.

\begin{remark}
    \label{rem:ICP}
    For the global convergence results~\cite{fornasier2021consensus,fornasier2021convergence} of CBO (which we recapitulate in Section~\ref{sec:CBO_convergence_results}), uniqueness of the global minimizer~$\globmin$ is required
    and implied by an additional local coercivity condition of the form $\N{x-\globmin}_\infty \leq (\CE(x)-\underbar\CE)^\nu/\eta$ for all $x\in B^\infty_{R_0}(\globmin)$ and $\CE(x)-\underbar\CE>\CE_\infty$ outside of $B^\infty_{R_0}(\globmin)$,
    where $\eta,\nu,\CE_\infty,R_0>0$ characterize the objective.
    It can be regarded as a tractability condition of the energy landscape of $\CE$
    and is also known as the inverse continuity property from~\cite{fornasier2020consensus_sphere_convergence} or as the error bound condition from \cite{anitescu2000degenerate,xu2017adaptive,bolte2017error,necoara2019linear}.

    To deploy CBO also in the setting of objective functions with several global minima,
\cite{bungert2022polarized,fornasiersun2024} propose a polarized variant of CBO, which localizes the dynamics by integrating a kernel in the computation of the consensus point~\eqref{eq:consensus_point}.
    This ensures that each particle is primarily influenced by particles close to it, allowing for the creation of clusters. We do not explore this more general setting in the present paper.
\end{remark}

Assumptions~\ref{asm:local_Lipschitz_quadratic_growth} and~\ref{asm:quadratic_growth} can be regarded as regularity conditions on the objective landscape of $\CE$.
The first part of \ref{asm:local_Lipschitz_quadratic_growth} is a local Lipschitz condition, which ensures that the objective function does not change too quickly, assuring that the information obtained when evaluating the function is informative within a region around the point of evaluation.
The second part of \ref{asm:local_Lipschitz_quadratic_growth} controls and limits the growth of the objective in the farfield.
In combination with the second option in \ref{asm:quadratic_growth} this forces the objective to grow quadratically in the farfield.
However, one can always redefine the objective outside a sufficiently large ball such that both conditions are met while the other assumptions are preserved.
Alternatively, the first option in \ref{asm:quadratic_growth} allows for bounded functions.
\ref{asm:local_Lipschitz_quadratic_growth} and~\ref{asm:quadratic_growth} are necessary for well-posedness of the CBO dynamics.

Assumption~\ref{asm:lambda-convex} requires the objective~$\CE$ to be semi-convex with parameter $\Lambda\in\bbR$.
For $\Lambda>0$, $\Lambda$-convexity is stronger than convexity (strong convexity with parameter~$\Lambda$).
For $\Lambda<0$, semi-convexity is weaker, i.e., potentially nonconvex functions~$\CE$ are included in the definition.
In particular, on a bounded set, all
smooth ($\CC^2$ is sufficient) functions are $\Lambda$-convex for a suitable $\Lambda<0$.
The semi-convexity assumption of objective functions is quite standard in the literature of gradient flows, since their general theory extends from the convex to this more general setting~\cite{savare2008gradientflows,santambrogio2017euclideanmetricandwassersteingradientflows}, which covers many of the relevant cases.
One useful property, which we shall exploit in this work, is that for semi-convex functions the time discretization of a gradient flow, potentially for a small step size, defined through an iterated scheme, the so-called {\it minimizing movement scheme}~\cite{de1993new}, is well-defined.
However, while semi-convexity is useful to ensure the well-posedness of gradient flows, it is not sufficient to obtain convergence to global minimizers.
Other properties such as the Polyak-{\L}ojasiewicz condition~\cite{karimi2016linear} or the log-Sobolev inequalities governing the flow of the Langevin dynamics~\cite{chizat2018global} may be necessary.

The class of objective functions~$\CE$ captured by Assumptions~\ref{asm:zero_global}--\ref{asm:lambda-convex} is quite broad and includes typical loss functions in  signal processing as well as machine learning.
Examples are the objectives of
lasso and ridge regression,
or empirical risk functions with for instance the least squares loss and weight decay.
Moreover, several standard benchmark functions in optimization~\cite{JamilYang2013}, such as the nonconvex Rastrigin or Ackley function, obey \ref{asm:zero_global}--\ref{asm:lambda-convex} as well.

\section{Consensus-based optimization and the main result} \label{sec:CBO}

Inspired by particle swarm optimization~(PSO)~\cite{kennedy1995particle,grassi2021particle}, CBO methods employ an interacting stochastic system of $N$ particles $X^1,\dots,X^N$ to explore the domain and to form consensus about the global minimizer~$\globmin$ over time.
More concretely, given an arbitrary finite number of time steps~$K$, a discrete time step size~$\Delta t>0$ and denoting the position of the $i$-th particle at time step~$k\in\{0,\dots,K\}$ by $X_{k}^i$, this position is computed for user-specified parameters~$\alpha,\lambda,\sigma>0$ according to the iterative update rule
\begin{align} \label{eq:CBO_dynamics}
\begin{aligned}
	X_{k}^i
	= X_{k-1}^i
    - \Delta t\lambda \left(X_{k-1}^i - \conspoint{\empmeasure{k-1}}\right) 
    + \sigma D\!\left(X_{k-1}^i-\conspoint{\empmeasure{k-1}}\right) B_{k}^i,
\end{aligned}
\end{align}
where $\empmeasure{k}$ denotes the empirical measure of the particles at time step $k$, i.e., $\empmeasure{k} = \frac{1}{N}\sum_{i=1}^N \delta_{X_k^i}$.
In the spirit of the exploration-exploitation philosophy of evolutionary computation techniques \cite{holland1992adaptation,back1997handbook,fogel2006evolutionary}, the dynamics~\eqref{eq:CBO_dynamics} of each particle is governed by two competing terms, one being stochastic, the other deterministic in nature.
The first of the two terms on the right-hand side of~\eqref{eq:CBO_dynamics} imposes a deterministic drift towards the so-called consensus point~$\conspointnoarg$, which 
is defined for a measure~$\varrho\in\CP(\bbR^d)$ by
\begin{equation}
    \label{eq:consensus_point}
	\conspoint{\varrho}
	:= \int x \frac{\omegaa(x)}{\N{\omegaa}_{L^1(\varrho)}} d\varrho(x),
\end{equation}
with $\omegaa(x) := \exp(-\alpha\CE(x))$.
Notice that in the case $\varrho=\empmeasure{k}$, Formula~\eqref{eq:consensus_point} is just a weighted (exploiting the particles' knowledge of their objective function values) convex combination of the positions~$X_k^i$.
To be precise, owed to the particular choice of Gibbs weights~$\omegaa$, larger mass is attributed to particles with comparably low objective value, whereas only little mass is given to particles whose value is undesirably high.
This facilitates the interpretation that $\conspoint{\empmeasure{k}}$ is an approximation to $\argmin_{i=1,\dots,N} \CE(X_k^i)$, which improves as $\alpha\rightarrow\infty$ and which can be regarded as a proxy for the global minimizer~$\globmin$, based on the information currently available to the particles.
Theoretically, this is justified by the log-sum-exp trick or the Laplace principle~\cite{dembo2009large,miller2006applied}.
Let us further remark that the particles communicate and exchange information amongst each other exclusively through sharing the consensus point~$\conspointnoarg$.
The other term in~\eqref{eq:CBO_dynamics} is a stochastic diffusion injecting randomness into the dynamics, thereby encoding its explorative nature.
Given i.i.d.\@ Gaussian random vectors~$B_{k}^i$ in~$\bbR^d$ with zero mean and covariance matrix $\Delta t\Id$, 
each particle is subject to anisotropic noise, i.e., $D(\,\dummy\,) = \mathrm{diag}(\,\dummy\,)$,\footnote{$\mathrm{diag}:\bbR^{d}\rightarrow\bbR^{d\times d}$ denotes the operator mapping a vector to a diagonal matrix with the vector as diagonal.} which favors exploration the farther a particle is away from the consensus point in a certain direction.
In particular, the diffusive character of the dynamics vanishes over time as consensus is reached.
The described exploration-exploitation mechanism can be seen as a multi-particle reincarnation of similar ones executed by 
simulated annealing (SA)~\cite{kirkpatrick1983optimization,geman1986diffusions,holley1988simulated} and the annealed Langevin dynamics~\cite{gelfand1991recursive}.
System~\eqref{eq:CBO_dynamics} is complemented with independent initial data~$x_{0}^i$ distributed according to a common probability measure~$\rho_0\in\CP(\bbR^d)$, i.e., $X_{0}^i = x_{0}^i \sim \rho_0$.

Hence, CBO distills fundamental principles from other popular and successful metaheuristics, in particular PSO and SA, yet it comes with two fundamental advantages compared to these algorithms.
Firstly, it outperforms such well-established methods in experiments over challenging benchmarks~\cite{grassi2021particle,grassi2021mean,huang2022global}.
Secondly, and remarkably, it comes with  theoretical guarantees of global convergence to global minimizers and the quantification of the convergence rate ~\cite{carrillo2018analytical,carrillo2019consensus,ha2021convergence,fornasier2020consensus_sphere_convergence,fornasier2021consensus,fornasier2021convergence,riedl2024perspective}.
For these reasons, it has to be considered as a baseline for understanding heuristics.

\definecolor{blue_intuition}{RGB}{0,114,189}
\definecolor{red_intuition}{RGB}{217,83,25}
\begin{figure*}[ht]
	\centering
	\subcaptionbox{A noisy Canyon function~$\CE$ with a valley shaped as a third degree polynomial. \label{fig:GrandCanyon3noisy}}{\includegraphics[trim=91 251 79 250,clip,width=0.42\textwidth]{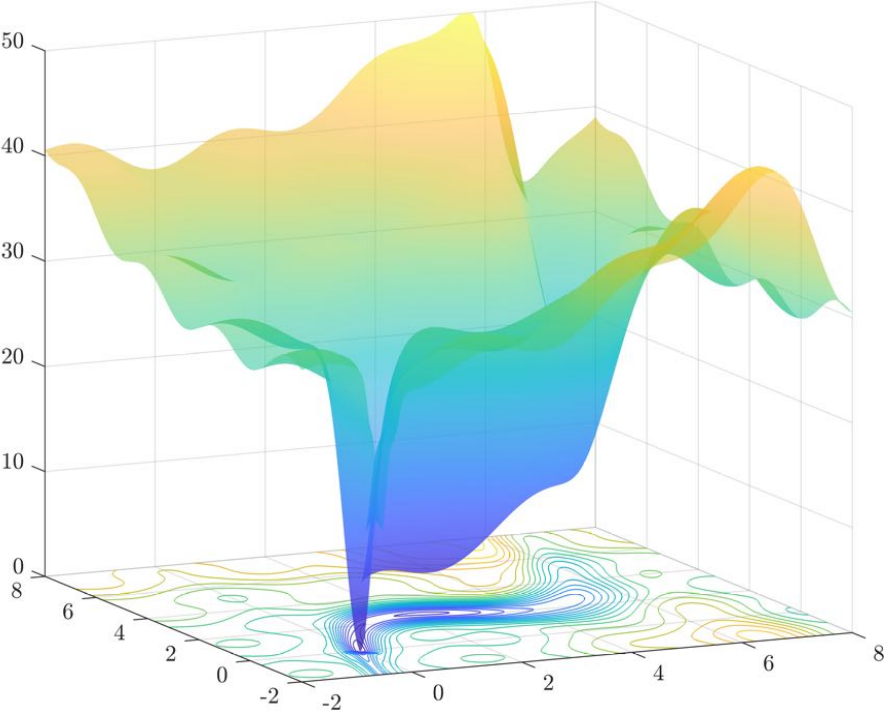}}
	\hspace{4em}
	\subcaptionbox{The CBO scheme~\eqref{eq:CBO} (sampled over several runs) follows on average the valley of $\CE$ while passing over local minima. \label{fig:CBO_GrandCanyon3noisy}}{\includegraphics[trim=28 209 31 200,clip,width=0.42\textwidth]{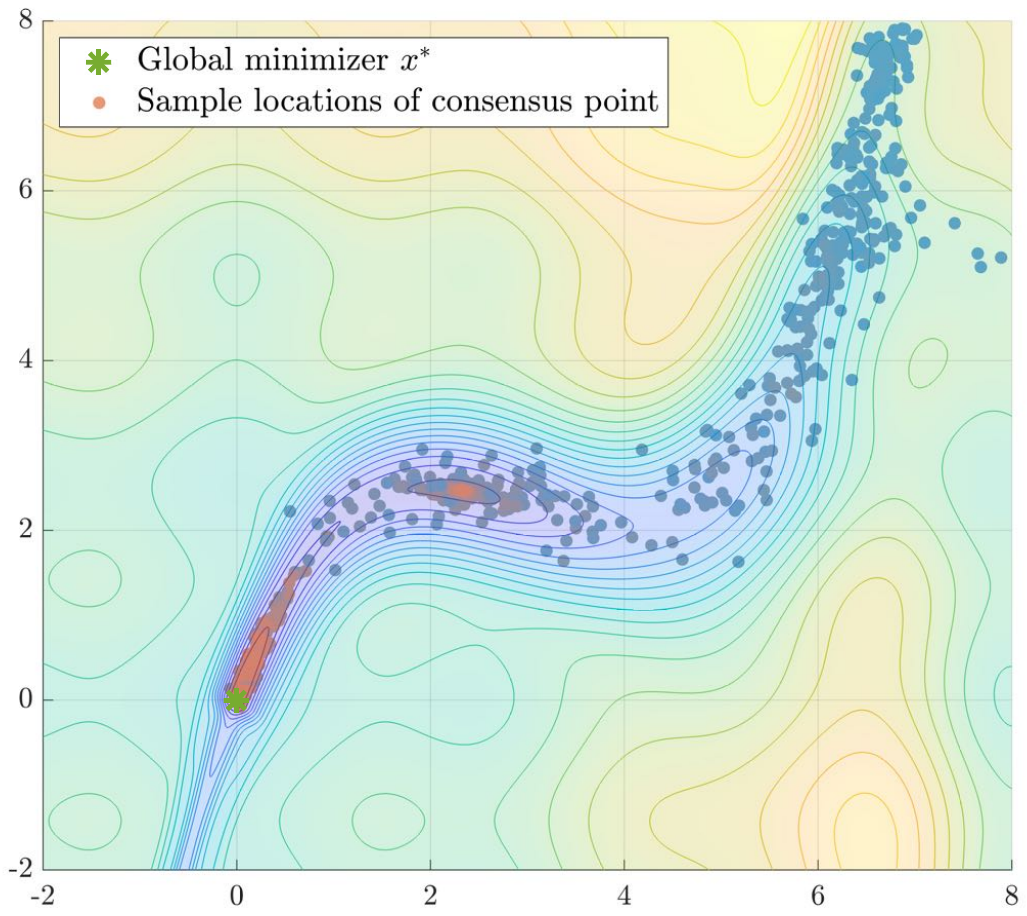}}
	\caption{An illustration of the intuition that the CBO scheme~\eqref{eq:CBO} can be regarded as a stochastic derivative-free (zero-order) relaxation of GD.
	To find the global minimizer~$\globmin$ of the nonconvex objective function~$\CE$ depicted in (a), 
	we run the CBO algorithm~\eqref{eq:CBO_dynamics} for  $K=250$ iterations with parameters $\Delta t=0.01$, $\alpha = 100$, $\lambda = 1$ and $\sigma = 1.6$, and $N=200$~particles, initialized i.i.d.\@~according to~$\rho_0 = \CN\big((8,8), 0.5\Id\big)$.
	This experiment is performed $50$ times.
	For each run we depict in (b) the positions of the consensus points computed during the CBO algorithm~\eqref{eq:CBO_dynamics}, i.e., the iterates of the CBO scheme~\eqref{eq:CBO} for $k=1,\dots,K$.
	The color of the individual points corresponds to time, i.e., iterates at the beginning of the scheme are plotted in {\color{blue_intuition}blue}, whereas later iterates are colored {\color{red_intuition}orange}.
	We observe that, after starting close to the initial position, the trajectories of the consensus points follow the path of the valley leading to the global minimizer~$\globmin$, until it is reached. 
	In particular, unlike GD (cf.\@~Figure~\ref{fig:GD_GrandCanyon3noisy}), the scheme~\eqref{eq:CBO} has the capability of jumping over locally deeper passages.
	Such desirable behavior is observed also for the Langevin dynamics (see Figure~\ref{fig:Langevin_GrandCanyon3noisy}), which can be regarded as a stochastic (noisy) version of GD.}
	\label{fig:intuitionGiAyN}
\end{figure*}
While in the recent established literature on the convergence analysis of CBO, the dynamics of the particle system \eqref{eq:CBO_dynamics} is analyzed in its ensemble nature,  a novel and insightful theoretical understanding of the behavior of CBO can be gained, as we are about to show, by tracking the dynamics of the consensus point~$\conspointnoarg$ of the CBO algorithm~\eqref{eq:CBO_dynamics}.
For this purpose, let us introduce the {\it CBO scheme} as the iterates~$(x^{\mathrm{CBO}}_{k})_{k=0,\dots,K}$ defined according to 
\begin{align}    \label{eq:CBO}
\begin{aligned}
	x^{\mathrm{CBO}}_{k}
	&= \conspoint{\empmeasure{k}},
	\quad\text{ with }\quad
    \empmeasure{k} = \frac{1}{N}\sum_{i=1}^N \delta_{X_{k}^i}, \\
	x^{\mathrm{CBO}}_{0}
	&=x_0\sim\rho_0,
\end{aligned}
\end{align}  
where the particles' positions~$X_{k}^i$ are given by Equation~\eqref{eq:CBO_dynamics}.
The main theoretical finding of this work is concerned with the observation that the iterates of the CBO scheme~\eqref{eq:CBO}, i.e., the trajectory of the consensus point~$\conspointnoarg$, follow, with high probability, a stochastically perturbed GD.
This is illustrated in Figure~\ref{fig:intuitionGiAyN} and made rigorous in the subsequent Theorem~\ref{thm:main_informal}, whose proof is deferred to Section~\ref{sec:proof:thm:main_informal}.
\begin{theorem}[CBO is a stochastic relaxation of GD (main result)]
    \label{thm:main_informal}
    Let $\CE\in\CC^{1}(\bbR^d)$ be $L$-smooth\footnote{A function $f\in\CC^{1}(\bbR^d)$ is $L$-smooth if $\N{\nabla f(x)-\nabla f(x')} _2 \leq L\N{x-x'}_2$ for all $x,x'\in \bbR^d$.} and satisfy minimal assumptions (summarized in Assumption~\ref{asm:objective} above).
    Then, for $\tau>0$ (satisfying $\tau<1/(-2\Lambda)$ if $\Lambda<0$) and with parameters~$\alpha,\lambda,\sigma,\Delta t>0$ such that $\alpha\gtrsim\frac{1}{\tau}d\log d$,
    the iterates~$(x^{\mathrm{CBO}}_{k})_{k=0,\dots,K}$ of the CBO scheme~\eqref{eq:CBO} follow a stochastically perturbed GD,
    i.e., they obey
    \begin{align} \label{eq:thm:main_informal}
        x^{\mathrm{CBO}}_{k}
        = x^{\mathrm{CBO}}_{k-1} - \tau \nabla \CE(x^{\mathrm{CBO}}_{k-1}) + g_k,
    \end{align}
    where $g_k$ is stochastic noise fulfilling for each $k=1,\dots,K$ with high probability the quantitative estimate
    $\N{g_k}_2 = \CO\big(\!\abs{\lambda-1/\Delta t} + \sigma\sqrt{\Delta t} + \sqrt{\tau/\alpha} + N^{-1/2}\big) + \CO(\tau)$
    The hidden constants depend on the parameters of the objective function~$\CE$ collected in \ref{asm:zero_global}\;\!--\,\ref{asm:quadratic_growth} and \ref{asm:zero_global}\;\!--\,\ref{asm:lambda-convex}, respectively.
\end{theorem}

Let us now comment on the technical aspects of Theorem~\ref{thm:main_informal}, describe its interpretation, and discuss its implications.
\\
Concerning the assumptions, it shall be mentioned that, compared to Polyak-{\L}ojasiewicz-like conditions~\cite{karimi2016linear} or certain families of log-Sobolev inequalities~\cite{chizat2018global} that are required to analyze the dynamics of gradient-based methods such as (S)GD or the Langevin dynamics, the assumptions under which our statement holds are rather weak and complementary.
Combined with similar assumptions being sufficient to prove global convergence of CBO (as stated in Theorem~\ref{thm:Global_CBO_convergence}), Theorem \ref{thm:main_informal} extends the class of functions, for which SGD-like methods are successful in global optimization.
\\
 Indeed, the statement of Theorem~\ref{thm:main_informal} has to be read with a twofold interpretation. 
First, in view of the capability of CBO to converge to global minimizers for rich classes of nonsmooth and nonconvex objective functions 
(see Theorem~\ref{thm:Global_CBO_convergence}), 
Theorem~\ref{thm:main_informal} states that there exist stochastic relaxations of GD that are provably able to robustly and reliably overcome energy barriers and reach deep levels of nonconvex functions.
Such relaxations may even be derivative-free and do not require smoothness of the objective, as is the case with CBO.
Second, and conversely, against the common wisdom that derivative-free optimization heuristics search the domain mainly by random exploration and therefore ought to be inefficient, we provide evidence that such heuristics in fact work successfully in finding benign optima~\cite{duchi2015optimal,nesterov2017random,chen2017zoo,nikolakakis2022black,chiang2023loss,engquist2023adaptive,heaton2023global}, because they are suitable stochastic relaxations of gradient-based methods.
\\
The interpretation of the CBO scheme~\eqref{eq:CBO} as a stochastic relaxation of GD is substantiated visually, analytically, and numerically.
While the trajectories of \eqref{eq:CBO} are to be seen in Figure~\ref{fig:CBO_GrandCanyon3noisy}, we depict for comparison in Figure~\ref{fig:Langevin_GrandCanyon3noisy} the discretized annealed Langevin dynamics~\cite{chiang1987diffusion,roberts1996exponential,durmus2017nonasymptotic}, $dX_t = -\nabla\CE(X_t)\,dt + \sqrt{2\smash[b]{\beta_t^{-1}}}\,dB_t$.
Both stochastic methods are capable of global minimization while overcoming energy barriers and escaping local minima.
For analyses of the (annealed) Langevin dynamics we refer to \cite{gelfand1991recursive,marquez1997convergence,pelletier1998weak,xu2018global,chizat2022meanfield}.

\begin{figure*}[ht]
	\centering
	\subcaptionbox{\label{fig:approximation_a}$\alpha=100$}{\includegraphics[trim=50 254 48 270,clip,width=0.42\columnwidth]{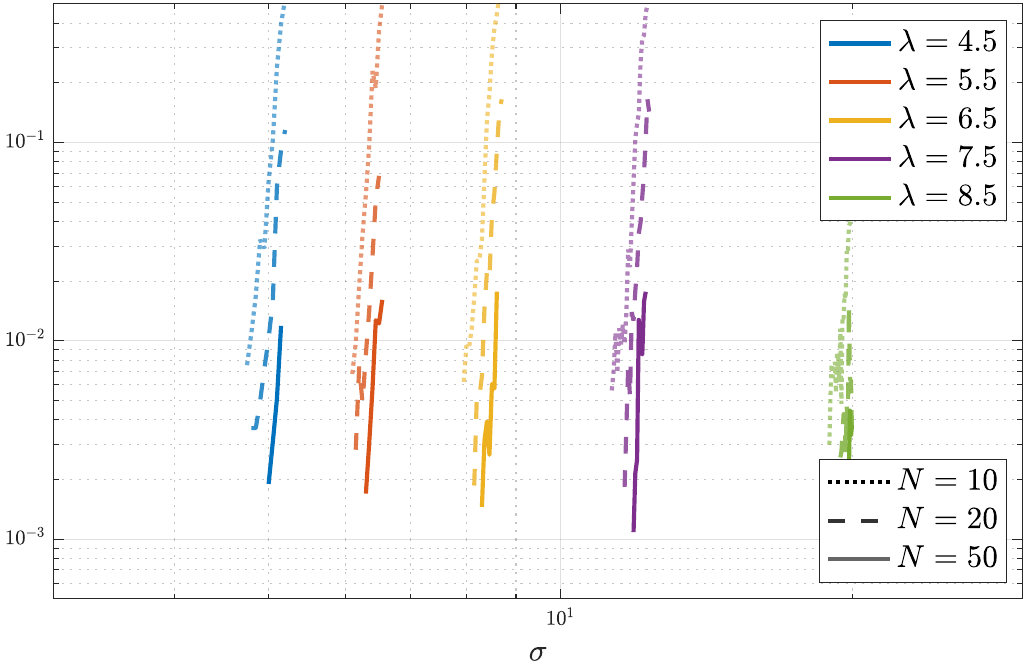}}
	\hspace{0.8em}
	\subcaptionbox{\label{fig:approximation_b}$\alpha=10^{16}$}{\includegraphics[trim=50 254 48 270,clip,width=0.42\columnwidth]{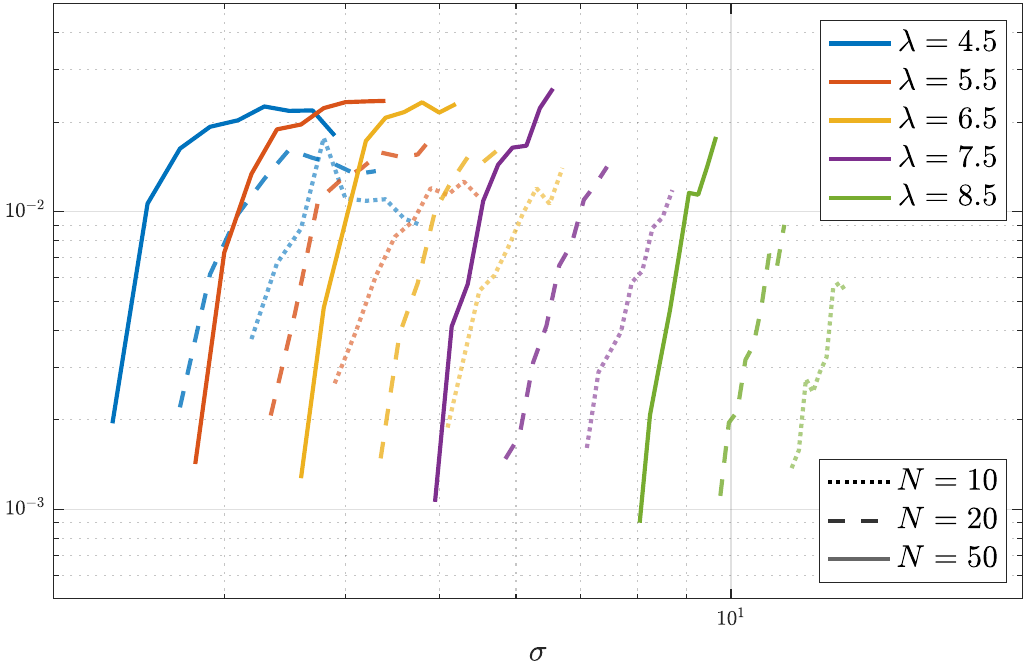}}
	\caption{Quantitative numerical analysis of the approximation error between the trajectories of the CBO scheme~\eqref{eq:CBO} and GD, i.e., the scaling of the stochastic noise~$g_k$ in \eqref{eq:thm:main_informal}. In the setting of the Canyon function~$\CE$ from Figure~\ref{fig:GrandCanyon3noisy} but without a local minimum in the valley,{\protect\footnotemark} \space we measure the distance between the two trajectories and plot the resulting approximation error for different values of $\alpha$ ((a) and (b)), different values of $\lambda$ (different colors), $\sigma$ (horizontal axis), and $N$ (different line styles).
    The other parameters of the CBO scheme~\eqref{eq:CBO} are $K=1000$ and $\Delta t=0.1$ with the remaining setting being as in Figure~\ref{fig:intuitionGiAyN}.\\
    The results validate the theoretical scalings on $\N{g_k}_2$ predicted by Theorem~\ref{thm:main_informal}.}
	\label{fig:approximation}
\end{figure*}

The stochastic perturbations~$g_k$ in \eqref{eq:thm:main_informal} are meaningful and not generic as they obey precise scalings thanks to the established bound in Theorem~\ref{thm:main_informal}.
As reflected by the first term of the bound on the error $\N{g_k}_2$,
the magnitude of this term become smaller as soon as
the discrete CBO time step size~$\Delta t\ll1$,
the drift parameter~$\lambda \approx 1/\Delta t$,
the noise parameter~$\sigma$ becomes smaller,
the weight parameter $\alpha$ is sufficiently large,
and the number of employed particles~$N$ becomes larger.
This behavior is confirmed numerically in Figure~\ref{fig:approximation} by measuring the closeness between the trajectories of the CBO scheme~\eqref{eq:CBO} and GD.
More precisely, better approximation is achieved for the values of $\lambda$ closer to $1/\Delta t$ (compare lines with different colors but same line style, and notice that smaller error can be obtained for larger $\lambda$), larger choices of $N$ (compare different line styles within a color), $\sigma$ as small as possible (each line decreases as $\sigma$ decreases), and larger values of $\alpha$ (compare the two subplots and notice the scaling of the approximation error).
For fixed $\lambda$ and $N$, however, $\sigma$ needs to be sufficiently large (in particular in case of a fixed number of iterates~$K$) to allow the CBO scheme~\eqref{eq:CBO} to iteratively explore the energy landscape within the given time horizon.
As visible from Figure~\ref{fig:approximation}, a larger number of particles~$N$ is needed to pass to smaller $\sigma$ and thus better approximation.
The second term of the bound on the error $\N{g_k}_2$ is likely an artifact of our proving technique. Indeed, we conjecture a potential amelioration of the estimate by refining the quantitative Laplace principle from \cite{fornasier2021consensus} involved in the proof of Proposition~\ref{prop:relaxation_CH_GF}, which would allow to remove the order $\CO(\tau)$ dependence of the bound.
Yet, as it stands, this term is about a {\it deterministic bounded} perturbation of the gradient, which is possibly of smaller magnitude than the gradient.
In fact, such bounded perturbation alone does not allow to explain the ability of the method to overcome local energy barriers in general (just think of a local minimizer, around which the magnitude of gradients grows faster than the displacement: in this case, any movement from the minimizer ought necessarily to get reverted).
Hence, it is the {\it stochastic} part of the perturbation that enables the convergence to global minimizers.
In fact, for
a moderate time step size $\Delta t>0$,
a drift parameter $\lambda>0$ relatively small compared to $1/\Delta t$,
a non-insignificant noise parameter $\sigma>0$,
a moderate value of the weight parameter $\alpha>0$
and
a modest number~$N$ of particles,
CBO is a stochastic relaxation of GD with strong noise.
\footnotetext{Otherwise, GD will necessarily get stuck in this local minimum located in the valley.}

\begin{remark}[Stochastic relaxations of GD]
    While the stochastic perturbation induced by CBO has the properties described above and is motivated by CBO's capability to provably converge to global minimizers of nonsmooth and nonconvex functions,
    there exist other stochastic relaxations of GD, which lead to different noise characteristics.
    Examples include the overdamped Langevin dynamics, the annealed Langevin dynamics, SGD, as well as mini-batch SGD.

    In the overdamped Langevin dynamics, the stochastic perturbation $g_k^{\text{LD}}$ is Brownian noise with zero mean and constant variance.
    In contrast, in the annealed Langevin dynamics, the stochastic perturbation $g_k^{\text{aLD}}$ is Brownian noise, which is damped out over time, i.e., $\Nnormal{g_k^{\text{aLD}}}_2\rightarrow0$ as $k\rightarrow\infty$.
    If one wishes to explicitly mimic such behavior in the noise induced by CBO, one could choose hyperparameters~$\lambda_k\rightarrow1/\Delta t$, $\sigma_k\rightarrow0$, $\alpha_k\rightarrow\infty$, and $N_k\rightarrow\infty$ that change during training as $k\rightarrow\infty$ enabling $\Nnormal{g_k}_2\rightarrow0$.

    For SGD and mini-batch SGD, on the other hand, the stochastic perturbations $g_k^{\text{SGD}}$ are data-dependent and depend, as is the case for the noise induced by CBO, non-trivially on the objective function~$\CE$. 
\end{remark}

Apart from gaining primarily theoretical insights from this link, let us conclude this section by mentioning a further, more practical aspect of establishing such a connection.
In several real-world applications, including various machine learning settings, using gradients may be undesirable or even not feasible.
This can be due to the black-box nature or nonsmoothness of the objective, memory limitations constraining the use of automatic differentiation, a substantial presence of spurious local minima, or the fact that gradients carry relevant information about data, which one may wish to keep private.
In machine learning, in specific,
the problems of hyperparameter tuning~\cite{bergstra2011algorithms,rapin2018nevergrad},
convex bandits~\cite{agarwal2011stochastic,shamir2017optimal},
reinforcement learning~\cite{sutton2018reinforcement}, the training of sparse and pruned neural networks~\cite{hoefler2021sparsity},
and federated learning~\cite{shokri2015privacy,mcmahan2017communication} stimulate interest in methods alternative to gradient-based ones. 
In such situations, if one still wishes to rely on a GD-like optimization behavior, Theorem~\ref{thm:main_informal} suggests the use of CBO (or related methods such as PSO~\cite{cipriani2021zero}), which will be both reliable and efficient (also through parallelization),\footnote{If gradients are available and cheap to compute, hybrid CBO-GD methods which additionally exploit this information are expected to be more efficient and competitive, in particular when it comes to large-scale high-dimensional problems. However, incorporating a gradient drift into CBO is possible~\cite{riedl2022leveraging} and may bear advantages of theoretical and practical nature as demonstrated in~\cite{riedl2022leveraging,trillos2023FedCBO,trillos2024CB2O,trillos2024attack}.} with linear complexity in the number of deployed particles.
We report, for instance, recent ideas in the setting of clustered federated learning~\cite{trillos2023FedCBO,trillos2024attack}, where CBO is leveraged to avoid reverse engineering of private data through exchange of gradients.
While we do not empirically investigate the complexity of CBO or provide comparisons with the state of the art for different applications in this paper,
a summary of the existing literature on this matter may be found in Section~\ref{sec:CBO_convergence_results}.

\section{Consensus-based optimization converges globally} \label{sec:CBO_convergence_results}

Let us recapitulate recent global convergence results for CBO. Optimizing a nonconvex objective~$\CE$ using the CBO dynamics~\eqref{eq:CBO_dynamics} corresponds to an evolution of $N$ particles in an interaction potential generated by $\CE$.
A global convergence analysis of this algorithm on the microscopic level proves difficult as it requires to study a system of a large number of interacting stochastic processes, which are highly correlated due to the dependence injected by communication through the consensus point~$\conspointnoarg$.
However, with the particles being interchangeable by design of the method~\cite{pinnau2017consensus}, the object of analytical interest is the empirical measure~$\empmeasure{t}$, whose continuous-time dynamics can be approximated, assuming propagation of chaos~\cite{sznitman1991propagation}, in the mean-field limit (large-particle limit) 
by the solution of the nonlinear nonlocal Fokker-Planck equation
\begin{align} \label{eq:fokker_planck}
\begin{split}
    \partial_t\rho_t
	= \;&\lambda\divergence \big(\!\left(x - \conspoint{\rho_t}\right)\rho_t\big)
	+ \frac{\sigma^2}{2}\sum_{k=1}^d \partial_{kk} \left(D\left(x-\conspoint{\rho_t}\right)_{kk}^2\rho_t\right).
\end{split}
\end{align}
This perspective enables the use of powerful deterministic calculus tools for analysis~\cite{carrillo2018analytical}.
\cite{fornasier2021consensus,fornasier2021convergence} proved that, in the mean-field limit, CBO performs a gradient descent of the Wasserstein-$2$ distance~$W_2$ to a Dirac measure located at the global minimizer~$\globmin$ with exponential rate.
Their results are valid for large classes of optimization problems under minimal assumptions about the initialization and are in particular generic in the sense that the convergence of $\rho_{t}$ is independent of the original hardness of the underlying optimization problem.

\begin{theorem}[CBO asymptotically convexifies nonconvex problems, {\cite[Theorem~2]{fornasier2021convergence}}] \label{thm:global_convergence_mfl}
    Fix $\varepsilon>0$.
    Let $\CE\in\CC(\bbR^d)$ satisfy \ref{asm:zero_global},
    and assume that for some constants~$\eta,\nu,\CE_\infty,R_0>0$ it holds the inverse continuity condition
    $\N{x-\globmin}_\infty \leq (\CE(x)-\underbar\CE)^\nu/\eta$ for all $x\in B^\infty_{R_0}(\globmin)$ and $\CE(x)-\underbar\CE>\CE_\infty$ for all $x\in (B^\infty_{R_0}(\globmin))^c$.
    Moreover, let $\rho_0\in\CP_4(\bbR^d)$ with $\globmin\in\supp\rho_0$.
    Then, for any $\vartheta\in(0,1)$ and parameters~$\lambda$, $\sigma>0$ with $2\lambda>\sigma^2$, there exists $\alpha_0=\alpha_0(\varepsilon,\vartheta,\lambda,\sigma,d,\nu,\eta,\rho_0)$ such that for all $\alpha\geq\alpha_0$ a weak solution~$(\rho_t)_{t\in[0,T^*]}$ to~\eqref{eq:fokker_planck} satisfies $W_2^2(\rho_T,\delta_\globmin) = \varepsilon$ with $T\in[\frac{1-\vartheta}{1+\vartheta/2}T^*,T^*]$, where $T^* := \frac{1}{(1-\vartheta)(2\lambda-\sigma^2)}\log\left(W_2^2(\rho_0,\delta_\globmin)/\varepsilon\right)$.
    Furthermore, on the time interval $[0,T]$, $W_2^2(\rho_t,\delta_\globmin)$ decays at least exponentially fast with rate $(1-\vartheta)(2\lambda-\sigma^2)$.
\end{theorem}
While Theorem~\ref{thm:global_convergence_mfl} captures a canonical convexification of a large class of nonconvex optimization problems as the number of optimizing particles of CBO approaches infinity,
it fails to explain empirically observed successes of the method using just few particles for high-dimensional problems coming from data analysis, signal processing, and machine learning~\cite{fornasier2020consensus_sphere_convergence,riedl2022leveraging,carrillo2019consensus,fornasier2021convergence,fornasier2020consensus_sphere_convergence,riedl2022leveraging,trillos2023FedCBO,trillos2024CB2O,trillos2024attack}.%
\footref{footnote:CBO_applications}
However, by ensuring that propagation of chaos~\cite{sznitman1991propagation} holds, \cite{fornasier2021consensus} quantify that the fluctuations of the empirical measure~$\empmeasure{t}$ around $\rho_{t}$ are of order $\CO(N^{-1/2})$ for any finite time horizon.
This allows to obtain probabilistic global convergence guarantees of the CBO dynamics~\eqref{eq:CBO_dynamics}.

\begin{theorem}[Global CBO convergence, {\cite[Theorem~3.8]{fornasier2021consensus}}]
    \label{thm:Global_CBO_convergence}
    Let $\varepsilon_{\mathrm{total}}>0$ and $\delta\in(0,1/2)$.
    Let $\CE\in\CC(\bbR^d)$ satisfy \ref{asm:zero_global}\;\!--\,\ref{asm:quadratic_growth} and consider valid the assumptions of Theorem~\ref{thm:global_convergence_mfl}.
    Then, the final iterations $(X_{K}^i)_{i=1,\dots,N}$ of \eqref{eq:CBO_dynamics} fulfill
    \begin{equation}
        \label{eq:thm:Global_CBO_convergence}
        \N{\frac{1}{N}\sum_{i=1}^N X_{K}^i-\globmin}_2^2 \leq \varepsilon_{\mathrm{total}}
    \end{equation}
    with probability larger than $1 - \left(\delta + \varepsilon_{\mathrm{total}}^{-1}(C_{\mathrm{D}}\Delta t + C_{\mathrm{MFA}}N^{-1} + \varepsilon)\right)$,
    where, besides problem-dependent constants, it hold $C_{\mathrm{D}}\!=\!C_{\mathrm{D}}(d,N,T^*,\delta^{-1})$~and~\mbox{$C_{\mathrm{MFA}}\!=\!C_{\mathrm{MFA}}(\alpha,T^*,\delta^{-1})$}.
\end{theorem}

Theorem~\ref{thm:Global_CBO_convergence} offers guidelines on the choice of the hyperparameters~$\Delta t$, $N$, and $K\propto T^*/\Delta t$ required to achieve the $\varepsilon_{\mathrm{total}}$ approximation guarantee with high probability.
For more details, we refer to \cite[Remark~3.9]{fornasier2021consensus}.

Despite the results of this section requiring the global minimizer~$\globmin$ to be unique as per inverse continuity condition, there exists a polarized CBO variant~\cite{bungert2022polarized,fornasiersun2024} capable of finding multiple global minimizers at once.


\section{Consensus-based optimization is a stochastic relaxation of gradient descent}
    \label{sec:CBO_stochastic_relaxation_GD}

In this section we present the technical details behind the main theoretical result of this work, Theorem~\ref{thm:main_informal},
i.e., we explain how to establish a connection between the CBO scheme~\eqref{eq:CBO}, which captures the flow of the derivative-free CBO dynamics~\eqref{eq:CBO_dynamics}, and GD.
\begin{figure*}[ht]
	\centering
	\subcaptionbox{The CH scheme~\eqref{eq:CH} (sampled over several runs) follows on average the valley of $\CE$ and can occasionally escape local minima. \label{fig:CH_GrandCanyon3noisy}}{\includegraphics[trim=28 209 31 200,clip,width=0.29\textwidth]{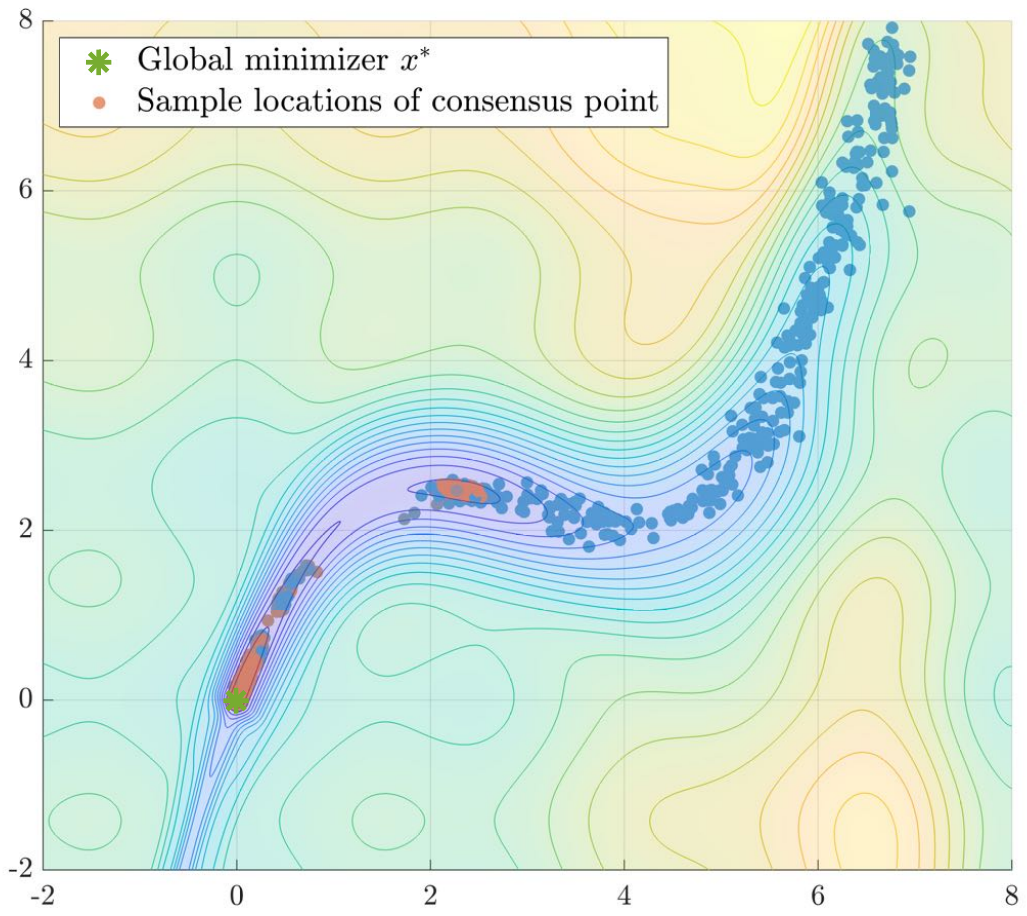}}
	\hspace{0.8em}
	\subcaptionbox{GD gets stuck in a local minimum of $\CE$. \label{fig:GD_GrandCanyon3noisy}}{\includegraphics[trim=28 209 31 200,clip,width=0.29\textwidth]{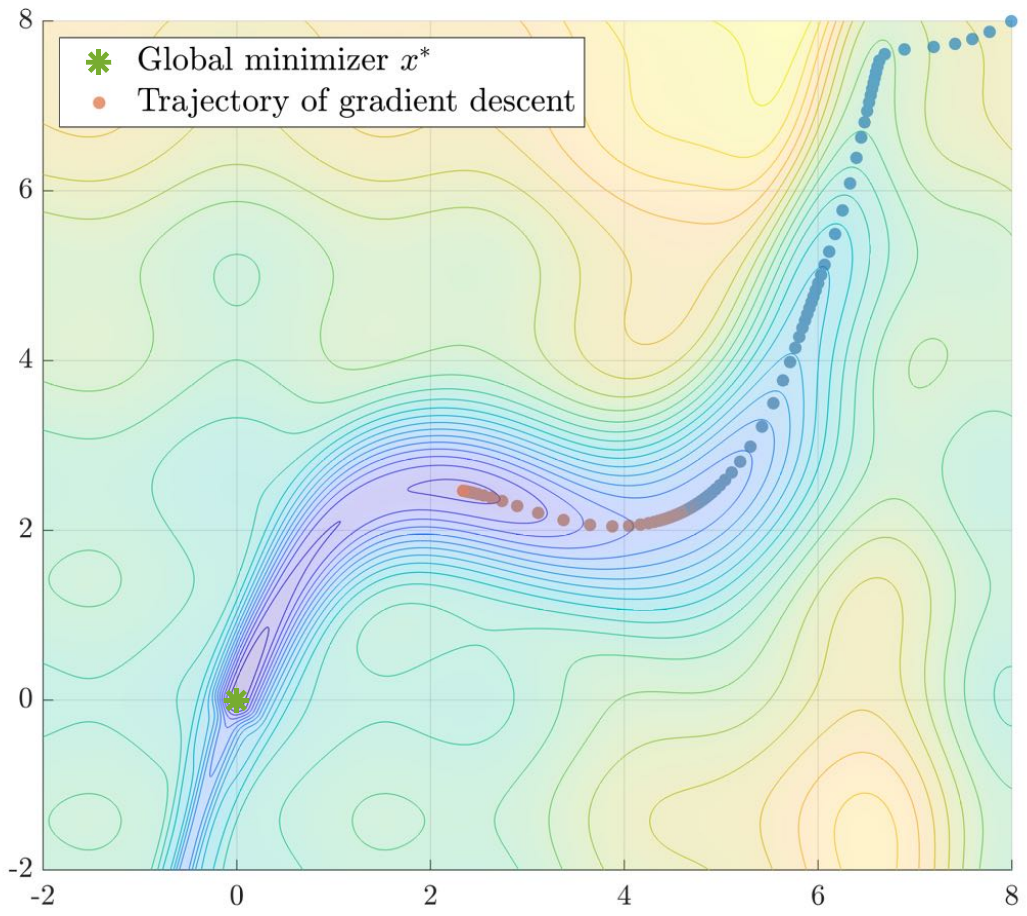}}
	\hspace{0.8em}
	\subcaptionbox{The Langevin dynamics (sampled over several runs) follows on average the valley of $\CE$ and escapes local minima. \label{fig:Langevin_GrandCanyon3noisy}}{\includegraphics[trim=28 209 31 200,clip,width=0.29\textwidth]{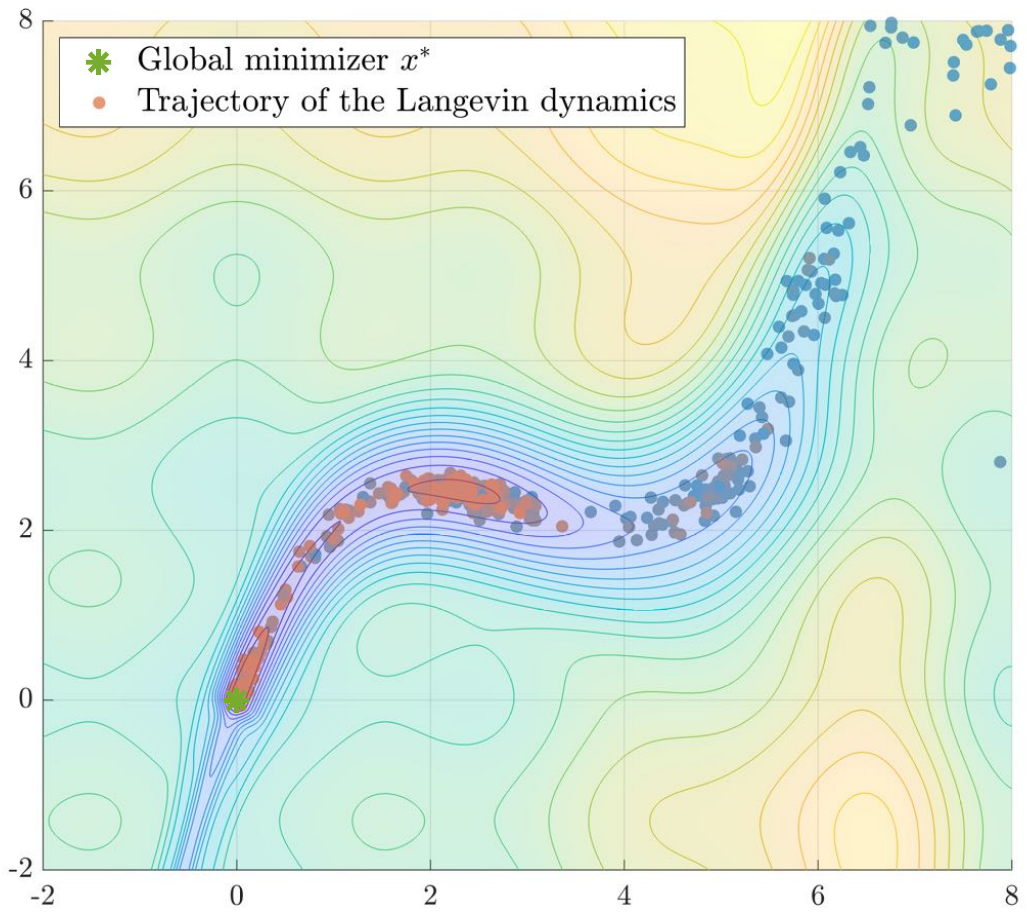}}
	\caption{An illustrative comparison between the algorithms discussed in this work.
	While GD (obtained as an explicit Euler time discretization of $\frac{d}{dt}x(t) = -\nabla\CE(x(t))$ with time step size $\Delta t=0.01$ and ran for $K=10^4$ iterations) gets stuck in a local minimum along the valley of $\CE$ (see (b)), the stochastic algorithms in (a) and (c) as well as Figure~\ref{fig:CBO_GrandCanyon3noisy} have the capability of escaping local minima.
	In (a) we depict the positions of the consensus hopping scheme~\eqref{eq:CH} for $K=250$ iterations with parameters $\alpha = 100$ and $\widetilde\sigma = 0.6$, and where we approximate the underlying measure~$\mu_k$ at each step~$k$ using $200$ samples.
	The ability of the CH scheme to escape local minima improves with larger $\widetilde\sigma$, see Figure~\ref{fig:comparison_CH} in Appendix~\ref{sec:appendix:additional_numerics}.
    In (c) we depict the trajectory of the annealed Langevin dynamics with $\beta_t=0.02\log(t+1)$ (obtained as an Euler-Maruyama time discretization with time step size $\Delta t=0.001$ and ran for $K=10^4$ iterations).
	The remaining setting is as in Figure~\ref{fig:intuitionGiAyN}, in particular, $50$ individual runs of the experiment are plotted in (a) and (c).}
	\label{fig:comparison_algorithms}
\end{figure*}

\textbf{From CBO to consensus hopping.}
Let us envision for the moment the movement of the particles during the CBO dynamics~\eqref{eq:CBO_dynamics}.
At every time step~$k$, after having computed $\conspoint{\empmeasure{k-1}}$, each particle moves a $\Delta t \lambda$ fraction of its distance towards this consensus point,
before being perturbed by stochastic noise.
As we let $\lambda\rightarrow1/\Delta t$, the particles' velocities increase,
until, in the case $\lambda=1/\Delta t$, each of them hops
directly to the previously computed consensus point, 
followed by a random fluctuation.
Put differently, we are left with a numerical scheme, which, at time step~$k$, samples $N$ particles around the old iterate in order to subsequently compute as new iterate the consensus point~\eqref{eq:consensus_point} of the empirical measure of the samples.
Such algorithm is precisely a Monte Carlo approximation of the {\it consensus hopping~(CH) scheme} with iterates $(x^{\mathrm{CH}}_{k})_{k=0,\dots,K}$ defined by
\begin{align}   \label{eq:CH}
\begin{aligned}
    x^{\mathrm{CH}}_{k}
	&= \conspoint{\mu_{k}},
	\quad\text{ with }\quad
    \mu_{k} = \CN\!\left(x^{\mathrm{CH}}_{k-1},\widetilde{\sigma}^2\Id\right)\!, \\
	x^{\mathrm{CH}}_{0}
	&=x_0.
\end{aligned}
\end{align}  

It resembles local search algorithms such as the Metropolis-Hastings algorithm~\cite{metropolis1953equation}, see, e.g., \cite[Section~2]{delahaye2019simulated},
as well as the covariance matrix adaptation evolution strategy (CMA-ES)~\cite{hansen2001completely}.

Theorem~\ref{thm:relaxation_CBO_CH} in Section~\ref{sec:appendix:proof_details} makes this intuition rigorous by quantifying the approximation quality between the CBO and the CH scheme in terms of the parameters of the two schemes.
Sample trajectories of the CH scheme are depicted in Figure~\ref{fig:CH_GrandCanyon3noisy}.

\textbf{From CH to GD.}
With the sampling measure~$\mu_k$ assigning (in particular for small $\widetilde\sigma$) most mass to the region close to the old iterate, the CH scheme~\eqref{eq:CH} improves at every time step $k$ its objective function value while staying near the previous iterate.
A conceptually analogous behavior to such localized sampling can be achieved through penalizing the length of the step taken at time step $k$.
This gives rise to an {\it implicit version of the CH scheme} with iterates $(\widetilde{x}^{\mathrm{CH}}_{k})_{k=0,\dots,K}$ given as
\begin{align}   \label{eq:CH_aux}
\begin{aligned}
    \widetilde{x}^{\mathrm{CH}}_{k}
    &= \argmin_{x\in\bbR^d} \; \widetilde\CE_k(x),
    \quad\text{ with }
    \quad\widetilde\CE_k(x) := \frac{1}{2\tau} \N{x^{\mathrm{CH}}_{k-1} - x}_2^2 + \CE(x), \\
    \widetilde{x}^{\text{CH}}_{0}
    &=x_0.
\end{aligned}
\end{align}  
The modulated objective $\widetilde\CE_k$ defined in \eqref{eq:CH_aux} naturally appears when writing out the expression of $\conspoint{\mu_{k}}$ from \eqref{eq:CH} using that $\mu_k$ is a Gaussian.
Formally, with details provided in \eqref{eq:proof:thm:error_decomposition_MMS_CH:21},
$\conspoint{\mu_{k}} \!=\! \fint\! x \exp(-\alpha\CE(x)) \exp\!\big(\!-\frac{1}{2\widetilde{\sigma}^2}\Nnormal{x\!-\!x^{\mathrm{CH}}_{k-\!1}}_2^2\big) d\lambda(x) \!=\! \fint\! x \exp(-\alpha\widetilde\CE_{k}(x)) d\lambda(x) \!=\! \conspointE{\widetilde\CE_{k}}{\lambda} \!\approx\! \widetilde{x}^{\mathrm{CH}}_{k}$, where we indicate by $\fint$ that we did not include the normalization.
This creates a link between the sampling width $\widetilde\sigma$ and the step size $\tau$.
The fact that the parameter~$\tau$ can be seen as the step size of \eqref{eq:CH_aux} becomes apparent when observing that the optimality condition of the $k$-th iterate of \eqref{eq:CH_aux} reads $\widetilde{x}^{\mathrm{CH}}_{k} = x^{\mathrm{CH}}_{k-1} - \tau\nabla\CE(\widetilde{x}^{\mathrm{CH}}_{k})$, which is an implicit gradient step.
Proposition~\ref{prop:relaxation_CH_GF} in Section~\ref{sec:appendix:proof_details} estimates the discrepancy between $x^{\mathrm{CH}}_{k}$ and $\widetilde{x}^{\mathrm{CH}}_{k}$ employing the quantitative Laplace principle~\cite[Proposition~4.5]{fornasier2021consensus}.

Let us conclude this discussion by remarking that the scheme~\eqref{eq:CH_aux} itself is not self-consistent but requires the computation of the iterates of the CH scheme~\eqref{eq:CH}.
For this reason we introduce the {\it minimizing movement scheme (MMS)~\cite{de1993new}} as the iterates $(x^{\mathrm{MMS}}_{k})_{k=0,\dots,K}$ given according to
\begin{align}   \label{eq:MMS}
\begin{aligned}
	x^{\mathrm{MMS}}_{k}
	&= \argmin_{x\in\bbR^d} \; \CE_k(x),
    \quad\text{ with }
    \quad\CE_k(x) := \frac{1}{2\tau} \N{x^{\mathrm{MMS}}_{k-1} - x}_2^2 + \CE(x), \\
    x^{\mathrm{MMS}}_{0}
	&=x_0,
\end{aligned}
\end{align}  
which is the discrete-time implicit Euler of the gradient flow $\frac{d}{dt} x(t) = -\nabla\CE(x(t))$~\cite{santambrogio2017euclideanmetricandwassersteingradientflows}.

\subsection{Proof of the main result, Theorem~\ref{thm:main_informal}}
    \label{sec:proof:thm:main_informal}


Let us provide upfront for the reader's convenience a schematic diagram that gives an overview of the different numerical schemes appearing throughout the paper and required in the proof of Theorem~\ref{thm:main_informal}.
\vspace{0.5em}

\definecolor{blue_intuition}{RGB}{0,114,189}
\definecolor{red_intuition}{RGB}{217,83,25}

\begin{center}
\begin{tikzpicture}
\def\width{4.5}
\def\widthrel{0.67} 
\def\gap{2.5}

\filldraw[fill=blue_intuition!40, draw=none] (0,5+\gap) rectangle (\width,2.5+\gap); %
\node[draw=none, align=center] at (\width/2,3.75+\gap) {\textbf{CBO scheme}~\eqref{eq:CBO}\\[1ex]
{\small$x^{\mathrm{CBO}}_{k}=\conspoint{\empmeasure{k}},$}\\[0.4ex]
{\small$\empmeasure{k} = \frac{1}{N}\sum_{i=1}^N \delta_{X_{k}^i}$}};
\node (A) at (\width/2-0.5, 2.5+\gap) {};
\node (B1) at (\width - 0.41-0.5, 2.5) {};
\node (B2) at (\width - 0.41 + 0.5, 2.5) {};
\node (C1) at (1.5*\width + 0.5 - 0.5, 2.5+\gap) {};
\node (C2) at (1.5*\width + 0.5 + 0.5, 2.5+\gap) {};
\node (D1) at (2*\width + 1.29 - 0.5, 2.5) {};
\node (D2) at (2*\width + 1.29+0.5, 2.5) {};
\node (E) at (2.5*\width + 1+0.5, 2.5+\gap) {};
\draw[align=center] [-] (A) -- (B1) node[midway,fill=white] {Thm.\@~\ref{thm:relaxation_CBO_CH},\\\eqref{proof:thm:relaxation_CBO_CH:7}, \eqref{proof:thm:relaxation_CBO_CH:9}};
\draw[align=center] [-] (B2) -- (C1) node[midway,fill=white] {Thm.\@~\ref{thm:relaxation_CBO_CH},\\\eqref{eq:proof:thm:relaxation_CBO_CH:43}};
\draw[align=center] [-] (C2) -- (D1) node[midway,fill=white] {Prop.\@~\ref{prop:relaxation_CH_GF}\\\eqref{eq:bound_CH_CH_aux}};
\draw[align=center] [-] (D2) -- (E) node[midway,fill=white] {Thm.\@~\ref{thm:relaxation_CH_GF},\\\eqref{eq:proof:thm:error_decomposition_MMS_CH:9}};

\filldraw[fill=blue_intuition!40, draw=none] (\width+0.5,5+\gap) rectangle (2*\width+0.5,2.5+\gap); %
\node[draw=none, align=center] at (1.5*\width+0.5,3.75+\gap) {\textbf{CH scheme}~\eqref{eq:CH}\\[1ex]
{\small$x^{\mathrm{CH}}_{k}=\conspoint{\mu_{k}},$}\\[0.4ex]
{\small$\mu_{k} = \CN\!\left(x^{\mathrm{CH}}_{k-1},\widetilde{\sigma}^2\Id\right)$}};
\filldraw[fill=blue_intuition!40, draw=none] (2*\width+1,5+\gap) rectangle (3*\width+1,2.5+\gap); %
\node[draw=none, align=center] at (2.5*\width+1,3.75+\gap) {\textbf{MMS}~\eqref{eq:MMS}\\[1ex]{\small$x^{\mathrm{MMS}}_{k}=\argmin_{x} \; \CE_k(x),$}\\[0.4ex]
{\small$\CE_k(x) \!:=\! \frac{\Nnormal{x^{\mathrm{MMS}}_{k-1} - x}_2^2 }{2\tau}\!+\!\CE(x)$}};
\filldraw[fill=gray!10, draw=none] (0+\width/2-0.29-0.2,2.5) rectangle (\width+\width/2-0.29,0); %
\node[draw=none, align=center] at (\width/2+\width/2-0.29-0.1,1.25) {{aux.\@ CBO scheme}~\eqref{eq:CBO_aux}\\[1ex]{\small$\widetilde{x}^{\mathrm{CBO}}_{k}=\conspoint{\empmeasuretilde{k}},$}\\[0.4ex]
{\small$\empmeasuretilde{k}=\frac{1}{N}\sum_{i=1}^N \delta_{\widetilde{X}_{k}^i\sim\CN(\widetilde{x}^{\mathrm{CBO}}_{k-1},\dots)}$}};
\filldraw[fill=gray!10, draw=none] (3*\width+1-\width/2-\width+0.29,2.5) rectangle (3*\width+1-\width/2+0.29+0.2,0); %
\node[draw=none, align=center] at (1.5*\width+0.5+\width/2+0.5+0.29+0.1,1.25) {{implicit CH scheme}~\eqref{eq:CH_aux}\\[1ex]{\small$\widetilde{x}^{\mathrm{CH}}_{k}=\argmin_{x} \; \widetilde\CE_k(x),$}\\[0.4ex]
{\small$\widetilde\CE_k(x)\!=\!\frac{\Nnormal{x^{\mathrm{CH}}_{k-1}-x}_2^2}{2\tau}\!+\!\CE(x)$}};
\end{tikzpicture}
\end{center}

The first (from left to right) three connections are relevant for the subsequent proof, the last connection constitutes an additional result that we present at the end of this paper.

\begin{proof}[Proof of Theorem~\ref{thm:main_informal}]
    From the optimality condition of the scheme $(\widetilde{x}^{\mathrm{CH}}_{k})_{k=1,\dots,K}$ in \eqref{eq:CH_aux} and with $(x^{\mathrm{CH}}_{k})_{k=1,\dots,K}$ as in \eqref{eq:CH}, we get $\left(\widetilde{x}^{\mathrm{CH}}_{k}\!-\!x^{\mathrm{CH}}_{k-1}\right) \!+\! \tau\nabla\CE(\widetilde{x}^{\mathrm{CH}}_{k}) \!=\! 0$.
    We now decompose
    \begin{align*}
    \begin{aligned}
        x^{\mathrm{CBO}}_{k}
        &= \widetilde{x}^{\mathrm{CH}}_{k} + \left(x^{\mathrm{CBO}}_{k}-\widetilde{x}^{\mathrm{CH}}_{k}\right)
        = x^{\mathrm{CH}}_{k-1} - \tau\nabla\CE(\widetilde{x}^{\mathrm{CH}}_{k}) + \left(x^{\mathrm{CBO}}_{k}-\widetilde{x}^{\mathrm{CH}}_{k}\right)\!.
    \end{aligned}
    \end{align*}
    Since $x^{\mathrm{CH}}_{k-1} = x^{\mathrm{CBO}}_{k-1} + \left(x^{\mathrm{CH}}_{k-1}-x^{\mathrm{CBO}}_{k-1}\right)$ and $ \nabla\CE(\widetilde{x}^{\mathrm{CH}}_{k}) = \nabla\CE(x^{\mathrm{CBO}}_{k-1}) + \left(\nabla\CE(\widetilde{x}^{\mathrm{CH}}_{k})-\nabla\CE(x^{\mathrm{CBO}}_{k-1})\right)$
    we can continue the former to obtain
    \begin{align*}
    \begin{aligned}
        x^{\mathrm{CBO}}_{k}
        = x^{\mathrm{CBO}}_{k-1} 
        &- \tau\nabla\CE(x^{\mathrm{CBO}}_{k-1}) + \left(x^{\mathrm{CH}}_{k-1}-x^{\mathrm{CBO}}_{k-1}\right) 
        - \tau\!\left(\nabla\CE(\widetilde{x}^{\mathrm{CH}}_{k})\!-\!\nabla\CE(x^{\mathrm{CBO}}_{k-1})\right) 
        + \left(x^{\mathrm{CBO}}_{k}\!-\!\widetilde{x}^{\mathrm{CH}}_{k}\right)\!,
    \end{aligned}
    \end{align*}
    where it remains to control the stochastic error term~$g_k$ from \eqref{eq:thm:main_informal}, which is comprised of $g_k^1:=x^{\mathrm{CH}}_{k-1}-x^{\mathrm{CBO}}_{k-1}$, $g_k^2:=\tau\!\left(\nabla\CE(\widetilde{x}^{\mathrm{CH}}_{k})-\nabla\CE(x^{\mathrm{CBO}}_{k-1})\right)$ and $g_k^3:=x^{\mathrm{CBO}}_{k}-\widetilde{x}^{\mathrm{CH}}_{k}$.
    By Theorem~\ref{thm:relaxation_CBO_CH},
    \begin{align*}
        \N{g_k^1}_2
        &= \N{x^{\mathrm{CH}}_{k-1}-x^{\mathrm{CBO}}_{k-1}}_2 
        = \CO\big(\!\abs{\lambda-1/\Delta t} + \sigma\sqrt{\Delta t} + \widetilde\sigma + N^{-1/2}\big)
    \end{align*}
    with high probability.
    For $g_k^2$, first notice that $\frac{1}{2\tau}\N{\widetilde{x}^{\mathrm{CH}}_{k}-x^{\mathrm{CH}}_{k-1}}_2^2 + \CE(\widetilde{x}^{\mathrm{CH}}_{k}) \leq \CE(x^{\mathrm{CH}}_{k-1})$ by definition of $\widetilde{x}^{\mathrm{CH}}_{k}$, which facilitates a bound on $\N{\widetilde{x}^{\mathrm{CH}}_{k}-x^{\mathrm{CH}}_{k-1}}_2$ of order $\CO(\tau)$ with high probability under \ref{asm:local_Lipschitz_quadratic_growth} and by means of Remark~\ref{rem:boundedness}.
    Since $\CE$ is $L$-smooth, with the latter and Theorem~\ref{thm:relaxation_CBO_CH},
    \begin{align*}
    \begin{aligned}
        \N{g_k^2}_2
        &= \tau\N{\nabla\CE(\widetilde{x}^{\mathrm{CH}}_{k})-\nabla\CE(x^{\mathrm{CBO}}_{k-1})}_2 
        \leq \tau L\N{\widetilde{x}^{\mathrm{CH}}_{k}-x^{\mathrm{CBO}}_{k-1}}_2 \\
        &\leq \tau L\left(\N{\widetilde{x}^{\mathrm{CH}}_{k}-x^{\mathrm{CH}}_{k-1}}_2+\N{x^{\mathrm{CH}}_{k-1}-x^{\mathrm{CBO}}_{k-1}}_2\right)\\
        &= \CO(\tau^2) \!+\! \CO\big(\tau\big(\!\abs{\lambda\!-\!1/\Delta t} \!+\! \sigma\sqrt{\Delta t} \!+\! \widetilde\sigma \!+\! N^{-1/2}\big)\big)
    \end{aligned}
    \end{align*}
    with high probability.
    By Theorem~\ref{thm:relaxation_CBO_CH} and Proposition~\ref{prop:relaxation_CH_GF} (quantitative Laplace principle~\cite[Proposition~4.5]{fornasier2021consensus}, see Proposition~\ref{prop:LaplacePrinciple}), it holds for a sufficiently large choice of $\alpha$ that
    \begin{align*}
    \begin{aligned}
        \N{g_k^3}_2
        &= \N{x^{\mathrm{CBO}}_{k}-\widetilde{x}^{\mathrm{CH}}_{k}}_2 
        \leq \N{x^{\mathrm{CBO}}_{k}-x^{\mathrm{CH}}_{k}}_2 + \N{x^{\mathrm{CH}}_{k}-\widetilde{x}^{\mathrm{CH}}_{k}}_2 \\
        &= \CO\big(\!\abs{\lambda-1/\Delta t} + \sigma\sqrt{\Delta t} + \widetilde\sigma + N^{-1/2}\big) + \CO(\tau)
    \end{aligned}
    \end{align*}
    with high probability concluding the proof after recalling $\widetilde\sigma^2=\tau/(2\alpha)$ as of Proposition~\ref{prop:relaxation_CH_GF}.
\end{proof}

\subsection{Technical details connecting CBO with GD via the CH scheme~\eqref{eq:CH}} \label{sec:appendix:proof_details}

We now make rigorous the central technical tools that we utilized to prove Theorem~\ref{thm:main_informal} in Section~\ref{sec:proof:thm:main_informal}, which were described colloquially at the beginning of Section~\ref{sec:CBO_stochastic_relaxation_GD}.
$\CM$ is the moment bound from Remark~\ref{rem:boundedness}.

\textbf{CBO is a stochastic relaxation of CH.}
Theorem~\ref{thm:relaxation_CBO_CH} explains how the CBO scheme~\eqref{eq:CBO} can be interpreted as a stochastic relaxation of the CH scheme~\eqref{eq:CH}.
\begin{theorem}[CBO relaxes CH] \label{thm:relaxation_CBO_CH}
	Fix $\varepsilon>0$ and $\delta\in(0,1/2)$.
	Let $\CE\in\CC(\bbR^d)$ satisfy \ref{asm:zero_global}\;\!--\,\ref{asm:quadratic_growth}.
	We denote by $(x^{\mathrm{CBO}}_{k})_{k=0,\dots,K}$ the iterates of the CBO scheme~\eqref{eq:CBO} and by $(x^{\mathrm{CH}}_{k})_{k=0,\dots,K}$ the ones of the CH scheme~\eqref{eq:CH}.
	Then, with probability larger than $1 - \left(\delta+\varepsilon\right)$, it holds for all $k=1,\dots,K$ that 
	\begin{equation} \label{eq:bound_CBO_CH}
		\N{x^{\mathrm{CBO}}_{k}\!-\!x^{\mathrm{CH}}_{k}}_2^2
		\leq \varepsilon^{-1} C \big(\!\abs{\lambda\!-\!1/\Delta t}^2 + \sigma^2\Delta t + \widetilde\sigma^2 + N^{-1}\big)
	\end{equation}
	with $C=C(\delta^{-1}, \Delta t, d, \alpha,\lambda,\sigma,b_1,b_2,C_1,C_2,K,\CM)$.
\end{theorem}

Auxiliary tools for the proof of Theorem~\ref{thm:relaxation_CBO_CH} are provided in Appendix~\ref{sec:appendix:thm:relaxation_CBO_CH}.

\begin{proof}[Proof of Theorem~\ref{thm:relaxation_CBO_CH}]
	We notice that for the choice~$\lambda=1/\Delta t$ the iterative update rule of the particles of the CBO dynamics~\eqref{eq:CBO_dynamics} becomes
	\begin{align} \label{proof:thm:relaxation_CBO_CH:1}
		\widetilde{X}_{k}^i
		= \conspoint{\empmeasuretilde{k-1}} + \sigma D\big(\widetilde{X}_{k-1}^i-\conspoint{\empmeasuretilde{k-1}}\big)\,B_{k}^i,
	\end{align}
	where $\empmeasuretilde{k} = \frac{1}{N}\sum_{i=1}^N \delta_{\widetilde{X}_{k}^i}$.
	In this case, the associated CBO scheme~\eqref{eq:CBO} reads
	\begin{align} \label{eq:CBO_aux}
	\begin{aligned}
		\widetilde{x}^{\mathrm{CBO}}_{k}
		&= \conspoint{\empmeasuretilde{k}}
		\quad\text{ with }
		\empmeasuretilde{k} = \frac{1}{N}\sum_{i=1}^N \delta_{\widetilde{X}_{k}^i},
		\text{ where }
		\widetilde{X}_{k}^i \sim \CN\!\left(\widetilde{x}^{\mathrm{CBO}}_{k-1},\Delta t\sigma^2D\big(\widetilde{X}_{k-1}^i-\widetilde{x}^{\mathrm{CBO}}_{k-1}\big)^2\right), \\
		\widetilde{x}^{\mathrm{CBO}}_{0}
		&=x_0,
	\end{aligned}
	\end{align}
	which resembles the CH dynamics~\eqref{eq:CH} with the difference in the underlying measure on which basis the consensus point~\eqref{eq:consensus_point} is computed.
	Let us further denote by $\widehat\mu^N_{k}$ the empirical measure $\widehat\mu^N_{k} = \frac{1}{N}\sum_{i=1}^N \delta_{Y_k^i}$, where $Y_k^i \sim \mu_k = \CN\!\left(x^{\mathrm{CH}}_{k-1},\widetilde{\sigma}^2\Id\right)$ for $i=1,\dots,N$, i.e., $Y_k^i = x^{\mathrm{CH}}_{k-1} + \widetilde\sigma B_{Y,k}^i$ with $B_{Y,k}^i$ being a standard Gaussian random vector.
	
	To obtain the probabilistic formulation of the statement, let us denote the underlying probability space over which all considered random variables get their realizations by $(\Omega,\CF,\bbP)$ and introduce the subset~$\Omega_M$ of $\Omega$ of suitably bounded random variables according to
	\begin{align*}
		\Omega_M
		:= \left\{\omega\in\Omega : \!\!\max_{k=0,\dots,K} \max\left\{\int\!\N{\;\!\dummy\;\!}_2^4 d\empmeasure{k}, \int\!\N{\;\!\dummy\;\!}_2^4 d\empmeasuretilde{k}, \int\!\N{\;\!\dummy\;\!}_2^4 d\mu_{k}, \int\!\N{\;\!\dummy\;\!}_2^4 d\widehat\mu^N_{k}\right\} \leq M^4 \right\}.
	\end{align*}	
	For the associated cutoff function (random variable) we write $\mathbbm{1}_{\Omega_M}$.
	Moreover, let us define the cutoff functions
	\begin{align}
		\CI_{M,k} =
		\begin{cases}
			1, & \text{ if } \max\left\{\int\N{\;\!\dummy\;\!}_2^4 d\empmeasure{k}, \int\N{\;\!\dummy\;\!}_2^4 d\empmeasuretilde{k}, \int\N{\;\!\dummy\;\!}_2^4 d\mu_{k}, \int\N{\;\!\dummy\;\!}_2^4 d\widehat\mu^N_{k}\right\} \leq M^4 \text{ for all } \ell\leq k,\\
			0, & \text{ else},
		\end{cases}
	\end{align}
	which are adapted to the natural filtration and satisfy $\mathbbm{1}_{\Omega_M}\leq\CI_{M,k}$ as well as $\CI_{M,k} = \CI_{M,k}\CI_{M,\ell}$ for all $\ell\leq k$.
	
	We can decompose the expected squared discrepancy $\bbE\N{x^{\mathrm{CBO}}_{k} - x^{\mathrm{CH}}_{k}}_2^2\mathbbm{1}_{\Omega_M}$ between the CBO scheme~\eqref{eq:CBO} and the CH scheme~\eqref{eq:CH} as
	\begin{align} \label{eq:proof:thm:relaxation_CBO_CH:3}
	\begin{aligned}
		\bbE\N{x^{\mathrm{CBO}}_{k} - x^{\mathrm{CH}}_{k}}_2^2\CI_{M,k}
		&\leq
		 2\bbE\N{x^{\mathrm{CBO}}_{k} - \widetilde{x}^{\mathrm{CBO}}_{k}}_2^2\CI_{M,k} + 2\bbE\N{\widetilde{x}^{\mathrm{CBO}}_{k} - x^{\mathrm{CH}}_{k}}_2^2\CI_{M,k}.
	\end{aligned}
	\end{align}
	In what follows we individually bound the two terms on the right-hand side of~\eqref{eq:proof:thm:relaxation_CBO_CH:3}.
	
	\textbf{First term:}
	Let us start with the term $\bbE\N{x^{\mathrm{CBO}}_{k} - \widetilde{x}^{\mathrm{CBO}}_{k}}_2^2\CI_{M,k}$, which we bound by combining the stability estimate for the consensus point, Lemma~\ref{lem:stability_consensus_points}, with Lemma~\ref{lem:stability_CBO_dynamics}, a stability estimate for the underlying CBO dynamics~\eqref{eq:CBO_dynamics} w.r.t.\@ its parameters~$\lambda$ and $\sigma$.
	Denoting the auxiliary cutoff function defined in~\eqref{eq:lem:stability_CBO_dynamics:cutoff1} in the setting~$\widehat\rho^{N,1}_{k}=\empmeasure{k}$ and $\widehat\rho^{N,2}_{k}=\empmeasuretilde{k}$ by $\overbar{\CI}^1_{M,k}$, we have due to Lemma~\ref{lem:stability_consensus_points} the estimate 
	\begin{align} \label{proof:thm:relaxation_CBO_CH:7}
	\begin{aligned}
		\bbE\N{x^{\mathrm{CBO}}_{k} - \widetilde{x}^{\mathrm{CBO}}_{k}}_2^2\CI_{M,k}
		&=\bbE\N{\conspoint{\empmeasure{k}} - \conspoint{\empmeasuretilde{k}}}_2^2\CI_{M,k}\\
		&\leq \bbE\N{\conspoint{\empmeasure{k}} - \conspoint{\empmeasuretilde{k}}}_2^2\overbar{\CI}^1_{M,k}
		\leq c_0 \bbE W_2^2(\empmeasure{k},\empmeasuretilde{k})\,\overbar{\CI}^1_{M,k}
	\end{aligned}
	\end{align}
	with a constant $c_0=c_0(\alpha,C_1,C_2,M)>0$.
	In the first inequality of~\eqref{proof:thm:relaxation_CBO_CH:7} we exploited $\CI_{M,k}\leq\overbar{\CI}^1_{M,k}$.
	The Wasserstein distance appearing on the right-hand side of~\eqref{proof:thm:relaxation_CBO_CH:7} can be upper bounded by choosing $\pi=\frac{1}{N}\sum_{i=1}^N\delta_{X_{k}^i}\otimes\delta_{\widetilde{X}_{k}^i}$ as viable transportation plan in Definition~\eqref{def:wassersteindistance}.
	This constitutes the first inequality in the estimate
	\begin{align} \label{proof:thm:relaxation_CBO_CH:9}
	\begin{aligned}
		\bbE W_2^2(\empmeasure{k},\empmeasuretilde{k})\,\overbar{\CI}^1_{M,k}
		&\leq \frac{1}{N}\sum_{i=1}^N\bbE\,\Nbig{X_{k}^i - \widetilde{X}_{k}^i}_2^2\,\overbar{\CI}^1_{M,k} \\
		&\leq c_1 \left(\abs{\lambda_1-\lambda_2}^2+\abs{\sigma_1-\sigma_2}^2\right) e^{c_2(k-1)} 
		\leq c_1 \abs{\lambda-\frac{1}{\Delta t}}^2 e^{c_2(k-1)},
	\end{aligned}
	\end{align}
	whereas the second step is a consequence of Lemma~\ref{lem:stability_CBO_dynamics} applied with $\lambda_1=\lambda, \sigma_1=\sigma$ and $\lambda_2=1/\Delta t, \sigma_2=\sigma$ as exploited in the third step.
	Hence, the constants are $c_1=c_1(\Delta t, d,b_1,b_2,M)>0$ and $c_2=c_2(\Delta t, d,\alpha,\lambda,\sigma,C_1,C_2,M)>0$.
	
	\textbf{Second term:}
	To control the term~$\bbE\N{\widetilde{x}^{\mathrm{CBO}}_{k} - x^{\mathrm{CH}}_{k}}_2^2\CI_{M,k}$ we start by decomposing it according to
	\begin{align} \label{proof:thm:relaxation_CBO_CH:21}
	\begin{aligned}
		\bbE\N{\widetilde{x}^{\mathrm{CBO}}_{k} - x^{\mathrm{CH}}_{k}}_2^2\CI_{M,k}
		\leq 2\bbE\N{\widetilde{x}^{\mathrm{CBO}}_{k} - \conspoint{\widehat\mu^N_{k}}}_2^2\CI_{M,k} + 2\bbE\N{\conspoint{\widehat\mu^N_{k}} - x^{\mathrm{CH}}_{k}}_2^2\CI_{M,k},
	\end{aligned}
	\end{align}
	where $\widehat\mu^N_{k}$ is as introduced at the beginning of the proof.
	For the first summand in~\eqref{proof:thm:relaxation_CBO_CH:21} the stability estimate for the consensus point, Lemma~\ref{lem:stability_consensus_points}, gives
	\begin{align} \label{proof:thm:relaxation_CBO_CH:23}
	\begin{aligned}
		\bbE\N{\widetilde{x}^{\mathrm{CBO}}_{k} - \conspoint{\widehat\mu^N_{k}}}_2^2\CI_{M,k}
		= \bbE\N{\conspoint{\empmeasuretilde{k}} - \conspoint{\widehat\mu^N_{k}}}_2^2\CI_{M,k}
		\leq c_0 \bbE W_2^2(\empmeasuretilde{k},\widehat\mu^N_{k})\,\CI_{M,k} 
	\end{aligned}
	\end{align}
	with a constant $c_0=c_0(\alpha,C_1,C_2,M)>0$.
	By choosing $\pi=\frac{1}{N}\sum_{i=1}^N\delta_{\widetilde{X}_{k}^i}\otimes\delta_{Y_{k}^i}$ as viable transportation plan in Definition~\eqref{def:wassersteindistance}, we can further bound
	\begin{align} \label{proof:thm:relaxation_CBO_CH:25}
	\begin{aligned}
		\bbE W_2^2(\empmeasuretilde{k},\widehat\mu^N_{k})\,\CI_{M,k} 
		\leq \frac{1}{N}\sum_{i=1}^N \bbE\Nbig{\widetilde{X}_{k}^i - Y_k^i}_2^2\,\CI_{M,k}
	\end{aligned}
	\end{align}
	and since $\widetilde{X}_{k}^i \sim \CN\big(\widetilde{x}^{\mathrm{CBO}}_{k-1},\Delta t\sigma^2D(\widetilde{X}_{k-1}^i-\widetilde{x}^{\mathrm{CBO}}_{k-1})^2\big)$ and $Y_k^i\sim \CN\!\left(x^{\mathrm{CH}}_{k-1},\widetilde{\sigma}^2\Id\right)$ we have
	\begin{align} \label{proof:thm:relaxation_CBO_CH:27}
	\begin{aligned}
		\frac{1}{N}\sum_{i=1}^N \bbE\Nbig{\widetilde{X}_{k}^i - Y_k^i}_2^2\,\CI_{M,k}
		&\leq 2\bbE\Nbig{\widetilde{x}^{\mathrm{CBO}}_{k-1}-x^{\mathrm{CH}}_{k-1}}_2^2\,\CI_{M,k-1}\\
		&\quad\;\! + \frac{4}{N}\sum_{i=1}^N \left(\sigma^2\bbE\Nbig{D\big(\widetilde{X}_{k-1}^i-\widetilde{x}^{\mathrm{CBO}}_{k-1}\big)B_{k}^i}_2^2\,\CI_{M,k-1} + \widetilde\sigma^2\bbE\Nbig{B_{Y,k}^i}_2^2\right) \\
		&\leq 2\bbE\Nbig{\widetilde{x}^{\mathrm{CBO}}_{k-1}-x^{\mathrm{CH}}_{k-1}}_2^2\,\CI_{M,k-1} + 8\sigma^2\Delta t \left(b_1 + (1+b_2) M^2\right) + 4\widetilde\sigma^2.
	\end{aligned}
	\end{align}
	Note that in the last step we exploited the definition of the cutoff function~$\CI_{M,k}$, which allowed to derive the bound
	\begin{align*}
	\begin{aligned}
		\frac{1}{N}\sum_{i=1}^N \bbE\Nbig{D\big(\widetilde{X}_{k-1}^i-\widetilde{x}^{\mathrm{CBO}}_{k-1}\big)B_{k}^i}_2^2\,\CI_{M,k-1}
		&\leq \frac{2}{N}\sum_{i=1}^N \bbE\left(\Nbig{\widetilde{X}_{k-1}^i}_2^2+\N{\widetilde{x}^{\mathrm{CBO}}_{k-1}}_2^2\right) \Nbig{B_{k}^i}_2^2\,\CI_{M,k-1} \\
		&\leq 2\bbE\N{\widetilde{x}^{\mathrm{CBO}}_{k-1}}_2^2\CI_{M,k-1} + \frac{2}{N}\sum_{i=1}^N \bbE\Nbig{\widetilde{X}_{k-1}^i}_2^2\,\CI_{M,k-1} \\
		&\leq 2\left(b_1 + (1+b_2) M^2\right)
	\end{aligned}
	\end{align*}
	by using Lemma~\ref{lem:boundedness_consensuspoint} and the fact that $B_{k}^i\sim\CN\left(0,\Delta t\Id\right)$ is independent from $\widetilde{X}_{k-1}^i$ and $\widetilde{x}^{\mathrm{CBO}}_{k-1}$.
	Inserting~\eqref{proof:thm:relaxation_CBO_CH:27} into \eqref{proof:thm:relaxation_CBO_CH:25} and this into \eqref{proof:thm:relaxation_CBO_CH:23} afterwards, we are left with
	\begin{align} \label{proof:thm:relaxation_CBO_CH:29}
	\begin{aligned}
		\bbE\N{\widetilde{x}^{\mathrm{CBO}}_{k} - \conspoint{\widehat\mu^N_{k}}}_2^2\CI_{M,k}
		\leq c\left(\bbE\Nbig{\widetilde{x}^{\mathrm{CBO}}_{k-1}-x^{\mathrm{CH}}_{k-1}}_2^2\,\CI_{M,k-1} + \sigma^2\Delta t + \widetilde\sigma^2\right)
	\end{aligned}
	\end{align}
	with a constant $c=c(c_0,b_1,b_2,M)>0$.
	For the second summand in~\eqref{proof:thm:relaxation_CBO_CH:21} we have by Lemma~\ref{lem:lln_consensuspoint}
	\begin{align} \label{proof:thm:relaxation_CBO_CH:31}
	\begin{aligned}
		\bbE\N{\conspoint{\widehat\mu^N_{k}} - x^{\mathrm{CH}}_{k}}_2^2\CI_{M,k}
		\leq \bbE\N{\conspoint{\widehat\mu^N_{k}} - \conspoint{\mu_k}}_2^2\overbar{\CI}^2_{M,k} 
		\leq c_3N^{-1},
	\end{aligned}
	\end{align}
	with $c_3=c_3(\alpha,b_1,b_2,C_2,M)>0$ and where $\overbar{\CI}^2_{M,k}$ is an auxiliary cutoff function as defined in~\eqref{eq:lem:stability_CBO_dynamics:cutoff2}.
	Combining \eqref{proof:thm:relaxation_CBO_CH:29} with \eqref{proof:thm:relaxation_CBO_CH:31} we arrive for \eqref{proof:thm:relaxation_CBO_CH:21} at
	\begin{align} \label{proof:thm:relaxation_CBO_CH:41}
	\begin{aligned}
		\bbE\N{\widetilde{x}^{\mathrm{CBO}}_{k} - x^{\mathrm{CH}}_{k}}_2^2\CI_{M,k}
		&\leq c\bbE\Nbig{\widetilde{x}^{\mathrm{CBO}}_{k-1}-x^{\mathrm{CH}}_{k-1}}_2^2\,\CI_{M,k-1} + c\sigma^2\Delta t + c\widetilde\sigma^2 + c_3N^{-1}.
	\end{aligned}
	\end{align}
	An application of the discrete variant of Gr\"onwall's inequality~\eqref{eq:gronwall_inequality} shows that
	\begin{align} \label{eq:proof:thm:relaxation_CBO_CH:43}
	\begin{aligned}
		\bbE\N{\widetilde{x}^{\mathrm{CBO}}_{k} - x^{\mathrm{CH}}_{k}}_2^2\CI_{M,k}
		&\leq c^k\bbE\N{\widetilde{x}^{\mathrm{CBO}}_{0} - x^{\mathrm{CH}}_{0}}_2^2 + \left(c\sigma^2\Delta t + c\widetilde\sigma^2 + c_3N^{-1}\right)e^{c(k-1)},
	\end{aligned}
	\end{align}
	where the first term vanishes as both schemes are initialized with $x_0$.
	
	\textbf{Concluding step:}
	Collecting the estimates~\eqref{proof:thm:relaxation_CBO_CH:7} combined with \eqref{proof:thm:relaxation_CBO_CH:9}, and \eqref{eq:proof:thm:relaxation_CBO_CH:43} yields for \eqref{eq:proof:thm:relaxation_CBO_CH:3} the bound
	\begin{align} \label{eq:proof:thm:relaxation_CBO_CH:51}
	\begin{aligned}
		\bbE\N{x^{\mathrm{CBO}}_{k} - x^{\mathrm{CH}}_{k}}_2^2\mathbbm{1}_{\Omega_M}
		&\lesssim
		c_0c_1 \abs{\lambda-\frac{1}{\Delta t}}^2 e^{c_2(k-1)} + \left(c\sigma^2\Delta t + c\widetilde\sigma^2 + c_3N^{-1}\right)e^{c(k-1)} \\
		&\leq
		C \left(\abs{\lambda-\frac{1}{\Delta t}}^2 + \sigma^2\Delta t + \widetilde\sigma^2 + c_3N^{-1}\right),
	\end{aligned}
	\end{align}
	with a constant $C=C(\Delta t, d, \alpha,\lambda,\sigma,b_1,b_2,C_1,C_2,K,M)>0$.
	Observe that we additionally used $\mathbbm{1}_{\Omega_M}\leq\CI_{M,k}$ as observed at the beginning.
	
	\textbf{Probabilistic formulation:}
	We first note that with Markov's inequality we have the estimate
	\begin{align*}
	\begin{aligned}
		\bbP\big(\Omega_M^c\big)
		&= \bbP\left(\max_{k=0,\dots,K} \max\left\{\int\!\N{\;\!\dummy\;\!}_2^4 d\empmeasure{k}, \int\!\N{\;\!\dummy\;\!}_2^4 d\empmeasuretilde{k}, \int\!\N{\;\!\dummy\;\!}_2^4 d\mu_{k}, \int\!\N{\;\!\dummy\;\!}_2^4 d\widehat\mu^N_{k}\right\} > M^4\right) \\
		&\leq \frac{1}{M^4} \left(\bbE\max_{k=0,\dots,K} \int\!\N{\;\!\dummy\;\!}_2^4 d\empmeasure{k}+\bbE\max_{k=0,\dots,K}\int\!\N{\;\!\dummy\;\!}_2^4 d\empmeasuretilde{k}\right.\\
		&\qquad\qquad+\left.\bbE\max_{k=0,\dots,K}\int\!\N{\;\!\dummy\;\!}_2^4 d\mu_{k}+\bbE\max_{k=0,\dots,K}\int\!\N{\;\!\dummy\;\!}_2^4 d\widehat\mu^N_{k}\right) \\
		&\leq \frac{1}{M^4} \big(\CM^{\mathrm{CBO}}+\widetilde\CM^{\mathrm{CBO}}+\CM^{\mathrm{CH}}+\widehat\CM^{\mathrm{CH}}\;\!\big),
	\end{aligned}
	\end{align*}
	where the last inequality is due to Lemmas~\ref{lem:boundedness_CBOdynamics}, ~\ref{lem:boundedness_CH} and~\ref{lem:boundedness_CH_empirical}.
	Here, $\widetilde\CM^{\mathrm{CBO}}$ represents the constant $\CM^{\mathrm{CBO}}$ from Lemma~\ref{lem:boundedness_CBOdynamics} in the setting where $\lambda=1/\Delta t$, i.e., we have that $
        \widetilde\CM^{\mathrm{CBO}}=\CM^{\mathrm{CBO}}(1/\Delta t, \sigma, d, b_1, b_2, K\Delta t, K, \rho_0)$.
	Thus, for any $\delta\in(0,1/2)$, a sufficiently large choice $M = M(\delta^{-1},\CM^{\mathrm{CBO}},\widetilde\CM^{\mathrm{CBO}},\CM^{\mathrm{CH}},\widehat\CM^{\mathrm{CH}})$ allows to ensure $\bbP\big(\Omega_M^c\big)\leq \delta$.
	To conclude the proof, let us denote by $K_\varepsilon\subset\Omega$ the set, where~\eqref{eq:bound_CBO_CH} does not hold and abbreviate
	\begin{align*}
		\epsilon
		= \varepsilon^{-1} C \left(\abs{\lambda-\frac{1}{\Delta t}}^2 + \sigma^2\Delta t + \widetilde\sigma^2 + c_3N^{-1}\right).
	\end{align*}
	For the probability of this set we can estimate
	\begin{align}
	\begin{aligned}
		\bbP\big(K_\varepsilon\big)
		&= \bbP\big(K_\varepsilon\cap\Omega_M\big) + \bbP\big(K_\varepsilon\cap\Omega_M^c\big)
		\leq \bbP\big(K_\varepsilon\,\big|\,\Omega_M\big)\,\bbP\big(\Omega_M\big) + \bbP\big(\Omega_M^c\big) \\
		&\leq \bbP\big(K_\varepsilon\,\big|\,\Omega_M\big) + \delta
		\leq \epsilon^{-1}\,\bbE\left[\N{x^{\mathrm{CBO}}_{k} - x^{\mathrm{CH}}_{k}}_2^2 \,\Big|\,\Omega_M\right] + \delta,
	\end{aligned}
	\end{align}
	where the last step is due to Markov's inequality.
    By definition of the conditional expectation we further have
	\begin{align*}
	\begin{aligned}
		\bbE\left[\N{x^{\mathrm{CBO}}_{k} - x^{\mathrm{CH}}_{k}}_2^2 \,\Big|\, \Omega_M\right]
		\leq \frac{1}{\bbP\big(\Omega_M\big)} \bbE\N{x^{\mathrm{CBO}}_{k} - x^{\mathrm{CH}}_{k}}_2^2\mathbbm{1}_{\Omega_M}
		\leq 2\bbE\N{x^{\mathrm{CBO}}_{k} - x^{\mathrm{CH}}_{k}}_2^2\mathbbm{1}_{\Omega_M}.
	\end{aligned}
	\end{align*}
	Inserting now the expression from \eqref{eq:proof:thm:relaxation_CBO_CH:51} concludes the proof.
\end{proof}

\subsection{Technical details connecting CBO with GD via the CH scheme~\eqref{eq:CH}} \label{sec:appendix:proof_details}

We now make rigorous the central technical tools that we utilized to prove Theorem~\ref{thm:main_informal} in Section~\ref{sec:proof:thm:main_informal}, which were described colloquially at the beginning of Section~\ref{sec:CBO_stochastic_relaxation_GD}.
$\CM$ is the moment bound from Remark~\ref{rem:boundedness}.

\textbf{CBO is a stochastic relaxation of CH.}
Theorem~\ref{thm:relaxation_CBO_CH} explains how the CBO scheme~\eqref{eq:CBO} can be interpreted as a stochastic relaxation of the CH scheme~\eqref{eq:CH}.
\begin{theorem}[CBO relaxes CH] \label{thm:relaxation_CBO_CH}
	Fix $\varepsilon>0$ and $\delta\in(0,1/2)$.
	Let $\CE\in\CC(\bbR^d)$ satisfy \ref{asm:zero_global}\;\!--\,\ref{asm:quadratic_growth}.
	We denote by $(x^{\mathrm{CBO}}_{k})_{k=0,\dots,K}$ the iterates of the CBO scheme~\eqref{eq:CBO} and by $(x^{\mathrm{CH}}_{k})_{k=0,\dots,K}$ the ones of the CH scheme~\eqref{eq:CH}.
	Then, with probability larger than $1 - \left(\delta+\varepsilon\right)$, it holds for all $k=1,\dots,K$ that 
	\begin{equation} \label{eq:bound_CBO_CH}
		\N{x^{\mathrm{CBO}}_{k}\!-\!x^{\mathrm{CH}}_{k}}_2^2
		\leq \varepsilon^{-1} C \big(\!\abs{\lambda\!-\!1/\Delta t}^2 + \sigma^2\Delta t + \widetilde\sigma^2 + N^{-1}\big)
	\end{equation}
	with $C=C(\delta^{-1}, \Delta t, d, \alpha,\lambda,\sigma,b_1,b_2,C_1,C_2,K,\CM)$.
\end{theorem}

Auxiliary tools for the proof of Theorem~\ref{thm:relaxation_CBO_CH} are provided in Appendix~\ref{sec:appendix:thm:relaxation_CBO_CH}.

\begin{proof}[Proof of Theorem~\ref{thm:relaxation_CBO_CH}]
	We notice that for the choice~$\lambda=1/\Delta t$ the iterative update rule of the particles of the CBO dynamics~\eqref{eq:CBO_dynamics} becomes
	\begin{align} \label{proof:thm:relaxation_CBO_CH:1}
		\widetilde{X}_{k}^i
		= \conspoint{\empmeasuretilde{k-1}} + \sigma D\big(\widetilde{X}_{k-1}^i-\conspoint{\empmeasuretilde{k-1}}\big)\,B_{k}^i,
	\end{align}
	where $\empmeasuretilde{k} = \frac{1}{N}\sum_{i=1}^N \delta_{\widetilde{X}_{k}^i}$.
	In this case, the associated CBO scheme~\eqref{eq:CBO} reads
	\begin{align} \label{eq:CBO_aux}
	\begin{aligned}
		\widetilde{x}^{\mathrm{CBO}}_{k}
		&= \conspoint{\empmeasuretilde{k}}
		\quad\text{ with }
		\empmeasuretilde{k} = \frac{1}{N}\sum_{i=1}^N \delta_{\widetilde{X}_{k}^i},
		\text{ where }
		\widetilde{X}_{k}^i \sim \CN\!\left(\widetilde{x}^{\mathrm{CBO}}_{k-1},\Delta t\sigma^2D\big(\widetilde{X}_{k-1}^i-\widetilde{x}^{\mathrm{CBO}}_{k-1}\big)^2\right), \\
		\widetilde{x}^{\mathrm{CBO}}_{0}
		&=x_0,
	\end{aligned}
	\end{align}
	which resembles the CH dynamics~\eqref{eq:CH} with the difference in the underlying measure on which basis the consensus point~\eqref{eq:consensus_point} is computed.
	Let us further denote by $\widehat\mu^N_{k}$ the empirical measure $\widehat\mu^N_{k} = \frac{1}{N}\sum_{i=1}^N \delta_{Y_k^i}$, where $Y_k^i \sim \mu_k = \CN\!\left(x^{\mathrm{CH}}_{k-1},\widetilde{\sigma}^2\Id\right)$ for $i=1,\dots,N$, i.e., $Y_k^i = x^{\mathrm{CH}}_{k-1} + \widetilde\sigma B_{Y,k}^i$ with $B_{Y,k}^i$ being a standard Gaussian random vector.
	
	To obtain the probabilistic formulation of the statement, let us denote the underlying probability space over which all considered random variables get their realizations by $(\Omega,\CF,\bbP)$ and introduce the subset~$\Omega_M$ of $\Omega$ of suitably bounded random variables according to
	\begin{align*}
		\Omega_M
		:= \left\{\omega\in\Omega : \!\!\max_{k=0,\dots,K} \max\left\{\int\!\N{\;\!\dummy\;\!}_2^4 d\empmeasure{k}, \int\!\N{\;\!\dummy\;\!}_2^4 d\empmeasuretilde{k}, \int\!\N{\;\!\dummy\;\!}_2^4 d\mu_{k}, \int\!\N{\;\!\dummy\;\!}_2^4 d\widehat\mu^N_{k}\right\} \leq M^4 \right\}.
	\end{align*}	
	For the associated cutoff function (random variable) we write $\mathbbm{1}_{\Omega_M}$.
	Moreover, let us define the cutoff functions
	\begin{align}
		\CI_{M,k} =
		\begin{cases}
			1, & \text{ if } \max\left\{\int\N{\;\!\dummy\;\!}_2^4 d\empmeasure{k}, \int\N{\;\!\dummy\;\!}_2^4 d\empmeasuretilde{k}, \int\N{\;\!\dummy\;\!}_2^4 d\mu_{k}, \int\N{\;\!\dummy\;\!}_2^4 d\widehat\mu^N_{k}\right\} \leq M^4 \text{ for all } \ell\leq k,\\
			0, & \text{ else},
		\end{cases}
	\end{align}
	which are adapted to the natural filtration and satisfy $\mathbbm{1}_{\Omega_M}\leq\CI_{M,k}$ as well as $\CI_{M,k} = \CI_{M,k}\CI_{M,\ell}$ for all $\ell\leq k$.
	
	We can decompose the expected squared discrepancy $\bbE\N{x^{\mathrm{CBO}}_{k} - x^{\mathrm{CH}}_{k}}_2^2\mathbbm{1}_{\Omega_M}$ between the CBO scheme~\eqref{eq:CBO} and the CH scheme~\eqref{eq:CH} as
	\begin{align} \label{eq:proof:thm:relaxation_CBO_CH:3}
	\begin{aligned}
		\bbE\N{x^{\mathrm{CBO}}_{k} - x^{\mathrm{CH}}_{k}}_2^2\CI_{M,k}
		&\leq
		 2\bbE\N{x^{\mathrm{CBO}}_{k} - \widetilde{x}^{\mathrm{CBO}}_{k}}_2^2\CI_{M,k} + 2\bbE\N{\widetilde{x}^{\mathrm{CBO}}_{k} - x^{\mathrm{CH}}_{k}}_2^2\CI_{M,k}.
	\end{aligned}
	\end{align}
	In what follows we individually bound the two terms on the right-hand side of~\eqref{eq:proof:thm:relaxation_CBO_CH:3}.
	
	\textbf{First term:}
	Let us start with the term $\bbE\N{x^{\mathrm{CBO}}_{k} - \widetilde{x}^{\mathrm{CBO}}_{k}}_2^2\CI_{M,k}$, which we bound by combining the stability estimate for the consensus point, Lemma~\ref{lem:stability_consensus_points}, with Lemma~\ref{lem:stability_CBO_dynamics}, a stability estimate for the underlying CBO dynamics~\eqref{eq:CBO_dynamics} w.r.t.\@ its parameters~$\lambda$ and $\sigma$.
	Denoting the auxiliary cutoff function defined in~\eqref{eq:lem:stability_CBO_dynamics:cutoff1} in the setting~$\widehat\rho^{N,1}_{k}=\empmeasure{k}$ and $\widehat\rho^{N,2}_{k}=\empmeasuretilde{k}$ by $\overbar{\CI}^1_{M,k}$, we have due to Lemma~\ref{lem:stability_consensus_points} the estimate 
	\begin{align} \label{proof:thm:relaxation_CBO_CH:7}
	\begin{aligned}
		\bbE\N{x^{\mathrm{CBO}}_{k} - \widetilde{x}^{\mathrm{CBO}}_{k}}_2^2\CI_{M,k}
		&=\bbE\N{\conspoint{\empmeasure{k}} - \conspoint{\empmeasuretilde{k}}}_2^2\CI_{M,k}\\
		&\leq \bbE\N{\conspoint{\empmeasure{k}} - \conspoint{\empmeasuretilde{k}}}_2^2\overbar{\CI}^1_{M,k}
		\leq c_0 \bbE W_2^2(\empmeasure{k},\empmeasuretilde{k})\,\overbar{\CI}^1_{M,k}
	\end{aligned}
	\end{align}
	with a constant $c_0=c_0(\alpha,C_1,C_2,M)>0$.
	In the first inequality of~\eqref{proof:thm:relaxation_CBO_CH:7} we exploited $\CI_{M,k}\leq\overbar{\CI}^1_{M,k}$.
	The Wasserstein distance appearing on the right-hand side of~\eqref{proof:thm:relaxation_CBO_CH:7} can be upper bounded by choosing $\pi=\frac{1}{N}\sum_{i=1}^N\delta_{X_{k}^i}\otimes\delta_{\widetilde{X}_{k}^i}$ as viable transportation plan in Definition~\eqref{def:wassersteindistance}.
	This constitutes the first inequality in the estimate
	\begin{align} \label{proof:thm:relaxation_CBO_CH:9}
	\begin{aligned}
		\bbE W_2^2(\empmeasure{k},\empmeasuretilde{k})\,\overbar{\CI}^1_{M,k}
		&\leq \frac{1}{N}\sum_{i=1}^N\bbE\,\Nbig{X_{k}^i - \widetilde{X}_{k}^i}_2^2\,\overbar{\CI}^1_{M,k} \\
		&\leq c_1 \left(\abs{\lambda_1-\lambda_2}^2+\abs{\sigma_1-\sigma_2}^2\right) e^{c_2(k-1)} 
		\leq c_1 \abs{\lambda-\frac{1}{\Delta t}}^2 e^{c_2(k-1)},
	\end{aligned}
	\end{align}
	whereas the second step is a consequence of Lemma~\ref{lem:stability_CBO_dynamics} applied with $\lambda_1=\lambda, \sigma_1=\sigma$ and $\lambda_2=1/\Delta t, \sigma_2=\sigma$ as exploited in the third step.
	Hence, the constants are $c_1=c_1(\Delta t, d,b_1,b_2,M)>0$ and $c_2=c_2(\Delta t, d,\alpha,\lambda,\sigma,C_1,C_2,M)>0$.
	
	\textbf{Second term:}
	To control the term~$\bbE\N{\widetilde{x}^{\mathrm{CBO}}_{k} - x^{\mathrm{CH}}_{k}}_2^2\CI_{M,k}$ we start by decomposing it according to
	\begin{align} \label{proof:thm:relaxation_CBO_CH:21}
	\begin{aligned}
		\bbE\N{\widetilde{x}^{\mathrm{CBO}}_{k} - x^{\mathrm{CH}}_{k}}_2^2\CI_{M,k}
		\leq 2\bbE\N{\widetilde{x}^{\mathrm{CBO}}_{k} - \conspoint{\widehat\mu^N_{k}}}_2^2\CI_{M,k} + 2\bbE\N{\conspoint{\widehat\mu^N_{k}} - x^{\mathrm{CH}}_{k}}_2^2\CI_{M,k},
	\end{aligned}
	\end{align}
	where $\widehat\mu^N_{k}$ is as introduced at the beginning of the proof.
	For the first summand in~\eqref{proof:thm:relaxation_CBO_CH:21} the stability estimate for the consensus point, Lemma~\ref{lem:stability_consensus_points}, gives
	\begin{align} \label{proof:thm:relaxation_CBO_CH:23}
	\begin{aligned}
		\bbE\N{\widetilde{x}^{\mathrm{CBO}}_{k} - \conspoint{\widehat\mu^N_{k}}}_2^2\CI_{M,k}
		= \bbE\N{\conspoint{\empmeasuretilde{k}} - \conspoint{\widehat\mu^N_{k}}}_2^2\CI_{M,k}
		\leq c_0 \bbE W_2^2(\empmeasuretilde{k},\widehat\mu^N_{k})\,\CI_{M,k} 
	\end{aligned}
	\end{align}
	with a constant $c_0=c_0(\alpha,C_1,C_2,M)>0$.
	By choosing $\pi=\frac{1}{N}\sum_{i=1}^N\delta_{\widetilde{X}_{k}^i}\otimes\delta_{Y_{k}^i}$ as viable transportation plan in Definition~\eqref{def:wassersteindistance}, we can further bound
	\begin{align} \label{proof:thm:relaxation_CBO_CH:25}
	\begin{aligned}
		\bbE W_2^2(\empmeasuretilde{k},\widehat\mu^N_{k})\,\CI_{M,k} 
		\leq \frac{1}{N}\sum_{i=1}^N \bbE\Nbig{\widetilde{X}_{k}^i - Y_k^i}_2^2\,\CI_{M,k}
	\end{aligned}
	\end{align}
	and since $\widetilde{X}_{k}^i \sim \CN\big(\widetilde{x}^{\mathrm{CBO}}_{k-1},\Delta t\sigma^2D(\widetilde{X}_{k-1}^i-\widetilde{x}^{\mathrm{CBO}}_{k-1})^2\big)$ and $Y_k^i\sim \CN\!\left(x^{\mathrm{CH}}_{k-1},\widetilde{\sigma}^2\Id\right)$ we have
	\begin{align} \label{proof:thm:relaxation_CBO_CH:27}
	\begin{aligned}
		\frac{1}{N}\sum_{i=1}^N \bbE\Nbig{\widetilde{X}_{k}^i - Y_k^i}_2^2\,\CI_{M,k}
		&\leq 2\bbE\Nbig{\widetilde{x}^{\mathrm{CBO}}_{k-1}-x^{\mathrm{CH}}_{k-1}}_2^2\,\CI_{M,k-1}\\
		&\quad\;\! + \frac{4}{N}\sum_{i=1}^N \left(\sigma^2\bbE\Nbig{D\big(\widetilde{X}_{k-1}^i-\widetilde{x}^{\mathrm{CBO}}_{k-1}\big)B_{k}^i}_2^2\,\CI_{M,k-1} + \widetilde\sigma^2\bbE\Nbig{B_{Y,k}^i}_2^2\right) \\
		&\leq 2\bbE\Nbig{\widetilde{x}^{\mathrm{CBO}}_{k-1}-x^{\mathrm{CH}}_{k-1}}_2^2\,\CI_{M,k-1} + 8\sigma^2\Delta t \left(b_1 + (1+b_2) M^2\right) + 4\widetilde\sigma^2.
	\end{aligned}
	\end{align}
	Note that in the last step we exploited the definition of the cutoff function~$\CI_{M,k}$, which allowed to derive the bound
	\begin{align*}
	\begin{aligned}
		\frac{1}{N}\sum_{i=1}^N \bbE\Nbig{D\big(\widetilde{X}_{k-1}^i-\widetilde{x}^{\mathrm{CBO}}_{k-1}\big)B_{k}^i}_2^2\,\CI_{M,k-1}
		&\leq \frac{2}{N}\sum_{i=1}^N \bbE\left(\Nbig{\widetilde{X}_{k-1}^i}_2^2+\N{\widetilde{x}^{\mathrm{CBO}}_{k-1}}_2^2\right) \Nbig{B_{k}^i}_2^2\,\CI_{M,k-1} \\
		&\leq 2\bbE\N{\widetilde{x}^{\mathrm{CBO}}_{k-1}}_2^2\CI_{M,k-1} + \frac{2}{N}\sum_{i=1}^N \bbE\Nbig{\widetilde{X}_{k-1}^i}_2^2\,\CI_{M,k-1} \\
		&\leq 2\left(b_1 + (1+b_2) M^2\right)
	\end{aligned}
	\end{align*}
	by using Lemma~\ref{lem:boundedness_consensuspoint} and the fact that $B_{k}^i\sim\CN\left(0,\Delta t\Id\right)$ is independent from $\widetilde{X}_{k-1}^i$ and $\widetilde{x}^{\mathrm{CBO}}_{k-1}$.
	Inserting~\eqref{proof:thm:relaxation_CBO_CH:27} into \eqref{proof:thm:relaxation_CBO_CH:25} and this into \eqref{proof:thm:relaxation_CBO_CH:23} afterwards, we are left with
	\begin{align} \label{proof:thm:relaxation_CBO_CH:29}
	\begin{aligned}
		\bbE\N{\widetilde{x}^{\mathrm{CBO}}_{k} - \conspoint{\widehat\mu^N_{k}}}_2^2\CI_{M,k}
		\leq c\left(\bbE\Nbig{\widetilde{x}^{\mathrm{CBO}}_{k-1}-x^{\mathrm{CH}}_{k-1}}_2^2\,\CI_{M,k-1} + \sigma^2\Delta t + \widetilde\sigma^2\right)
	\end{aligned}
	\end{align}
	with a constant $c=c(c_0,b_1,b_2,M)>0$.
	For the second summand in~\eqref{proof:thm:relaxation_CBO_CH:21} we have by Lemma~\ref{lem:lln_consensuspoint}
	\begin{align} \label{proof:thm:relaxation_CBO_CH:31}
	\begin{aligned}
		\bbE\N{\conspoint{\widehat\mu^N_{k}} - x^{\mathrm{CH}}_{k}}_2^2\CI_{M,k}
		\leq \bbE\N{\conspoint{\widehat\mu^N_{k}} - \conspoint{\mu_k}}_2^2\overbar{\CI}^2_{M,k} 
		\leq c_3N^{-1},
	\end{aligned}
	\end{align}
	with $c_3=c_3(\alpha,b_1,b_2,C_2,M)>0$ and where $\overbar{\CI}^2_{M,k}$ is an auxiliary cutoff function as defined in~\eqref{eq:lem:stability_CBO_dynamics:cutoff2}.
	Combining \eqref{proof:thm:relaxation_CBO_CH:29} with \eqref{proof:thm:relaxation_CBO_CH:31} we arrive for \eqref{proof:thm:relaxation_CBO_CH:21} at
	\begin{align} \label{proof:thm:relaxation_CBO_CH:41}
	\begin{aligned}
		\bbE\N{\widetilde{x}^{\mathrm{CBO}}_{k} - x^{\mathrm{CH}}_{k}}_2^2\CI_{M,k}
		&\leq c\bbE\Nbig{\widetilde{x}^{\mathrm{CBO}}_{k-1}-x^{\mathrm{CH}}_{k-1}}_2^2\,\CI_{M,k-1} + c\sigma^2\Delta t + c\widetilde\sigma^2 + c_3N^{-1}.
	\end{aligned}
	\end{align}
	An application of the discrete variant of Gr\"onwall's inequality~\eqref{eq:gronwall_inequality} shows that
	\begin{align} \label{eq:proof:thm:relaxation_CBO_CH:43}
	\begin{aligned}
		\bbE\N{\widetilde{x}^{\mathrm{CBO}}_{k} - x^{\mathrm{CH}}_{k}}_2^2\CI_{M,k}
		&\leq c^k\bbE\N{\widetilde{x}^{\mathrm{CBO}}_{0} - x^{\mathrm{CH}}_{0}}_2^2 + \left(c\sigma^2\Delta t + c\widetilde\sigma^2 + c_3N^{-1}\right)e^{c(k-1)},
	\end{aligned}
	\end{align}
	where the first term vanishes as both schemes are initialized with $x_0$.
	
	\textbf{Concluding step:}
	Collecting the estimates~\eqref{proof:thm:relaxation_CBO_CH:7} combined with \eqref{proof:thm:relaxation_CBO_CH:9}, and \eqref{eq:proof:thm:relaxation_CBO_CH:43} yields for \eqref{eq:proof:thm:relaxation_CBO_CH:3} the bound
	\begin{align} \label{eq:proof:thm:relaxation_CBO_CH:51}
	\begin{aligned}
		\bbE\N{x^{\mathrm{CBO}}_{k} - x^{\mathrm{CH}}_{k}}_2^2\mathbbm{1}_{\Omega_M}
		&\lesssim
		c_0c_1 \abs{\lambda-\frac{1}{\Delta t}}^2 e^{c_2(k-1)} + \left(c\sigma^2\Delta t + c\widetilde\sigma^2 + c_3N^{-1}\right)e^{c(k-1)} \\
		&\leq
		C \left(\abs{\lambda-\frac{1}{\Delta t}}^2 + \sigma^2\Delta t + \widetilde\sigma^2 + c_3N^{-1}\right),
	\end{aligned}
	\end{align}
	with a constant $C=C(\Delta t, d, \alpha,\lambda,\sigma,b_1,b_2,C_1,C_2,K,M)>0$.
	Observe that we additionally used $\mathbbm{1}_{\Omega_M}\leq\CI_{M,k}$ as observed at the beginning.
	
	\textbf{Probabilistic formulation:}
	We first note that with Markov's inequality we have the estimate
	\begin{align*}
	\begin{aligned}
		\bbP\big(\Omega_M^c\big)
		&= \bbP\left(\max_{k=0,\dots,K} \max\left\{\int\!\N{\;\!\dummy\;\!}_2^4 d\empmeasure{k}, \int\!\N{\;\!\dummy\;\!}_2^4 d\empmeasuretilde{k}, \int\!\N{\;\!\dummy\;\!}_2^4 d\mu_{k}, \int\!\N{\;\!\dummy\;\!}_2^4 d\widehat\mu^N_{k}\right\} > M^4\right) \\
		&\leq \frac{1}{M^4} \left(\bbE\max_{k=0,\dots,K} \int\!\N{\;\!\dummy\;\!}_2^4 d\empmeasure{k}+\bbE\max_{k=0,\dots,K}\int\!\N{\;\!\dummy\;\!}_2^4 d\empmeasuretilde{k}\right.\\
		&\qquad\qquad+\left.\bbE\max_{k=0,\dots,K}\int\!\N{\;\!\dummy\;\!}_2^4 d\mu_{k}+\bbE\max_{k=0,\dots,K}\int\!\N{\;\!\dummy\;\!}_2^4 d\widehat\mu^N_{k}\right) \\
		&\leq \frac{1}{M^4} \big(\CM^{\mathrm{CBO}}+\widetilde\CM^{\mathrm{CBO}}+\CM^{\mathrm{CH}}+\widehat\CM^{\mathrm{CH}}\;\!\big),
	\end{aligned}
	\end{align*}
	where the last inequality is due to Lemmas~\ref{lem:boundedness_CBOdynamics}, ~\ref{lem:boundedness_CH} and~\ref{lem:boundedness_CH_empirical}.
	Here, $\widetilde\CM^{\mathrm{CBO}}$ represents the constant $\CM^{\mathrm{CBO}}$ from Lemma~\ref{lem:boundedness_CBOdynamics} in the setting where $\lambda=1/\Delta t$, i.e., we have that $
        \widetilde\CM^{\mathrm{CBO}}=\CM^{\mathrm{CBO}}(1/\Delta t, \sigma, d, b_1, b_2, K\Delta t, K, \rho_0)$.
	Thus, for any $\delta\in(0,1/2)$, a sufficiently large choice $M = M(\delta^{-1},\CM^{\mathrm{CBO}},\widetilde\CM^{\mathrm{CBO}},\CM^{\mathrm{CH}},\widehat\CM^{\mathrm{CH}})$ allows to ensure $\bbP\big(\Omega_M^c\big)\leq \delta$.
	To conclude the proof, let us denote by $K_\varepsilon\subset\Omega$ the set, where~\eqref{eq:bound_CBO_CH} does not hold and abbreviate
	\begin{align*}
		\epsilon
		= \varepsilon^{-1} C \left(\abs{\lambda-\frac{1}{\Delta t}}^2 + \sigma^2\Delta t + \widetilde\sigma^2 + c_3N^{-1}\right).
	\end{align*}
	For the probability of this set we can estimate
	\begin{align}
	\begin{aligned}
		\bbP\big(K_\varepsilon\big)
		&= \bbP\big(K_\varepsilon\cap\Omega_M\big) + \bbP\big(K_\varepsilon\cap\Omega_M^c\big)
		\leq \bbP\big(K_\varepsilon\,\big|\,\Omega_M\big)\,\bbP\big(\Omega_M\big) + \bbP\big(\Omega_M^c\big) \\
		&\leq \bbP\big(K_\varepsilon\,\big|\,\Omega_M\big) + \delta
		\leq \epsilon^{-1}\,\bbE\left[\N{x^{\mathrm{CBO}}_{k} - x^{\mathrm{CH}}_{k}}_2^2 \,\Big|\,\Omega_M\right] + \delta,
	\end{aligned}
	\end{align}
	where the last step is due to Markov's inequality.
    By definition of the conditional expectation we further have
	\begin{align*}
	\begin{aligned}
		\bbE\left[\N{x^{\mathrm{CBO}}_{k} - x^{\mathrm{CH}}_{k}}_2^2 \,\Big|\, \Omega_M\right]
		\leq \frac{1}{\bbP\big(\Omega_M\big)} \bbE\N{x^{\mathrm{CBO}}_{k} - x^{\mathrm{CH}}_{k}}_2^2\mathbbm{1}_{\Omega_M}
		\leq 2\bbE\N{x^{\mathrm{CBO}}_{k} - x^{\mathrm{CH}}_{k}}_2^2\mathbbm{1}_{\Omega_M}.
	\end{aligned}
	\end{align*}
	Inserting now the expression from \eqref{eq:proof:thm:relaxation_CBO_CH:51} concludes the proof.
\end{proof}

\textbf{CH behaves like a gradient-based method.}
Since by definition of the iterates~$\widetilde{x}^{\mathrm{CH}}_{k}$ in~\eqref{eq:CH_aux}, it holds $\widetilde{x}^{\mathrm{CH}}_{k} = x^{\mathrm{CH}}_{k-1} - \tau\nabla\CE(\widetilde{x}^{\mathrm{CH}}_{k})$, Proposition~\ref{prop:relaxation_CH_GF} constitutes that (granted a sufficiently large choice of $\alpha$ and a suitably small choice of $\widetilde\sigma$) the CH scheme~\eqref{eq:CH} performs a gradient step at every time step~$k$.

\begin{proposition}[CH performs gradient steps]
    \label{prop:relaxation_CH_GF}
    Fix $\varepsilon>0$ and $\delta\in(0,1/2)$.
	Let $\CE\in\CC(\bbR^d)$ satisfy \ref{asm:zero_global}\;\!--\,\ref{asm:lambda-convex}.
	We denote by $(x^{\mathrm{CH}}_{k})_{k=0,\dots,K}$ the iterations of the CH scheme~\eqref{eq:CH} and by $(\widetilde{x}^{\mathrm{CH}}_{k})_{k=0,\dots,K}$ the ones of the scheme~\eqref{eq:CH_aux}.
	Moreover, assume that the parameters~$\alpha,\tau$ and $\widetilde\sigma$ are such that $\tau<1/(-2\Lambda)$ if $\Lambda<0$,  $\alpha\gtrsim\frac{1}{\tau}d\log d$ is sufficiently large and $\widetilde\sigma^2=\tau/(2\alpha)$.
	Then, with probability larger than $1 - \left(\delta+\varepsilon\right)$, it holds for all $k=1,\dots,K$ that 
	\begin{equation}
		\label{eq:bound_CH_CH_aux}
		\N{x^{\mathrm{CH}}_{k}\!-\!\widetilde{x}^{\mathrm{CH}}_{k}}_2^2
		\leq \varepsilon^{-1} c \tau^2
	\end{equation}
	with $c=c(\delta^{-1},C_1,\CM)$.
\end{proposition}

The proof of Proposition~\ref{prop:relaxation_CH_GF} is based on the quantitative Laplace principle~\cite[Proposition~4.5]{fornasier2021consensus} (see also Proposition~\ref{prop:LaplacePrinciple} for a recap).
We conjecture that a refinement thereof may allow to control the error in \eqref{eq:bound_CH_CH_aux} just through $\alpha$ and $\widetilde\sigma$ without creating a dependence on $\tau$.
Nevertheless, the bound is sufficient to suggest a gradient-like behavior of the CH scheme~\eqref{eq:CH} (see the discussion after Theorem~\ref{thm:main_informal}).

Auxiliary tools for the proof of Proposition~\ref{prop:relaxation_CH_GF} are provided in Appendix~\ref{sec:appendix:thm:relaxation_CH_GF}.

\begin{proof}[Proof of Proposition~\ref{prop:relaxation_CH_GF}]
    By using the quantitative Laplace principle~\ref{prop:LaplacePrinciple}, we make rigorous and quantify the fact that $x^{\mathrm{CH}}_{k}$ approximates the minimizer of $\widetilde\CE_{k}$, denoted by $\widetilde{x}_k$, for sufficiently large~$\alpha$.

    To obtain the probabilistic formulation of the statement, let us again denote the underlying probability space by $(\Omega,\CF,\bbP)$ (note that we can use the same probability space as before since the stochasticity of both schemes~\eqref{eq:CH} and \eqref{eq:CH_aux} is solely coming from the initialization) and introduce the subset~$\widetilde\Omega_M$ of $\Omega$ of suitably bounded random variables according to
	\begin{align*}
		\widetilde\Omega_M
		:= \left\{\omega\in\Omega : \max_{k=0,\dots,K} \max\left\{\N{x^{\mathrm{CH}}_{k}}_2,\Nbig{\widetilde{x}^{\mathrm{CH}}_{k}}_2\right\} \leq M \right\}.
	\end{align*}	
	For the associated cutoff function (random variable) we write $\mathbbm{1}_{\widetilde\Omega_M}$.
 
    We first notice that by definition of the consensus point~$\conspointnoarg$ in~\eqref{eq:consensus_point} it holds
	\begin{align}
	\label{eq:proof:thm:error_decomposition_MMS_CH:21}
	\begin{aligned}
		x^{\mathrm{CH}}_{k}
		= \conspoint{\mu_{k}}
		&= \int x \frac{\exp(-\alpha\CE(x))}{\N{\exp(-\alpha\CE)}_{L^1(\mu_{k})}} d\mu_{k}(x) \\
		&= \int x \frac{\exp(-\alpha\CE(x)) \exp\left(-\frac{1}{4\widetilde{\sigma}^2}\N{x-x^{\mathrm{CH}}_{k-1}}_2^2\right)}{\int \exp(-\alpha\CE(x')) \exp\left(-\frac{1}{4\widetilde{\sigma}^2}\N{x'-x^{\mathrm{CH}}_{k-1}}_2^2\right) d\nu_{k}(x')} d\nu_{k}(x) \\
		&= \int x \frac{\exp(-\alpha\widetilde\CE_{k}(x))}{\Nnormal{\exp(-\alpha\widetilde\CE_{k})}_{L^1(\nu_{k})}} d\nu_{k}(x) \\
		&= \conspointE{\widetilde\CE_{k}}{\nu_k},
	\end{aligned}
	\end{align}
	which introduces the relation $\tau=2\alpha\widetilde\sigma^2$ and where we chose $\nu_k = \CN\!\left(x^{\mathrm{CH}}_{k-1},2\widetilde{\sigma}^2\Id\right)$, which is globally supported, i.e., $\supp(\nu_k)=\bbR^d$.
	Since, according to Lemma~\ref{lem:widetildeCEhasICP}, $\widetilde\CE_{k}$ satisfies the inverse continuity property~\eqref{eq:ICP_widetildeCE} with $\nu = 1/2$ and $\eta = \sqrt{\frac{1}{2\tau} + \frac{\Lambda}{2}}>0$, the quantitative Laplace principle, Proposition~\ref{prop:LaplacePrinciple}, gives for any $r,q>0$ the bound
	\begin{align}
	\label{eq:proof:thm:error_decomposition_MMS_CH:23}
	\begin{aligned}
		\N{x^{\mathrm{CH}}_{k} - \widetilde{x}^{\mathrm{CH}}_{k}}_2
		= \Nbig{\conspointE{\widetilde\CE_{k}}{\nu_k} - \widetilde{x}^{\mathrm{CH}}_{k}}_2
		\leq \frac{\big(q + (\widetilde\CE_{k})_r\big)^\nu}{\eta} + \frac{\exp(-\alpha q)}{\nu_k\big(B_r(\widetilde{x}^{\mathrm{CH}}_{k})\big)} \int \N{x-\widetilde{x}^{\mathrm{CH}}_{k}}_2 d\nu_k(x),
	\end{aligned}
	\end{align}
	where $(\widetilde\CE_{k})_r := \sup_{x\in B_r(\widetilde{x}^{\mathrm{CH}}_{k})} \widetilde\CE_{k}(x) - \widetilde\CE_{k}(\widetilde{x}^{\mathrm{CH}}_{k})$.
	We further notice that by the assumption $\tau<1/(-2\Lambda)$ if $\Lambda<0$ it holds $\eta\geq1/(2\sqrt{\tau})$ (in the case $\Lambda\geq0$ the same bound holds trivially).
	Combining~\eqref{eq:proof:thm:error_decomposition_MMS_CH:23} with the technical estimates of Lemma~\ref{lem:auxiliary_estimates_Laplace} and the definition of the cutoff function~$\mathbbm{1}_{\widetilde\Omega_M}$ allows to obtain
	\begin{align}
	\label{eq:proof:thm:error_decomposition_MMS_CH:31}
	\begin{aligned}
		&\bbE\N{x^{\mathrm{CH}}_{k} - \widetilde{x}^{\mathrm{CH}}_{k}}_2^2\mathbbm{1}_{\widetilde\Omega_M}\\
		&\quad\leq 2\bbE\left[\frac{\big(q + (\widetilde\CE_{k})_r\big)}{\eta^2}\mathbbm{1}_{\widetilde\Omega_M}\right] + 2\bbE\left[\frac{\exp(-2\alpha q)}{\nu_k\big(B_r(\widetilde{x}^{\mathrm{CH}}_{k})\big)^2} \left(\int \N{x-\widetilde{x}^{\mathrm{CH}}_{k}}_2 d\nu_k(x)\right)^2\mathbbm{1}_{\widetilde\Omega_M}\right]\\
		&\quad\leq 8\tau\Big(q + \left(\tfrac{r}{2\tau} + C_1 + C_1r + 6C_1M \right)r\Big)\\
        &\quad\quad\;\!+ 4\exp\left(-2\alpha q\!+\!\frac{1}{\widetilde\sigma^2}\left(r^2\!+\!12\tau^2 C_1^2(1\!+\!2M^2)\right)\!\right)\!\frac{2^d(2\widetilde\sigma^2)^{d}}{r^{2d}} \Gamma\!\left(\tfrac{d}{2}\!+\!1\right)^2\! \left(4\tau^2 C_1^2(1\!+\!2M)^2 \!+\! 2d\widetilde{\sigma}^2\right)\\
		&\quad= 8\tau\Big(q + \left(\tfrac{r}{2\tau} + C_1 + C_1r + 6C_1M \right)r\Big)\\
        &\quad\quad\;\!+ 4\exp\left(-2\alpha \left(q\!-\!\left(\frac{r^2}{\tau}\!+\!12\tau C_1^2(1\!+\!2M^2)\right)\right)\right) \!\frac{2^d\tau^{d}}{\alpha^dr^{2d}} \Gamma\!\left(\tfrac{d}{2}\!+\!1\right)^2 \!\left(4\tau^2 C_1^2(1\!+\!2M)^2 \!+\! d\frac{\tau}{\alpha}\right),
	\end{aligned}
	\end{align}
	where in the last step we just replaced $2\widetilde\sigma^2$ by $\tau/\alpha$ according to the relation.
	We now choose
 	\begin{align*}
		r = \tau,
		\ \ \ \text{ }
		q = \frac{3}{2}\tau \!+\! 12\tau C_1^2(1\!+\!2M^2)
		\ \ \ \text{ and }\ \ \ 
		\alpha\geq\alpha_0
		:= \frac{1}{\tau}\Big(d\log2 \!+\! \log(1\!+\!d) \!+\! 2\log\Gamma\!\left(\tfrac{d}{2}\!+\!1\right)\!\Big),
	\end{align*}
    where $\Gamma$ denotes Euler's gamma function, for which, by Stirling's approximation, it holds $\Gamma\left(x+1\right) \sim \sqrt{2\pi x}\left(x/e\right)^x$ as $x\rightarrow\infty$.
	With this we can continue the computations of~\eqref{eq:proof:thm:error_decomposition_MMS_CH:31} with
	\begin{align}
		\label{eq:proof:thm:error_decomposition_MMS_CH:33}
	\begin{aligned}
		\bbE\N{x^{\mathrm{CH}}_{k} - \widetilde{x}^{\mathrm{CH}}_{k}}_2^2\mathbbm{1}_{\widetilde\Omega_M}
		&\leq 8 \Big(2 + C_1 + C_1\tau + 6C_1M + 12C_1^2(1+2M^2)\Big) \tau^2 \\
		&\quad\;\! + 4\exp\left(-\alpha \tau\right) \frac{2^d}{\alpha^d\tau^{d}} \Gamma\!\left(\tfrac{d}{2}+1\right)^2 \left(4\tau^2 C_1^2(1+2M^2) + d\frac{\tau}{\alpha}\right)\\
		&\leq 8 \Big(3 + C_1 + C_1\tau + 6C_1M + 14C_1^2(1+M^2)\Big) \tau^2 \\
        &\leq c\tau^2
	\end{aligned}
	\end{align}
    with a constant $c=c(C_1,M)$.
	Notice that to obtain the next-to-last inequality one may first note and exploit that one has $\alpha\tau\geq1$ as well as $1/\alpha\leq\tau$ as a consequence of $\alpha\geq1/\tau$.

    \textbf{Probabilistic formulation:}
	We first note that with Markov's inequality we have the estimate
	\begin{align*}
	\begin{aligned}
		\bbP\big(\widetilde\Omega_M^c\big)
		&= \bbP\left(\max_{k=0,\dots,K} \max\left\{\N{x^{\mathrm{CH}}_{k}}_2,\Nbig{\widetilde{x}^{\mathrm{CH}}_{k}}_2\right\}   > M \right) \\
		&\leq \frac{1}{M^4} \left(\bbE\max_{k=0,\dots,K} \N{x^{\mathrm{CH}}_{k}}_2^4+\bbE\max_{k=0,\dots,K}\Nbig{\widetilde{x}^{\mathrm{CH}}_{k}}_2^4\right)\leq \frac{1}{M^4} \big(\CM^{\mathrm{CH}}+\widetilde{\CM}^{\mathrm{CH}}\;\!\big),
	\end{aligned}
	\end{align*}
	where the last inequality is due to Lemmas~\ref{lem:boundedness_CH} and~\ref{lem:boundedness_auxiliaryMMS}.
	Thus, for any $\delta\in(0,1/2)$, a sufficiently large choice $M = M(\delta^{-1},\CM^{\mathrm{CH}},\widetilde{\CM}^{\mathrm{CH}})$ allows to ensure $\bbP\big(\widetilde\Omega_M^c\big)\leq \delta$.
	To conclude the proof, let us denote by $\widetilde K_\varepsilon\subset\Omega$ the set, where~\eqref{eq:bound_CH_CH_aux} does not hold and abbreviate
	\begin{align*}
		\epsilon
		= \varepsilon^{-1} c\tau^2.
	\end{align*}
	For the probability of this set we can estimate
	\begin{align}
	\begin{aligned}
		\bbP\big(\widetilde K_\varepsilon\big)
		&= \bbP\big(\widetilde K_\varepsilon\cap\widetilde\Omega_M\big) + \bbP\big(\widetilde K_\varepsilon\cap\widetilde\Omega_M^c\big)
		\leq \bbP\big(\widetilde K_\varepsilon\,\big|\,\widetilde\Omega_M\big)\,\bbP\big(\widetilde\Omega_M\big) + \bbP\big(\widetilde\Omega_M^c\big) \\
		&\leq \bbP\big(\widetilde K_\varepsilon\,\big|\,\widetilde\Omega_M\big) + \delta
		\leq \epsilon^{-1}\,\bbE\left[\N{x^{\mathrm{CH}}_{k}-\widetilde{x}^{\mathrm{CH}}_{k}}_2^2 \,\Big|\, \widetilde\Omega_M\right] + \delta,
	\end{aligned}
	\end{align}
	where the last step is due to Markov's inequality.
    By definition of the conditional expectation we further have
	\begin{align*}
	\begin{aligned}
		\bbE\left[\N{x^{\mathrm{CH}}_{k}-\widetilde{x}^{\mathrm{CH}}_{k}}_2^2 \,\Big|\, \widetilde\Omega_M\right]
		\leq \frac{1}{\bbP\big(\widetilde\Omega_M\big)} \bbE\N{x^{\mathrm{CH}}_{k}-\widetilde{x}^{\mathrm{CH}}_{k}}_2^2\mathbbm{1}_{\widetilde\Omega_M}
		\leq 2\bbE\N{x^{\mathrm{CH}}_{k}-\widetilde{x}^{\mathrm{CH}}_{k}}_2^2\mathbbm{1}_{\widetilde\Omega_M}.
	\end{aligned}
	\end{align*}
	Inserting now the expression from \eqref{eq:proof:thm:error_decomposition_MMS_CH:33} concludes the proof.
\end{proof}

As an immediate consequence of Proposition~\ref{prop:relaxation_CH_GF}, which can be derived by combining the former statement with a stability argument for the MMS and applying Gr\"onwall's inequality (a detailed proof is deferred to Appendix~\ref{subsec:appendix:proof:thm:relaxation_CH_GF}),
we are able to control the divergence between the CH scheme~\eqref{eq:CH} and the MMS~\eqref{eq:MMS}.
Given that the CH scheme resembles several Monte Carlo-inspired evolution strategies, such as CMA-ES~\cite{hansen2001completely},
which are commonly believed to behave GD-like in some scenarios~\cite{lorenzo2023chat,ollivier2017information}, this observation may be of independent interest and provide an explanation for such folklore, see also \cite{fornasier2026consensus}.

\begin{theorem}[CH relaxes a gradient flow] \label{thm:relaxation_CH_GF}
	Fix $\varepsilon>0$ and $\delta\in(0,1/2)$.
	Let $\CE\in\CC(\bbR^d)$ satisfy \ref{asm:zero_global}\;\!--\,\ref{asm:lambda-convex}.
	We denote by $(x^{\mathrm{CH}}_{k})_{k=0,\dots,K}$ the iterations of the CH scheme~\eqref{eq:CH} and by $(x^{\mathrm{MMS}}_{k})_{k=0,\dots,K}$ the ones of the MMS~\eqref{eq:MMS}.
	Moreover, assume that the parameters~$\alpha,\tau$ and $\widetilde\sigma$ are such that $\tau<1/(-2\Lambda)$ if $\Lambda<0$,  $\alpha\gtrsim\frac{1}{\tau}d\log d$ is sufficiently large and $\widetilde\sigma^2=\tau/(2\alpha)$.
	Then, with probability larger than $1 - \left(\delta+\varepsilon\right)$, it holds for all $k=1,\dots,K$ that 
	\begin{equation}
		\label{eq:bound_CH_GF}
		\N{x^{\mathrm{CH}}_{k}\!-\!x^{\mathrm{MMS}}_{k}}_2^2
		\leq \varepsilon^{-1} c(1\!+\!\vartheta^{-1})\,\tau^2 \sum_{\ell=0}^{k-1} \left(\frac{1\!+\!\vartheta}{\left(1\!+\!\tau\Lambda\right)^2}\right)^\ell
	\end{equation}
	for any $\vartheta\in(0,1)$ and with $c=c(\delta^{-1},C_1,\CM)$.
\end{theorem}

It is crucial to notice that the right-hand side of \eqref{eq:bound_CH_GF} must necessarily not vanish in the nonconvex case $\Lambda<0$.
This is due to the CH scheme's ability to overcome, unlike GD, local energy barriers, see Figure~\ref{fig:CH_GrandCanyon3noisy}.
Yet, as the objective function becomes more and more convex ($\Lambda>0$ becoming increasingly larger) the trajectories of the CH scheme~\eqref{eq:CH} and the MMS~\eqref{eq:MMS} become closer and closer.
This can be observed in the following corollary, which merely explicitly bounds the right-hand side of \eqref{eq:bound_CH_GF} by a geometric series in the case where $\Lambda>0$.
Let us mention, however, that the strength of both statements would correlate with and benefit from the afore-addressed refinement of the quantitative Laplace principle, allowing, in particular, to obtain a right-hand side in \eqref{eq:bound_CH_GF_convex} that can be made arbitrarily small as $\tau\rightarrow0$.

\begin{corollary}
	Fix $\varepsilon>0$ and $\delta\in(0,1/2)$. 
	Let $\CE\in\CC(\bbR^d)$ satisfy \ref{asm:zero_global}\;\!--\,\ref{asm:lambda-convex} with $\Lambda>0$.
	Then, in the setting of Theorem~\ref{thm:relaxation_CH_GF} and with probability larger than $1 - \left(\delta+\varepsilon\right)$, it holds for all $k=1,\dots,K$ that 
	\begin{equation}
		\label{eq:bound_CH_GF_convex}
		\N{x^{\mathrm{CH}}_{k}\!-\!x^{\mathrm{MMS}}_{k}}_2^2
		\leq \varepsilon^{-1} c(1\!+\!\vartheta^{-1})\tau^2 \frac{\left(1\!+\!\tau\Lambda\right)^2}{\left(1\!+\!\tau\Lambda\right)^2\!-\!(1\!+\!\vartheta)}.
	\end{equation}
\end{corollary}

\section{Conclusions} \label{sec:conclusions}

In this paper, we provided a novel analytical perspective on the theoretical understanding of gradient-based learning algorithms 
by showing that consensus-based optimization~(CBO), an intrinsically derivative-free optimization method guaranteed to globally converge to global minimizers of potentially nonsmooth and nonconvex loss functions, implicitly behaves like a gradient-based method.
This allows to interpret CBO as a stochastic relaxation of gradient descent.
Besides forging such unexpected link and thereby driving forward our theoretical understanding of both gradient-based learning methods and metaheuristic black-box optimization algorithms,
we widen the scope of applications of methods which\,---\,in one way or another, be it explicitly or implicitly\,---\,estimate and exploit gradients.
At the example of CBO, we show that stochastic perturbations of gradient descent (other than the ones of stochastic gradient descent or the Langevin dynamics) exist which provably allow to overcome energy barriers and reach deep levels of nonconvex functions.
Our theoretical main result, Theorem~\ref{thm:main_informal}, together with the global convergence guarantees of CBO, Theorem~\ref{thm:Global_CBO_convergence}, suggests that choosing a drift parameter $\lambda>0$ relatively small compared to $1/\Delta t$, a non-insignificant noise parameter $\sigma>0$ but such that $2\lambda>\sigma^2$ holds,\footnote{Experimental evidence across the literature suggests that it is typically best to start with even larger $\sigma$ and decreasing it during the algorithm until $2\lambda>\sigma^2$.} as well as a moderate value of the weight parameter $\alpha>\alpha_0$ as hyperparameters of CBO lead to capable stochastic perturbations.
With this specific perturbation being derivative-free,
we believe these insights to bear the potential for designing efficient and reliable training methods which behave like first-order methods while not relying on the ability of computing gradients.
Potential areas of application in machine learning may include the usage of nonsmooth losses, hyperparameter tuning, convex bandits, reinforcement learning, the training of sparse and pruned neural networks, or federated learning.

An analogous analysis approach may be carried over to second-order methods (with momentum), allowing to establish a link between Adam~\cite{adam2015} and the well-known particle swarm optimization method~\cite{kennedy1995particle}, which is related to CBO through a zero-inertia limit~\cite{grassi2021particle,cipriani2021zero}.
Together with recent observations~\cite{lorenzo2023chat} based on tools from kinetic theory that simulated annealing~\cite{kirkpatrick1983optimization,geman1986diffusions,holley1988simulated} is related to the Langevin dynamics~\cite{chiang1987diffusion,roberts1996exponential,durmus2017nonasymptotic},
this would strengthen even further the surprising and yet largely unexplored link between gradient-based learning algorithms and derivative-free metaheuristic optimization methods.
Beyond that we envisage the likely connections between consensus-based sampling~\cite{carrillo2022consensus} and log-concave sampling~or sampling by Langevin flows~\cite{frieze1994sampling,bernton2018langevin,dwivedi2018log,lee2021structured}.

\section*{Acknowledgments} 
The authors would like to profusely thank Hui Huang, Giuseppe Savar\'e, and Alessandro Scagliotti for many fruitful and stimulating discussions about the topic.

This work has been funded by the German Federal Ministry of Education and Research and the Bavarian State Ministry for Science and the Arts.
The authors of this work take full responsibility for its content.

\bibliographystyle{abbrv}
\bibliography{biblio.bib}

\newpage

\section*{Supplemental Material}

\appendix

This supplemental material is organized into the following appendices.

\begin{itemize}
	\item Appendix~\ref{sec:appendix:IntroductoryFacts}: Introductory facts
    \item Appendix~\ref{sec:appendix:BoundednessNumericalSchemes}: Boundedness of the numerical schemes
	\item Appendix~\ref{sec:appendix:thm:relaxation_CBO_CH}: Proof details for Theorem~\ref{thm:relaxation_CBO_CH}
	\item Appendix~\ref{sec:appendix:thm:relaxation_CH_GF}: Proof details for Proposition~\ref{prop:relaxation_CH_GF} and Theorem~\ref{thm:relaxation_CH_GF}
	\item Appendix~\ref{sec:appendix:additional_numerics}: Additional numerical experiments
\end{itemize}

\section{Introductory facts} \label{sec:appendix:IntroductoryFacts}

\paragraph{Notation}
To keep the notation concise, we hide generic constants, i.e., we write $a\lesssim b$ for $a\leq cb$, if $c$ is a constant independent of problem-dependent constants.
Moreover, since we work with random variables in several instances, many equalities and inequalities hold almost surely without being mentioned explicitly.
We abbreviate with i.i.d.\@ independently and identically distributed.

We write $\N{\;\!\dummy\;\!}_2$ and $\langle\;\!\dummy\;\!,\;\!\dummy\;\!\rangle$ for the Euclidean norm and scalar product on $\bbR^d$, respectively.
Euclidean balls are denoted by $B_{r}(x) := \{z \in \bbR^d: \Nnormal{z-x}_2 \leq r\}$.
Moreover, we write $\N{\;\!\dummy\;\!}_\infty$ for the $\ell^\infty$-norm and denote the associated $\ell^\infty$-balls by $B^\infty_{r}(x) := \{z \in \bbR^d: \Nnormal{z-x}_\infty \leq r\}$.

For the space of continuous functions~$f:X\rightarrow Y$ we write $\CC(X,Y)$, with $X\subset\bbR^n$ and a suitable topological space $Y$.
For an open set $X\subset\bbR^n$ and for $Y=\bbR^m$ the space~$\CC^k(X,Y)$ contains functions~$f\in\CC(X,Y)$ that are $k$-times continuously differentiable.
We omit $Y$ in the real-valued case, i.e., $\CC(X)=\CC(X,\bbR)$ and $\CC^k(X)=\CC^k(X,\bbR)$.

The operator $\nabla$ denotes the gradient of a function on~$\bbR^d$.

\paragraph{Convex analysis}
For a convex function~$f\in\CC(\bbR^d)$ the subdifferential~$\partial f(x)$ at a point $x\in\bbR^d$ is the set
\begin{align*}
	\partial f(x) = \left\{p\in\bbR^d: f(y) \geq f(x) + \left\langle p, y-x\right\rangle \text{ for all } y\in\bbR^d\right\}.
\end{align*}
In the setting~$f\in\CC(\bbR^d)$, $\partial f(x)$ is closed, convex, nonempty and bounded.
If $f\in\CC^1(\bbR^d)$, $\partial f(x) = \{\nabla f(x)\}$.
Moreover, it is straightforward to verify that for $x_1,x_2,p_1,p_2\in\bbR^d$ with $p_1\in\partial f(x_1)$ and $p_2\in\partial f(x_2)$ it holds $\left\langle p_1-p_2, x_1-x_2 \right\rangle \geq 0$.

\paragraph{Probability measures}
The set of all Borel probability measures over $\bbR^d$ is denoted by $\CP(\bbR^d)$.
For $p>0$, we collect measures~$\indivmeasure \in \CP(\bbR^d)$ with finite $p$-th moment $\int \N{x}_2^pd\indivmeasure(x)$ in $\CP_p(\bbR^d)$.

The Dirac delta $\delta_x$ for a point $x\in\bbR^d$ is a measure satisfying $\delta(B) = 1$ if $x\in B$ and $\delta(B) = 0$ if $x\not\in B$ for any measurable set $B\subset\bbR^d$.

\paragraph{Wasserstein distance}
For any $1\leq p<\infty$, the \mbox{Wasserstein-$p$} distance between two Borel probability measures~$\indivmeasure,\indivmeasure'\in\CP_p(\bbR^d)$ is defined by
\begin{align} \label{def:wassersteindistance}
	W_p(\indivmeasure,\indivmeasure') = \left(\inf_{\gamma\in\Pi(\indivmeasure,\indivmeasure')}\int\N{x-x'}_2^pd\gamma(x,x')\right)^{1/p},
\end{align}
where $\Pi(\indivmeasure,\indivmeasure')$ denotes the set of all couplings of (a.k.a.\@ transport plans between) $\indivmeasure$ and $\indivmeasure'$, i.e., the collection of all Borel probability measures over $\mathbb{R}^d\times\mathbb{R}^d$ with marginals $\indivmeasure$ and $\indivmeasure'$ on the first and second component, respectively, see, e.g., \cite{savare2008gradientflows,villani20090oldandnew}.
$\CP_p(\bbR^d)$ endowed with the \mbox{Wasserstein-$p$} distance~$W_p$ is a complete separable metric space~\cite[Proposition~7.1.5]{savare2008gradientflows}.

\paragraph{A generalized triangle-type inequality}
It holds for $p,J\in\bbN$ by H\"older's inequality
\begin{equation} \label{eq:triangel_inequality_general}
	\abs{\sum_{j=1}^J a_j}^p
	\leq J^{p-1}\sum_{j=1}^J \abs{a_j}^p.
\end{equation}

\paragraph{A discrete variant of Gr\"onwall's inequality}

If $z_k \leq az_{k-1} + b$ with $a,b\geq0$ for all $k\geq1$, then
\begin{equation}
	\label{eq:gronwall_inequality}
\begin{split}
	z_k\leq a^kz_0 + b\sum_{\ell=0}^{k-1}a^\ell
		\leq a^kz_0 + b\prod_{\ell=1}^{k-1}(1+a)
		\leq a^kz_0 + be^{a(k-1)}
\end{split}
\end{equation}
for all $k\geq1$.
Notice that, while the first inequality in~\eqref{eq:gronwall_inequality} is as sharp as the initial estimates, the remaining two inequalities are rather rough upper bounds.

\section{Boundedness of the numerical schemes} \label{sec:appendix:BoundednessNumericalSchemes}

Before showing the boundedness in expectation of the numerical schemes~\eqref{eq:CBO}, \eqref{eq:CH}, \eqref{eq:MMS} and \eqref{eq:CH_aux} over time in Sections~\ref{subsec:appendix:boundednessCBO}\;\!--\,\ref{subsec:appendix:boundednessAuxiliaryMMS}, respectively, let us first recall from~\cite[Lemma~3.3]{carrillo2018analytical} an estimate on the consensus point~\eqref{eq:consensus_point}, which facilitates the subsequent proofs.

\begin{lemma}[Boundedness of consensus point~$\conspointnoarg$]
	\label{lem:boundedness_consensuspoint}
	Let $\CE\in\CC(\bbR^d)$ satisfy \ref{asm:zero_global}\;\!--\,\ref{asm:quadratic_growth}.
	Moreover, let $\varrho\in\CP_2(\bbR^d)$.
	Then it holds
	\begin{equation*}
		\N{\conspoint{\varrho}}_2^2
		\leq b_1 + b_2 \int \N{x}_2^2 d\varrho(x)
	\end{equation*}
	with constants $b_1=0$ and $b_2 = b_2(\alpha,\underbar\CE,\overbar\CE)>0$ in case the first condition of \ref{asm:quadratic_growth} holds and with $b_i = b_i(\alpha,C_2,C_3,C_4)>0$ for $i=1,2$ as given in~\eqref{eq:constants_b1b2} in case of the second condition of \ref{asm:quadratic_growth}.
\end{lemma}

\begin{proof}
	In case the first condition of \ref{asm:quadratic_growth} holds, we have by definition of the consensus point~$\conspointnoarg$ in~\eqref{eq:consensus_point} and Jensen's inequality
	\begin{align*}
		\N{\conspoint{\varrho}}_2^2
		\leq \int \N{x}_2^2 \frac{\omegaa(x)}{\N{\omegaa}_{L_1(\varrho)}}d\varrho(x)
		\leq e^{\alpha(\overbarscript\CE-\underbarscript\CE)} \int \N{x}_2^2 d\varrho(x).
	\end{align*}
	In case of the second condition of \ref{asm:quadratic_growth}, the statement follows from \cite[Lemma~3.3]{carrillo2018analytical} with constants
	\begin{align}
		\label{eq:constants_b1b2}
		b_1=C_4^2+b_2
		\quad\text{and}\quad
		b_2=2\frac{C_2}{C_3}\left(1+\frac{1}{\alpha C_3}\frac{1}{C_4^2}\right)\!,
	\end{align}
	which concludes the proof.
\end{proof}

With this estimate we have all necessary tools at hand to prove the boundedness of the numerical schemes investigated in this paper.

\subsection{Boundedness of the consensus-based optimization~(CBO) dynamics~\eqref{eq:CBO_dynamics} and~\eqref{eq:CBO}}
\label{subsec:appendix:boundednessCBO}

Let us remind the reader that the iterates~$(x^{\mathrm{CBO}}_{k})_{k=0,\dots,K}$ of the consensus-based optimization~(CBO) scheme~\eqref{eq:CBO} are defined by
\begin{align*}
\begin{aligned}
	x^{\mathrm{CBO}}_{k}
	&= \conspoint{\empmeasure{k}},
	\quad\text{ with }\quad
    \empmeasure{k} = \frac{1}{N}\sum_{i=1}^N \delta_{X_{k}^i}, \\
	x^{\mathrm{CBO}}_{0}
	&=x_0\sim\rho_0,
\end{aligned}
\end{align*}
where the iterates~$\big((X_{k}^i)_{k=0,\dots,K}\big)_{i=1,\dots,N}$ are given as in \eqref{eq:CBO_dynamics}  by 
\begin{align*}
	X_{k}^i
	&= X_{k-1}^i - \Delta t\lambda \left(X_{k-1}^i - \conspoint{\empmeasure{k-1}}\right) + \sigma D\!\left(X_{k-1}^i-\conspoint{\empmeasure{k-1}}\right)  B_{k}^i,\\
	X_{0}^i
	&= x_{0}^i \sim \rho_0
\end{align*}
with $B_{k}^i$ being i.i.d.\@ Gaussian random vectors in~$\bbR^d$ with zero mean and covariance matrix $\Delta t\Id$ for $k=0,\dots,K$ and $i=1,\dots,N$, i.e., $B_{k}^i\sim\CN(0,\Delta t\Id)$.

\begin{lemma}[Boundedness of the CBO dynamics~\eqref{eq:CBO_dynamics} and the CBO scheme~\eqref{eq:CBO}]
	\label{lem:boundedness_CBOdynamics}
	Let $\CE\in\CC(\bbR^d)$ satisfy \ref{asm:zero_global}\;\!--\,\ref{asm:quadratic_growth}.
	Moreover, let $\rho_0\in\CP_4(\bbR^d)$.
	Then, for the empirical random measures~$(\empmeasure{k})_{k=0,\dots,K}$ and the iterates $(X_{k}^i)_{k=0,\dots,K}$ of~\eqref{eq:CBO_dynamics} it holds
	\begin{equation*}
		\bbE\max_{k=0,\dots,K} \int\N{x}_2^4 d\empmeasure{k}(x)
		\leq \CM^{\mathrm{CBO}}
		\quad\text{ and }\quad
		\max_{i=1,\dots,N}  \bbE\max_{k=0,\dots,K} \N{X_{k}^i}_2^4
		\leq \CM^{\mathrm{CBO}}
	\end{equation*}
	with a constant $\CM^{\mathrm{CBO}}=\CM^{\mathrm{CBO}}(\lambda, \sigma, d, b_1, b_2, K\Delta t, K, \rho_0)>0$.
	Moreover, for the iterates~$(x^{\mathrm{CBO}}_{k})_{k=0,\dots,K}$ of~\eqref{eq:CBO} it holds
	\begin{equation*}
		\bbE \max_{k=0,\dots,K} \N{x^{\mathrm{CBO}}_{k}}_2^4
		\leq \CM^{\mathrm{CBO}}.
	\end{equation*}
\end{lemma}

\begin{proof}
	We first note that $X_{k}^i$ as defined iteratively in~\eqref{eq:CBO_dynamics} satisfies
	\begin{align*}
		X_{k}^i
		= X_{0}^i - \Delta t\lambda \sum_{\ell=1}^k \left(X_{\ell-1}^i - \conspoint{\empmeasure{\ell-1}}\right) + \sigma \sum_{\ell=1}^k D\!\left(X_{\ell-1}^i-\conspoint{\empmeasure{\ell-1}}\right)  B_{\ell}^i
	\end{align*}
	and that for any $k=1,\dots,K$ by means of the standard inequality~\eqref{eq:triangel_inequality_general} for $p=4$ and $J=3$ we have
	\begin{align} \label{eq:proof:lem:boundedness_CBOdynamics_3}
	\begin{aligned}
		\max_{\ell=0,\dots,k} \N{X_{\ell}^i}_2^4
		\lesssim
			\N{X_{0}^i}_2^4
			& + (\Delta t\lambda)^4 \max_{\ell=1,\dots,k} \N{\sum_{m=1}^\ell \left(X_{m-1}^i - \conspoint{\empmeasure{m-1}}\right)}_2^4\\
			& + \sigma^4 \max_{\ell=1,\dots,k} \N{\sum_{m=1}^\ell D\!\left(X_{m-1}^i-\conspoint{\empmeasure{m-1}}\right)  B_{m}^i}_2^4.
	\end{aligned}
	\end{align}
	Noticing that the random process~$Y_\ell^i := \sum_{m=1}^\ell D\!\left(X_{m-1}^i-\conspoint{\empmeasure{m-1}}\right) B_{m}^i$, $\ell=0,\dots,k$ is a martingale w.r.t.\@ the filtration~$\big\{\CF_\ell=\sigma\left(\{X^i_0\}\cup\{B_{m}^i,m=1,\dots,\ell\}\right)\!\big\}_{\ell=0}^{k-1}$ since it satisfies $\bbE\left[Y_\ell^i\mid\CF_{\ell-1}\right]=Y_{\ell-1}^i$ for $\ell=1,\dots,k$,
	we can apply a discrete version of the Burkholder-Davis-Gundy inequality~\cite[Corollary~11.2.1]{chow2003probability} yielding
	\begin{align*}
		\bbE\! \max_{\ell=1,\dots,k} \N{\sum_{m=1}^\ell D\!\left(X_{m-1}^i\!-\!\conspoint{\empmeasure{m-1}}\right)  B_{m}^i}_2^4
		&\lesssim d\,\bbE\sum_{j=1}^d \left(\sum_{\ell=1}^k \left(D\!\left(X_{\ell-1}^i\!-\!\conspoint{\empmeasure{\ell-1}}\right)\right)_{jj}^2 (B_{\ell}^i)_{j}^2\right)^{2}\!.
	\end{align*}
	Thus, when taking the expectation on both sides of~\eqref{eq:proof:lem:boundedness_CBOdynamics_3} and employing Jensen's inequality, we can use the latter to obtain
	\begin{align}
		\label{eq:proof:lem:boundedness_CBOdynamics_5}
	\begin{aligned}
		\bbE \max_{\ell=0,\dots,k} \N{X_{\ell}^i}_2^4
		&\lesssim
			\bbE\N{X_{0}^i}_2^4
			+ (\Delta t\lambda)^4K^3 \, \bbE \sum_{\ell=1}^k \N{X_{\ell-1}^i - \conspoint{\empmeasure{\ell-1}}}_2^4\\
			& \qquad\qquad\;\;\;\, + \sigma^4dK\,\bbE\sum_{j=1}^d \sum_{\ell=1}^k \left(D\!\left(X_{\ell-1}^i-\conspoint{\empmeasure{\ell-1}}\right)\right)_{jj}^4 (B_{\ell}^i)_{j}^4 \\
		&\lesssim
			\bbE\N{X_{0}^i}_2^4
			+ (\Delta t\lambda)^4 K^3 \, \bbE \sum_{\ell=1}^k \left(\N{X_{\ell-1}^i}_2^4 + \N{\conspoint{\empmeasure{\ell-1}}}_2^4\right)\\
			& \qquad\qquad\;\;\;\, + (\Delta t)^2\sigma^4dK\,\bbE\sum_{j=1}^d \sum_{\ell=1}^k \left(\left(X_{\ell-1}^i\right)_{j}^4+\left(\conspoint{\empmeasure{\ell-1}}\right)_{j}^4\right) \\
		&\lesssim
			\left(1 + (\Delta t\lambda)^4 K^3 + (\Delta t\sigma^2d)^2K\right) \, \bbE \sum_{\ell=1}^k \left(\N{X_{\ell-1}^i}_2^4 + \N{\conspoint{\empmeasure{\ell-1}}}_2^4\right) \\
		&\lesssim
			\left(1 + \lambda^4 (K\Delta t)^4 + \sigma^4d^2(K\Delta t)^2\right) \, \bbE \max_{\ell=1,\dots,k} \left(\N{X_{\ell-1}^i}_2^4 + \N{\conspoint{\empmeasure{\ell-1}}}_2^4\right) \\
		&\leq
			C\, \bbE \max_{\ell=1,\dots,k} \left(\N{X_{\ell-1}^i}_2^4 + b_1^2 + b_2^2 \int \N{x}_2^4 d\empmeasure{\ell-1}(x)\right)
	\end{aligned}
	\end{align}
	with a constant $C=C(\lambda, \sigma, d, K\Delta t)$.
	In the second step we made use of the standard inequality~\eqref{eq:triangel_inequality_general} for $p=4$ and $J=2$, exploited that $B_{\ell}^i$ is independent from $D\!\left(X_{\ell-1}^i-\conspoint{\empmeasure{\ell-1}}\right)$ for any $\ell=1,\dots,k$ and used that the fourth moment of a Gaussian random variable~$B\sim\CN(0,1)$ is $\bbE B^4 = 3$ (e.g., by recalling that $\bbE B^4 = \frac{d^4}{dx^4} M_B(x)\big|_{x=0}$, where $M_B$ denotes the moment-generating function of $B$).
	Moreover, recall that $K\Delta t$ denotes the final time horizon, and note that the last step is due to Lemma~\ref{lem:boundedness_consensuspoint}.
	Averaging~\eqref{eq:proof:lem:boundedness_CBOdynamics_5} over $i$ allows to bound
	\begin{align}
	\label{eq:proof:lem:boundedness_CBOdynamics_7}
	\begin{aligned}
		\frac{1}{N} \sum_{i=1}^N \bbE \max_{\ell=0,\dots,k} \N{X_{\ell}^i}_2^4
		&\leq \widetilde{C} \left(1 + \frac{1}{N} \sum_{i=1}^N \bbE \max_{\ell=1,\dots,k} \N{X_{\ell-1}^i}_2^4 \right)
	\end{aligned}
	\end{align}
	with a constant $\widetilde{C}=\widetilde{C}(\lambda, \sigma, d, b_1, b_2, K\Delta t)$.
	Since $\bbE\int\N{x}_2^4 d\empmeasure{0}(x) = \frac{1}{N}\sum_{i=1}^N \bbE\,\Nnormal{x_0^i}_2^4$,
	an application of the discrete variant of Gr\"onwall's inequality~\eqref{eq:gronwall_inequality} yields the second inequality in
	\begin{align}
	\label{eq:proof:lem:boundedness_CBOdynamics_9}
	\begin{aligned}
		\bbE\max_{\ell=0,\dots,k} \int\N{x}_2^4 d\empmeasure{\ell}(x)
		&\leq\frac{1}{N} \sum_{i=1}^N \bbE \max_{\ell=0,\dots,k} \N{X_{\ell}^i}_2^4 \\
		&\leq \widetilde{C}^{k} \, \bbE\int\N{x}_2^4 d\empmeasure{0}(x) + \widetilde{C}e^{\widetilde{C}(k-1)},
	\end{aligned}
	\end{align}
	showing that the left-hand side is bounded independently of $N$, which gives the first bound in the first part of the statement.
	Making use thereof in \eqref{eq:proof:lem:boundedness_CBOdynamics_5} also yields the second part after another application of Gr\"onwall's inequality.
	The second part of the statement follows by noting that an application of Lemma~\ref{lem:boundedness_consensuspoint} gives
	\begin{align*}
		\bbE \max_{\ell=1,\dots,k} \N{x^{\mathrm{CBO}}_{\ell}}_2^4
		&= \bbE \max_{\ell=1,\dots,k} \N{\conspoint{\empmeasure{\ell}}}_2^4 \\
		&\leq 2b_1^2 + 2b_2^2 \, \bbE \max_{\ell=1,\dots,k} \int \N{x}_2^4 d\empmeasure{\ell}(x),
	\end{align*}
	where the last expression is bounded as in \eqref{eq:proof:lem:boundedness_CBOdynamics_9}.
	Recalling that $x^{\mathrm{CBO}}_0=x_0\sim\rho_0\in\CP_4(\bbR^d)$ and 
	choosing the constant~$\CM^{\mathrm{CBO}}$ large enough for all three estimates to hold with $k=K$ concludes the proof.
\end{proof}

\subsection{Boundedness of the consensus hopping scheme~\eqref{eq:CH}} \label{subsec:appendix:boundednessCH}

Let us recall that the iterates~$(x^{\mathrm{CH}}_{k})_{k=0,\dots,K}$ of the consensus hopping~(CH) scheme~\eqref{eq:CH} are defined by
\begin{align*}
\begin{aligned}
    x^{\mathrm{CH}}_{k}
	&= \conspoint{\mu_{k}},
	\quad\text{ with }\quad
    \mu_{k} = \CN\!\left(x^{\mathrm{CH}}_{k-1},\widetilde{\sigma}^2\Id\right)\!, \\
	x^{\mathrm{CH}}_{0}
	&=x_0.
\end{aligned}
\end{align*}

\begin{lemma}[Boundedness of the CH scheme~\eqref{eq:CH}]
	\label{lem:boundedness_CH}
	Let $\CE\in\CC(\bbR^d)$ satisfy \ref{asm:zero_global}\;\!--\,\ref{asm:quadratic_growth}.
	Moreover, let $\rho_0\in\CP_4(\bbR^d)$.
	Then, for the random measures~$\left(\mu_{k}\right)_{k=1,\dots,K}$ in~\eqref{eq:CH} it holds
	\begin{equation*}
		\bbE \max_{k=1,\dots,K} \int\N{x}_2^4 d\mu_{k}(x)
		\leq \CM^{\mathrm{CH}}
	\end{equation*}
	with a constant $\CM^{\mathrm{CH}}=\CM^{\mathrm{CH}}(\widetilde\sigma,d,b_1,b_2,K,\rho_0)>0$.
	Moreover, for the iterates~$(x^{\mathrm{CH}}_{k})_{k=0,\dots,K}$ of~\eqref{eq:CH} it holds
	\begin{equation*}
		\bbE \max_{k=0,\dots,K} \N{x^{\mathrm{CH}}_{k}}_2^4
		\leq \CM^{\mathrm{CH}}.
	\end{equation*}
\end{lemma}

\begin{proof}
	According to the definition of the scheme~\eqref{eq:CH} and with the standard inequality~\eqref{eq:triangel_inequality_general} for $p=4$ and $J=2$, we observe that for any $k=2,\dots,K$ it holds 
	\begin{align*}
	\begin{aligned}
		\int\N{x}_2^4 d\mu_{k}(x)
		&= \int\N{x}_2^4 d\CN\!\left(x^{\mathrm{CH}}_{k-1},\widetilde{\sigma}^2\Id\right)(x) \\
		&\lesssim \N{x^{\mathrm{CH}}_{k-1}}_2^4 + \int\N{x}_2^4 d\CN\!\left(0,\widetilde{\sigma}^2\Id\right)(x) \\
		&= \N{\conspoint{\mu_{k-1}}}_2^4 + (d^2+2d)\,\widetilde{\sigma}^4 \\
		&\lesssim b_1^2 + b_2^2 \int \N{x}_2^4 d\mu_{k-1}(x) + d^2\widetilde{\sigma}^4,
	\end{aligned}
	\end{align*}
	where for the third step we explicitly computed that for the fourth moment of a multivariate Gaussian distribution it holds $\int \!\Nnormal{x}_2^4 \,d\CN\!\left(0,\Id\right)(x) = d^2+2d$.
	Moreover, in the final step we employed Lemma~\ref{lem:boundedness_consensuspoint} together with Jensen's inequality.
	Along the same lines we have $\int\!\Nnormal{x}_2^4\, d\mu_{1}(x) \lesssim \Nnormal{x_0}_2^4 + d^2\widetilde{\sigma}^4$.
	An application of the discrete variant of Gr\"onwall's inequality~\eqref{eq:gronwall_inequality} therefore allows to obtain
	\begin{align*}
	\begin{aligned}
		\int\N{x}_2^4 d\mu_{k}(x)
		&\lesssim b_2^{2k} \N{x_0}_2^4 + \left(b_1^2+d^2\widetilde\sigma^4\right)e^{cb_2^2(k-1)}
	\end{aligned}
	\end{align*}
	with a generic constant~$c>0$.
	Taking the maximum over the iterations~$k$ and the expectation w.r.t.\@ the initial condition~$\rho_0$ gives the first part of the statement.
	Recalling that $x^{\mathrm{CH}}_0=x_0\sim\rho_0\in\CP_4(\bbR^d)$,
	the second part follows after an application of Lemma~\ref{lem:boundedness_consensuspoint}, since
	\begin{align*}
		\bbE \max_{\ell=1,\dots,k} \N{x^{\mathrm{CH}}_{\ell}}_2^4
		&= \bbE \max_{\ell=1,\dots,k} \N{\conspoint{\mu_\ell}}_2^4 \\
		&\leq 2b_1^2 + 2b_2^2 \, \bbE \max_{\ell=1,\dots,k} \int \N{x}_2^4 d\mu_\ell(x).
	\end{align*}
	Choosing the constant~$\CM^{\mathrm{CH}}$ large enough for either estimate to hold with $k=K$ concludes the proof.
\end{proof}

\begin{lemma}
	\label{lem:boundedness_CH_empirical}
	Let $Y_k^i \sim \mu_k$ for $i=1,\dots,N$ and let $\widehat\mu^N_{k} = \frac{1}{N}\sum_{i=1}^N \delta_{Y_k^i}$.
	Then, under the assumptions of Lemma~\ref{lem:boundedness_CH}, for the empirical random measures~$(\widehat\mu^N_{k})_{k=1,\dots,K}$ it holds
	\begin{equation*}
		\bbE \max_{k=1,\dots,K} \int\N{x}_2^4 d\widehat\mu^N_{k}(x)
		\leq \widehat\CM^{\mathrm{CH}}
	\end{equation*}
	with a constant $\widehat\CM^{\mathrm{CH}}=\widehat\CM^{\mathrm{CH}}(\widetilde\sigma,d,b_1,b_2,K,\rho_0)>0$.
\end{lemma}

\begin{proof}
	By definition of the empirical measure~$\widehat\mu^N_{k}$ it holds
	\begin{align}
		\label{eq:proof:lem:boundedness_CH_empirical:1}
	\begin{aligned}
		\bbE \max_{k=1,\dots,K} \int\N{x}_2^4 d\widehat\mu^N_{k}(x)
		= \bbE \max_{k=1,\dots,K} \frac{1}{N}\sum_{i=1}^N \N{Y_k^i}_2^4 
		\leq \frac{1}{N}\sum_{i=1}^N \bbE \max_{k=1,\dots,K} \N{Y_k^i}_2^4.
	\end{aligned}
	\end{align}
	Since $Y_k^i \sim \mu_k = \CN\!\left(x^{\mathrm{CH}}_{k-1},\widetilde{\sigma}^2\Id\right)$ for any $k=1,\dots,K$ and $i=1,\dots,N$, we can write $Y_k^i = x^{\mathrm{CH}}_{k-1} + \widetilde\sigma B_{Y,k}^i$, where $B_{Y,k}^i$ is a standard Gaussian random vector, i.e., $B_{Y,k}^i\sim\CN\!\left(0,\Id\right)$.
	By means of the standard inequality~\eqref{eq:triangel_inequality_general} for $p=4$ and $J=2$ we thus have
	\begin{align}
		\label{eq:proof:lem:boundedness_CH_empirical:3}
	\begin{aligned}
		\bbE \max_{k=1,\dots,K} \N{Y_k^i}_2^4
		&\lesssim \bbE \max_{k=1,\dots,K} \N{x^{\mathrm{CH}}_{k-1}}_2^4 + \widetilde\sigma^4 \bbE \max_{k=1,\dots,K} \N{B_{Y,k}^i}_2^4\\
		&\leq \CM^{\mathrm{CH}} + K\widetilde\sigma^4 (d^2+2d),
	\end{aligned}
	\end{align}
	where in the last step we employed Lemma~\ref{lem:boundedness_CH} for the first term and bounded the maximum by the sum in the second term before using again that $\bbE \Nnormal{B}^4_2 = d^2+2d$ for $B\sim\CN(0,\Id)$.
	Inserting \eqref{eq:proof:lem:boundedness_CH_empirical:3} into \eqref{eq:proof:lem:boundedness_CH_empirical:1} yields the claim.
\end{proof}

\subsection{Boundedness of the minimizing movement scheme~\eqref{eq:MMS}}
\label{subsec:appendix:boundednessMMS}

We recall that the iterates~$(x^{\mathrm{MMS}}_{k})_{k=0,\dots,K}$ of the minimizing movement scheme~(MMS)~\eqref{eq:MMS} are defined by
\begin{align*}
\begin{aligned}
	x^{\mathrm{MMS}}_{k}
	&= \argmin_{x\in\bbR^d} \; \CE_k(x),
    \quad\text{ with }\quad
    \CE_k(x) := \frac{1}{2\tau} \N{x^{\mathrm{MMS}}_{k-1} - x}_2^2 + \CE(x), \\
    x^{\mathrm{MMS}}_{0}
	&=x_0.
\end{aligned}
\end{align*}

\begin{lemma}[Boundedness of the MMS~\eqref{eq:MMS}]
\label{lem:boundedness_MMS}
	Let $\CE\in\CC(\bbR^d)$ satisfy \ref{asm:zero_global}\;\!--\,\ref{asm:local_Lipschitz_quadratic_growth}.
	Moreover, let $\rho_0\in\CP_4(\bbR^d)$.
	Then, for the iterates~$(x^{\mathrm{MMS}}_{k})_{k=0,\dots,K}$ of~\eqref{eq:MMS} it holds
	\begin{equation*}
		\bbE \max_{k=0,\dots,K} \N{x^{\mathrm{MMS}}_{k}}_2^4
		\leq \CM^{\mathrm{MMS}}
	\end{equation*}
	with a constant $\CM^{\mathrm{MMS}}=\CM^{\mathrm{MMS}}(K\tau,C_2,\rho_0)>0$.
\end{lemma}

\begin{proof}
	Since $x^{\mathrm{MMS}}_{k}$ is the minimizer of $\CE_{k}$, see \eqref{eq:MMS}, a comparison with the old iterate $x^{\mathrm{MMS}}_{k-1}$ yields
	\begin{align*}
		\frac{1}{2\tau} \N{x^{\mathrm{MMS}}_{k-1} - x^{\mathrm{MMS}}_{k}}_2^2 + \CE(x^{\mathrm{MMS}}_{k})
		\leq \CE(x^{\mathrm{MMS}}_{k-1})
	\end{align*}
	for any $k=1,\dots,K$.
	Using the standard inequality~\eqref{eq:triangel_inequality_general} for $p=2$ and $J=k$, this can be utilized to obtain
	\begin{align*}
		\N{x^{\mathrm{MMS}}_{k}}_2^2
		&\leq 2\N{x^{\mathrm{MMS}}_{0}}_2^2 + 2K\sum_{\ell=1}^k \N{x^{\mathrm{MMS}}_{\ell}-x^{\mathrm{MMS}}_{\ell-1}}_2^2 \\
		&\leq 2\N{x^{\mathrm{MMS}}_0}_2^2 + 4K\tau\sum_{\ell=1}^k \left(\CE(x^{\mathrm{MMS}}_{\ell-1}) - \CE(x^{\mathrm{MMS}}_{\ell})\right) \\
		&= 2\N{x_0^{\mathrm{MMS}}}_2^2 + 4K\tau \left(\CE(x^{\mathrm{MMS}}_{0}) - \CE(x^{\mathrm{MMS}}_{k})\right) \\
		&\leq 2\N{x_0}_2^2 + 4K\tau \left(\CE(x_{0}) - \underbar\CE\right) \\
		&\leq 2\N{x_0}_2^2 + 4K\tau C_2(1 + \N{x_0}_2^2) \\
		&= 2\left(1+2K\tau C_2\right)\N{x_0}_2^2 + 4K\tau C_2,
	\end{align*}
		which trivially also holds for $k=0$.
	Taking the square and expectation w.r.t.\@ the initial condition~$\rho_0$ on both sides concludes the proof.
\end{proof}


\subsection{Boundedness of the implicit version of the CH scheme~\eqref{eq:CH_aux}}
\label{subsec:appendix:boundednessAuxiliaryMMS}

Let us recall that the iterates~$(\widetilde{x}^{\mathrm{CH}}_{k})_{k=0,\dots,K}$ of the scheme~\eqref{eq:CH_aux} are defined by
\begin{align*}
\begin{aligned}
    \widetilde{x}^{\mathrm{CH}}_{k}
    &= \argmin_{x\in\bbR^d} \; \widetilde\CE_k(x),
    \quad\text{ with }\quad
    \widetilde\CE_k(x) := \frac{1}{2\tau} \N{x^{\mathrm{CH}}_{k-1} - x}_2^2 + \CE(x), \\
    \widetilde{x}^{\text{CH}}_{0}
    &=x_0.
\end{aligned}
\end{align*}

\begin{lemma}[Boundedness of the implicit version of the CH scheme~\eqref{eq:CH_aux}]
	\label{lem:boundedness_auxiliaryMMS}
	Let $\CE\in\CC(\bbR^d)$ satisfy \ref{asm:zero_global}\;\!--\,\ref{asm:quadratic_growth}.
	Moreover, let $\rho_0\in\CP_4(\bbR^d)$.
	Then, for the iterates~$(\widetilde{x}^{\mathrm{CH}}_{k})_{k=0,\dots,K}$ of~\eqref{eq:CH_aux} it holds
	\begin{equation*}
		\bbE \max_{k=0,\dots,K} \N{\widetilde{x}^{\mathrm{CH}}_{k}}_2^4
		\leq \widetilde{\CM}^{\mathrm{CH}}
	\end{equation*}
	with a constant $\widetilde{\CM}^{\mathrm{CH}}=\widetilde{\CM}^{\mathrm{CH}}(\tau,C_2,\CM^{\mathrm{CH}})>0$.
\end{lemma}

\begin{proof}
	Since $\widetilde{x}^{\mathrm{CH}}_{k}$ is the minimizer of $\widetilde\CE_{k}$, see \eqref{eq:CH_aux}, a comparison with $x^{\mathrm{CH}}_{k-1}$ yields
	\begin{align*}
		\frac{1}{2\tau} \N{x^{\mathrm{CH}}_{k-1} - \widetilde{x}^{\mathrm{CH}}_{k}}_2^2 + \CE(\widetilde{x}^{\mathrm{CH}}_{k}) \leq \CE(x^{\mathrm{CH}}_{k-1}).
	\end{align*}
	This can be utilized to obtain
	\begin{align*} 
	\begin{aligned}
		\N{\widetilde{x}^{\mathrm{CH}}_{k}}_2^2
		&= 2\N{\widetilde{x}^{\mathrm{CH}}_{k}-x^{\mathrm{CH}}_{k-1}}_2^2 + 2\N{x^{\mathrm{CH}}_{k-1}}_2^2 \\
		&\leq 4\tau\left(\CE(x^{\mathrm{CH}}_{k-1})-\CE(\widetilde{x}^{\mathrm{CH}}_{k})\right) + 2\N{x^{\mathrm{CH}}_{k-1}}_2^2 \\
		&\leq 4\tau\left(\CE(x^{\mathrm{CH}}_{k-1})-\underbar\CE\right) + 2\N{x^{\mathrm{CH}}_{k-1}}_2^2 \\
		&\leq 4\tau C_2\left(1 + \N{x^{\mathrm{CH}}_{k-1}}_2^2\right) + 2\N{x^{\mathrm{CH}}_{k-1}}_2^2
		\\
        &=2\left(1 + 2\tau C_2\right)\N{x^{\mathrm{CH}}_{k-1}}_2^2 + 4\tau C_2.
	\end{aligned}
	\end{align*}
	Taking the square and expectation w.r.t.\@ the initial condition~$\rho_0$ on both sides concludes the proof by virtue of Lemma~\ref{lem:boundedness_CH}.
\end{proof}


\subsection{Boundedness of all numerical schmemes}

\begin{remark}[Boundedness of the schemes~\eqref{eq:CBO}, \eqref{eq:CH}, \eqref{eq:CH_aux} and~\eqref{eq:MMS}]
	\label{rem:boundedness}
    To keep the notation of the main body of the paper concise, we denote by $\CM$ the collective moment bound
    \begin{align}
        \CM
        = \max\left\{
        \CM^{\mathrm{CBO}},
        \widetilde\CM^{\mathrm{CBO}},
        \CM^{\mathrm{CH}},
        \widehat\CM^{\mathrm{CH}},
        \widehat\CM^{\mathrm{MMS}},
        \widetilde{\CM}^{\mathrm{CH}}
        \right\},
    \end{align}
    where $\CM^{\mathrm{CBO}}$, $
        \CM^{\mathrm{CH}}$, $
        \widehat\CM^{\mathrm{CH}}$,$
        \widehat\CM^{\mathrm{MMS}}$, and $
        \widetilde{\CM}^{\mathrm{CH}}$ are as defined in
    Lemmas~\ref{lem:boundedness_CBOdynamics},  \ref{lem:boundedness_CH}, \ref{lem:boundedness_CH_empirical}, \ref{lem:boundedness_MMS}, and~\ref{lem:boundedness_auxiliaryMMS}, respectively.
    Moreover, $
        \widetilde\CM^{\mathrm{CBO}}=\CM^{\mathrm{CBO}}(1/\Delta t, \sigma, d, b_1, b_2, K\Delta t, K, \rho_0)$.
\end{remark}

\section{Proof details for Theorem~\ref{thm:relaxation_CBO_CH}} \label{sec:appendix:thm:relaxation_CBO_CH}

Theorem~\ref{thm:relaxation_CBO_CH} is centered around the observation that, as $\lambda\rightarrow 1/\Delta t$ in the CBO dynamics~\eqref{eq:CBO_dynamics}, the CBO scheme~\eqref{eq:CBO} resembles an implementation of the CH scheme~\eqref{eq:CH} via sampling from the underlying distribution~$\mu_k$ and computing the associated weighted empirical average.
Accordingly, the proof of Theorem~\ref{thm:relaxation_CBO_CH} consists of three ingredients.
First, a stability estimate for the CBO dynamics~\eqref{eq:CBO_dynamics} w.r.t.\@ the parameter~$\lambda$, see Lemma~\ref{lem:stability_CBO_dynamics}.
Second, a quantification of the structural difference in the noise component between the CBO scheme~\eqref{eq:CBO} and the CH scheme~\eqref{eq:CH}, and third a large deviation bound to control the sampling error associated with the Monte Carlo approximation of the CH scheme~\eqref{eq:CH}, see Lemma~\ref{lem:lln_consensuspoint}.

\subsection{Stability of the consensus point~\eqref{eq:consensus_point} w.r.t.\@ the underlying measure}
	\label{subsec:appendix:stability_consensus_points}
	
We first recall from~\cite[Lemma~3.2]{carrillo2018analytical} in a slightly modified form a stability estimate for the consensus point~\eqref{eq:consensus_point} w.r.t.\@~the measure from which it is computed.
Loosely speaking, we show that the mapping $\conspointnoarg: \CP(\bbR^d)\rightarrow\bbR^d$ is Lipschitz-continuous in the Wasserstein-$2$ metric.
	
\begin{lemma}[{Stability of the consensus point~$\conspointnoarg$}] \label{lem:stability_consensus_points}
	Let $\CE\in\CC(\bbR^d)$ satisfy \ref{asm:zero_global}\;\!--\,\ref{asm:local_Lipschitz_quadratic_growth}.
	Moreover, let $\varrho,\varrho'\in\CP(\bbR^d)$ 
	be random measures
	and define the cutoff function (random variable)
	\begin{align*}
		\overbar{\CI}^1_{M}=
		\begin{cases}
			1, & \text{ if } \max\left\{\int\N{\;\!\dummy\;\!}_2^4 d\varrho, \int\N{\;\!\dummy\;\!}_2^4 d\varrho'\right\} \leq M^4,\\
			0, & \text{ else}.
		\end{cases}
	\end{align*}
	Then it holds
	\begin{equation*}
		\N{\conspoint{\varrho}-\conspoint{\varrho'}}_2\overbar{\CI}^1_{M}
		\leq c_0 W_2(\varrho,\varrho')\overbar{\CI}^1_{M}
	\end{equation*}
	with a constant $c_0=c_0(\alpha,C_1,C_2,M)>0$.
\end{lemma}

\begin{proof}
	
	To start with, we note that under \ref{asm:local_Lipschitz_quadratic_growth} and with Jensen's inequality it holds
	\begin{align}
		\label{eq:proof:lem:stability_consensus_points:1}
	\begin{aligned}
		\frac{e^{-\alpha\underbarscript{\CE}}\,\overbar{\CI}^1_{M}}{\N{\omegaa}_{L_1(\varrho)}}
		&= \frac{\overbar{\CI}^1_{M}}{\int \exp\left(-\alpha(\CE(x)-\underbarscript{\CE})\right) d\varrho(x)} 
		\leq \frac{\overbar{\CI}^1_{M}}{\int \exp\!\big(\!-\!\alpha C_2(1+\Nnormal{x}_2^2)\big) \,d\varrho(x)}\\
		&\leq \frac{\overbar{\CI}^1_{M}}{\exp\!\big(\!-\!\alpha C_2(1+\int\Nnormal{x}_2^2\,d\varrho(x))\big)} 
		\leq \exp\!\big(\alpha C_2(1+M^2)\big) =: c_M.
	\end{aligned}
	\end{align}
	An analogous statement can be obtained for the measure~$\varrho'$.
	
	By definition of the consensus point~$\conspointnoarg$ in~\eqref{eq:consensus_point}, it holds for any coupling $\gamma\in\Pi(\varrho,\varrho')$ between $\varrho$ and $\varrho'$ by Jensen's inequality 
	\begin{align}
		\label{eq:proof:lem:stability_consensus_points:3}	
	\begin{aligned}
		\N{\conspoint{\varrho}-\conspoint{\varrho'}}_2\overbar{\CI}^1_{M}
		&\leq \iint \N{x \frac{\omegaa(x)}{\N{\omegaa}_{L_1(\varrho)}} - x' \frac{\omegaa(x')}{\N{\omegaa}_{L_1(\varrho')}}}_2 d\gamma(x,x')\,\overbar{\CI}^1_{M} \\
		&\leq \iint \big(\!\N{T_1(x,x')}_2 + \N{T_2(x,x')}_2 + \N{T_3(x,x')}_2\!\big) \,d\gamma(x,x')\,\overbar{\CI}^1_{M},
	\end{aligned}
	\end{align}
	where the terms~$T_1$, $T_2$ and $T_3$ are defined implicitly and bounded as follows.
	For the first term~$T_1$ we have
	\begin{align}
		\label{eq:proof:lem:stability_consensus_points:5}
	\begin{aligned}
		\N{T_1(x,x')}_2\overbar{\CI}^1_{M}
		= \N{x-x'}_2 \frac{\omegaa(x)}{\N{\omegaa}_{L_1(\varrho)}}\overbar{\CI}^1_{M}
		\leq c_M \N{x-x'}_2 \overbar{\CI}^1_{M},
	\end{aligned}
	\end{align}
	where we utilized \eqref{eq:proof:lem:stability_consensus_points:1} in the last step.
	For the second term~$T_2$, with \ref{asm:local_Lipschitz_quadratic_growth} and again \eqref{eq:proof:lem:stability_consensus_points:1} we obtain
	\begin{align}
		\label{eq:proof:lem:stability_consensus_points:7}
	\begin{aligned}
		\N{T_2(x,x')}_2\overbar{\CI}^1_{M}
		&= \N{x'}_2 \frac{\abs{\omegaa(x)-\omegaa(x')}}{\N{\omegaa}_{L_1(\varrho)}}\overbar{\CI}^1_{M} \\
		&\leq \N{x'}_2 \frac{\alpha e^{-\alpha\underbarscript{\CE}}C_1(1+\N{x}_2+\N{x'}_2)\N{x-x'}_2}{\N{\omegaa}_{L_1(\varrho)}}\overbar{\CI}^1_{M} \\
		&\leq \alpha c_MC_1 \N{x'}_2 (1+\N{x}_2+\N{x'}_2)\N{x-x'}_2\overbar{\CI}^1_{M}.
	\end{aligned}
	\end{align}
	Eventually, for the third therm~$T_3$ it holds by following similar steps
	\begin{align}
		\label{eq:proof:lem:stability_consensus_points:9}
	\begin{aligned}
		\N{T_3(x,x')}_2\overbar{\CI}^1_{M}
		&= \N{x'}_2\omegaa(x') \frac{\abs{\N{\omegaa}_{L_1(\varrho')}-\N{\omegaa}_{L_1(\varrho)}}}{\N{\omegaa}_{L_1(\varrho)}\N{\omegaa}_{L_1(\varrho')}}\overbar{\CI}^1_{M} \\
		&\leq c_M \N{x'}_2 \frac{\iint \alpha e^{-\alpha\underbarscript{\CE}}C_1 (1+\N{x}_2+\N{x'}_2)\N{x-x'}_2 d\pi(x,x')}{\N{\omegaa}_{L_1(\varrho)}}\overbar{\CI}^1_{M} \\
		&\leq \alpha c_M^2C_1 \N{x'}_2 \iint (1+\N{x}_2+\N{x'}_2)\N{x-x'}_2 d\pi(x,x')\,\overbar{\CI}^1_{M}.
	\end{aligned}
	\end{align}
	Collecting the estimates \eqref{eq:proof:lem:stability_consensus_points:5}\,--\! \eqref{eq:proof:lem:stability_consensus_points:9} in \eqref{eq:proof:lem:stability_consensus_points:3}, we obtain with Cauchy-Schwarz inequality and by exploiting the definition of $\overbar{\CI}^1_{M}$ that
	\begin{align}
		\label{eq:proof:lem:stability_consensus_points:11}	
	\begin{aligned}
		\N{\conspoint{\varrho}-\conspoint{\varrho'}}_2\overbar{\CI}^1_{M}
		&\leq c_M\left(1 + 3\alpha C_1(1+c_M)M(1+3M)\right)\sqrt{\iint \N{x-x'}_2^2 d\gamma(x,x') \, \overbar{\CI}^1_{M}}.
	\end{aligned}
	\end{align}
	Squaring both sides and optimizing over all couplings $\gamma\in\Pi(\varrho,\varrho')$ concludes the proof.
\end{proof}

\subsection{Stability of the CBO dynamics~\eqref{eq:CBO_dynamics} w.r.t.\@~the parameters~$\lambda$ and $\sigma$}
	\label{subsec:appendix:stability_CBO}
	
Let us now show the stability of the CBO dynamics~\eqref{eq:CBO_dynamics} w.r.t.\@~its parameters, in particular, the drift and noise parameters~$\lambda$ and~$\sigma$.
For this we control in Lemma~\ref{lem:stability_CBO_dynamics} below the mismatch of the iterates of the CBO dynamics~\eqref{eq:CBO_dynamics} for different parameters, however, provided coinciding initialization and discrete Brownian motion paths. 

\begin{lemma}[Stability of the CBO dynamics~\eqref{eq:CBO_dynamics}] \label{lem:stability_CBO_dynamics}
	Let $\CE\in\CC(\bbR^d)$ satisfy \ref{asm:zero_global}\;\!--\,\ref{asm:quadratic_growth}.
	Moreover, let $\rho_0\in\CP_4(\bbR^d)$.
	We denote by $\big((X_{k}^{i,1})_{k=0,\dots,K}\big)_{i=1,\dots,N}$ and $\big((X_{k}^{i,2})_{k=0,\dots,K}\big)_{i=1,\dots,N}$ solutions to~\eqref{eq:CBO_dynamics} with parameters $\lambda_1,\sigma_1$ and $\lambda_2,\sigma_2$, respectively.
	Furthermore, we write $(\widehat\rho^{N,1}_{k})_{k=0,\dots,K}$ and $(\widehat\rho^{N,2}_{k})_{k=0,\dots,K}$ for the associated empirical measures and introduce the cutoff function (random variable)
	\begin{align} \label{eq:lem:stability_CBO_dynamics:cutoff1}
		\overbar{\CI}^1_{M,k} =
		\begin{cases}
			1, & \text{ if } \max\left\{\int\N{\;\!\dummy\;\!}_2^4 d\widehat\rho^{N,1}_{k}, \int\N{\;\!\dummy\;\!}_2^4 d\widehat\rho^{N,2}_{k}\right\} \leq M^4, \\
			0, & \text{ else}.
		\end{cases}
	\end{align}
	
	Then, under the assumption of coinciding initial conditions $X_{0}^{i,1}=X_{0}^{i,2}$ for all $i=1,\dots,N$ as well as Gaussian random vectors~$B_{k}^i$ for all $k=1,\dots,K$ and all $i=1,\dots,N$, it holds
	\begin{equation*}
		\frac{1}{N}\sum_{i=1}^N\bbE\Nbig{X_{k}^{i,1} - X_{k}^{i,2}}_2^2\,\overbar{\CI}^1_{M,k}
		\leq c_1\left(\abs{\lambda_1-\lambda_2}^2+\abs{\sigma_1-\sigma_2}^2\right) e^{c_2(k-1)}
	\end{equation*}
	with constants
	$c_1=c_1(\Delta t, d,b_1,b_2,M)>0$
	and $c_2=c_2(\Delta t, d,\alpha,\lambda_2,\sigma_2,C_1,C_2,M)>0$ for all $k\geq1$.
\end{lemma}

\begin{proof}
	Let us first remark that the cutoff function~$\overbar{\CI}^1_{M,k}$ defined in~\eqref{eq:lem:stability_CBO_dynamics:cutoff1} is adapted to the natural filtration~$\{\CF_k\}_{k=0,\dots,K}$, where $\CF_{k}$ denotes the sigma algebra generated by the random variables {$\{B_\ell^i, \,\ell=1,\dots,k, \, i=1,\dots,N\}$}.
	Now, using the iterative update rule~\eqref{eq:CBO_dynamics} for $X_{k}^{i,1}$ and $X_{k}^{i,2}$ with parameters $\lambda_1,\sigma_1$ and $\lambda_2,\sigma_2$, respectively, we obtain, by employing the standard inequality~\eqref{eq:triangel_inequality_general} for $p=2$ and $J=5$, for their squared norm difference the upper bound
	\begin{align}
		\label{proof:lem:stability_CBO_dynamics:1}
	\begin{aligned}
		\Nbig{X_{k}^{i,1} - X_{k}^{i,2}}_2^2
		&\lesssim \Nbig{X_{k-1}^{i,1} - X_{k-1}^{i,2}}_2^2 
		+ \left(\Delta t \abs{\lambda_1\!-\!\lambda_2}\right)^2\left(\Nbig{X_{k-1}^{i,1}}_2^2+\Nbig{\conspoint{\widehat\rho^{N,1}_{k-1}}}_2^2\right) \\
		&\qquad\quad + (\Delta t \lambda_2)^2\left(\Nbig{X_{k-1}^{i,1} - X_{k-1}^{i,2}}_2^2 + \Nbig{\conspoint{\widehat\rho^{N,1}_{k-1}} - \conspoint{\widehat\rho^{N,2}_{k-1}}}_2^2\right) \\
		&\qquad\quad + \abs{\sigma_1\!-\!\sigma_2}^2\left(\Nbig{X_{k-1}^{i,1}}_2^2+\Nbig{\conspoint{\widehat\rho^{N,1}_{k-1}}}_2^2\right)\N{B^i_k}_2^2 \\
		&\qquad\quad + \sigma_2^2\left(\Nbig{X_{k-1}^{i,1} - X_{k-1}^{i,2}}_2^2 + \Nbig{\conspoint{\widehat\rho^{N,1}_{k-1}} - \conspoint{\widehat\rho^{N,2}_{k-1}}}_2^2\right)\N{B^i_k}_2^2 \\
		&\lesssim \left(1\!+\!(\Delta t \lambda_2)^2\!+\!\sigma_2^2\N{B^i_k}_2^2\right)\left(\Nbig{X_{k-1}^{i,1} \!-\! X_{k-1}^{i,2}}_2^2 \!+\! \Nbig{\conspoint{\widehat\rho^{N,1}_{k-1}} \!-\! \conspoint{\widehat\rho^{N,2}_{k-1}}}_2^2 \right) \\
		&\qquad\quad + \left(\left(\Delta t \abs{\lambda_1\!-\!\lambda_2}\right)^2+\abs{\sigma_1\!-\!\sigma_2}^2\N{B^i_k}_2^2\right)\left(\Nbig{X_{k-1}^{i,1}}_2^2\!+\!\Nbig{\conspoint{\widehat\rho^{N,1}_{k-1}}}_2^2\right)\!.
	\end{aligned}
	\end{align}
	Since $\overbar{\CI}^1_{M,k}$ satisfies $\overbar{\CI}^1_{M,k} = \overbar{\CI}^1_{M,k}\overbar{\CI}^1_{M,\ell}$ for all $\ell\leq k$ and $\overbar{\CI}^1_{M,k}\leq1$, we obtain from \eqref{proof:lem:stability_CBO_dynamics:1} that
	\begin{align*}
	\begin{aligned}
		&\Nbig{X_{k}^{i,1} - X_{k}^{i,2}}_2^2\,\overbar{\CI}^1_{M,k}\\
		&\quad\lesssim \left(1\!+\!(\Delta t \lambda_2)^2\!+\!\sigma_2^2\N{B^i_k}_2^2\right)\left(\Nbig{X_{k-1}^{i,1} - X_{k-1}^{i,2}}_2^2 + \Nbig{\conspoint{\widehat\rho^{N,1}_{k-1}} - \conspoint{\widehat\rho^{N,2}_{k-1}}}_2^2 \right)\overbar{\CI}^1_{M,k-1} \\
		&\quad\qquad\quad + \left(\left(\Delta t \abs{\lambda_1\!-\!\lambda_2}\right)^2+\abs{\sigma_1\!-\!\sigma_2}^2\N{B^i_k}_2^2\right)\left(\Nbig{X_{k-1}^{i,1}}_2^2+\Nbig{\conspoint{\widehat\rho^{N,1}_{k-1}}}_2^2\right)\overbar{\CI}^1_{M,k-1}.
	\end{aligned}
	\end{align*}
	With the random variables~$X_{k-1}^{i,1}$, $X_{k-1}^{i,2}$, $\conspoint{\widehat\rho^{N,1}_{k-1}}$, $\conspoint{\widehat\rho^{N,2}_{k-1}}$ and $\overbar{\CI}^1_{M,k-1}$ being $\CF_{k-1}$-measurable, taking the expectation w.r.t.\@ the sampling of the random vectors~$B_k^i$, $i=1,\dots,N$, i.e., the conditional expectation $\bbE_k = \bbE\left[\;\dummy\;\!\!\mid\!\CF_{k-1}\right]$, yields
	\begin{align*}
		&\bbE_k\Nbig{X_{k}^{i,1} - X_{k}^{i,2}}_2^2\,\overbar{\CI}^1_{M,k}\\
		&\quad\lesssim \big(1\!+\!(\Delta t \lambda_2)^2\!+\!d\Delta t\sigma_2^2\big)\left(\Nbig{X_{k-1}^{i,1} - X_{k-1}^{i,2}}_2^2 + \Nbig{\conspoint{\widehat\rho^{N,1}_{k-1}} - \conspoint{\widehat\rho^{N,2}_{k-1}}}_2^2 \right)\overbar{\CI}^1_{M,k-1} \\
		&\quad\qquad\quad + \left(\left(\Delta t \abs{\lambda_1\!-\!\lambda_2}\right)^2+d\Delta t\abs{\sigma_1\!-\!\sigma_2}^2\right)\left(\Nbig{X_{k-1}^{i,1}}_2^2+\Nbig{\conspoint{\widehat\rho^{N,1}_{k-1}}}_2^2\right)\overbar{\CI}^1_{M,k-1},
	\end{align*}
	where we used the fact that $\bbE_k\Nnormal{B_k^i}_2^2 = d\Delta t$.
	Taking now the total expectation~$\bbE$ on both sides, we have by tower property (law of total expectation)
	\begin{align}
	\label{proof:lem:stability_CBO_dynamics:3}
	\begin{aligned}
		&\bbE\Nbig{X_{k}^{i,1} - X_{k}^{i,2}}_2^2\,\overbar{\CI}^1_{M,k} \\
		&\quad\lesssim \left(1\!+\!(\Delta t \lambda_2)^2\!+\!d\Delta t \sigma_2^2\right)\left(\bbE\Nbig{X_{k-1}^{i,1} \!-\! X_{k-1}^{i,2}}_2^2\,\overbar{\CI}^1_{M,k-1} + \bbE\Nbig{\conspoint{\widehat\rho^{N,1}_{k-1}} \!-\! \conspoint{\widehat\rho^{N,2}_{k-1}}}_2^2\,\overbar{\CI}^1_{M,k-1} \right) \\
		&\quad\qquad\quad + \left(\left(\Delta t \abs{\lambda_1\!-\!\lambda_2}\right)^2 + d\Delta t\abs{\sigma_1\!-\!\sigma_2}^2\right)\left(\bbE\Nbig{X_{k-1}^{i,1}}_2^2\,\overbar{\CI}^1_{M,k-1} + \bbE\Nbig{\conspoint{\widehat\rho^{N,1}_{k-1}}}_2^2\,\overbar{\CI}^1_{M,k-1}\right).
	\end{aligned}
	\end{align}
	As a consequence of the stability estimate for the consensus point, Lemma~\ref{lem:stability_consensus_points}, it holds for a constant $c_0=c_0(\alpha,C_1,C_2,M)>0$ that
	\begin{align*}
		\bbE\Nbig{\conspoint{\widehat\rho^{N,1}_{k-1}} - \conspoint{\widehat\rho^{N,2}_{k-1}}}_2^2\,\overbar{\CI}^1_{M,k-1}
		&\leq c_0\bbE W_2^2\big(\widehat\rho^{N,1}_{k-1},\widehat\rho^{N,2}_{k-1}\big)\,\overbar{\CI}^1_{M,k-1} \\
		&\leq c_0 \frac{1}{N}\sum_{i=1}^N\bbE\Nbig{X_{k-1}^{i,1} - X_{k-1}^{i,2}}_2^2\,\overbar{\CI}^1_{M,k-1},
	\end{align*}
	where we chose $\pi=\frac{1}{N}\sum_{i=1}^N\delta_{X_{k-1}^{i,1}}\otimes\delta_{X_{k-1}^{i,2}}$ as viable transportation plan in Definition~\eqref{def:wassersteindistance} to upper bound the Wasserstein distance in the second step.
	Utilizing this when averaging \eqref{proof:lem:stability_CBO_dynamics:3} over $i$ gives
	\begin{align}
		\label{proof:lem:stability_CBO_dynamics:5}
	\begin{aligned}
		\frac{1}{N}\sum_{i=1}^N\bbE\Nbig{X_{k}^{i,1} \!-\! X_{k}^{i,2}}_2^2\,\overbar{\CI}^1_{M,k}
		&\lesssim (1\!+\!c_0)\!\left(1\!+\!(\Delta t \lambda_2)^2\!+\!d\Delta t \sigma_2^2\right)\frac{1}{N}\sum_{i=1}^N\bbE\Nbig{X_{k-1}^{i,1} \!-\! X_{k-1}^{i,2}}_2^2\,\overbar{\CI}^1_{M,k-1} \\
		&\qquad\quad + \left(\left(\Delta t \abs{\lambda_1\!-\!\lambda_2}\right)^2 + d\Delta t\abs{\sigma_1\!-\!\sigma_2}^2\right)\left(b_1 + (1+b_2)M^2\right),
	\end{aligned}
	\end{align}	
	where we employed Lemma~\ref{lem:boundedness_consensuspoint} together with the definition of the cutoff function~$\overbar{\CI}^1_{M,k-1}$ to obtain the bound in the second line of~\eqref{proof:lem:stability_CBO_dynamics:5}.	
	Exploiting that $X_{0}^{i,1}=X_{0}^{i,2}$ for $i=1,\dots,N$ by assumption, we conclude the proof by an application of the discrete variant of Gr\"onwall's inequality~\eqref{eq:gronwall_inequality}, which proves that for all $k\geq1$ it holds
	\begin{align*}
		\frac{1}{N}\sum_{i=1}^N\bbE\Nbig{X_{k}^{i,1} - X_{k}^{i,2}}_2^2\,\overbar{\CI}^1_{M,k}
		&\leq c_1\left(\left(\Delta t \abs{\lambda_1-\lambda_2}\right)^2+d\Delta t\abs{\sigma_1-\sigma_2}^2\right) e^{c_2(k-1)}
	\end{align*}
	with constants~$c_1=c_1(b_1,b_2,M)>0$ and $c_2=c_2(c_0,\Delta t, d, \lambda_2, \sigma_2)>0$.
\end{proof}

\subsection{A large deviation bound for the consensus point~\eqref{eq:consensus_point}}
	\label{subsec:appendix:lln_consensuspoint}

For a given measure~$\varrho\in\CP(\bbR^d)$ and a set of $N$ i.i.d.\@~random variables~$Y^i\sim\varrho$ with empirical random measure $\widehat\varrho^N = \frac{1}{N}\sum_{i=1}^N \delta_{Y^i}$, one expects that under certain regularity assumptions it holds by the law of large numbers 
\begin{align*}
	\conspoint{\widehat\varrho^N}
	\xrightarrow[]{\text{ a.s.\,}}\conspoint{\varrho}
	\quad\text{ as }
	N\rightarrow\infty.
\end{align*}
This is made rigorous in the subsequent lemma, which is based on arguments from~\cite[Lemma~3.1]{fornasier2020consensus_hypersurface_wellposedness} and \cite[Lemma~A.2]{fornasier2021consensus}.

\begin{lemma}[Large deviation bound for the consensus point~$\conspointnoarg$]
	\label{lem:lln_consensuspoint}
	Let $\CE\in\CC(\bbR^d)$ satisfy \ref{asm:zero_global}\;\!--\,\ref{asm:local_Lipschitz_quadratic_growth}.
	Moreover, for $k=1,\dots,K$, let $\mu_k\in\CP(\bbR^d)$ be a random measure, let $(Y_k^i)_{i=1,\dots,N}$ be $N$ i.i.d.\@ random variables distributed according to~$\mu_k$, denote by $\widehat\mu^N_{k}$ the empirical random measure~$\widehat\mu^N_{k} = \frac{1}{N}\sum_{i=1}^N \delta_{Y_k^i}$
	and define the cutoff function (random variable)
		\begin{align} \label{eq:lem:stability_CBO_dynamics:cutoff2}
		\overbar{\CI}^2_{M,k} =
		\begin{cases}
			1, & \text{ if } \max\left\{\int\N{\;\!\dummy\;\!}_2^4 d\widehat\mu^{N}_{k}, \int\N{\;\!\dummy\;\!}_2^4 d\mu_{k}\right\} \leq M^4, \\
			0, & \text{ else}.
		\end{cases}
	\end{align}
	Then it holds
	\begin{equation*}
		\max_{k=1,\dots,K}\bbE\N{\conspoint{\widehat\mu^N_{k}} - \conspoint{\mu_k}}_2^2\overbar{\CI}^2_{M,k}
		\leq c_3N^{-1}
	\end{equation*}
	with a constant $c_3=c_3(\alpha,b_1,b_2,C_2,M)>0$.
\end{lemma}

\begin{proof}
	To start with, we note that under \ref{asm:local_Lipschitz_quadratic_growth} and with Jensen's inequality it holds
	\begin{align}
		\label{eq:proof:lem:lln_consensuspoint:1}
	\begin{aligned}
		\frac{e^{-\alpha\underbarscript{\CE}}\,\overbar{\CI}^2_{M,k}}{\frac{1}{N} \sum_{j=1}^N \omegaa(Y_k^j)}
		&= \frac{\overbar{\CI}^2_{M,k}}{\frac{1}{N} \sum_{j=1}^N\exp\!\big(\!-\!\alpha(\CE(Y_k^j)-\underbarscript{\CE})\big)}
		\leq \frac{\overbar{\CI}^2_{M,k}}{\frac{1}{N} \sum_{j=1}^N\exp\!\big(\!-\!\alpha C_2(1+\Nnormal{Y_k^j}_2^2)\big)}\\
		&\leq \frac{\overbar{\CI}^2_{M,k}}{\exp\!\big(\!-\!\alpha C_2(1+\frac{1}{N} \sum_{j=1}^N\Nnormal{Y_k^j}_2^2)\big)}
		\leq \exp\!\big(\alpha C_2(1+M^2)\big) =: c_M.
	\end{aligned}
	\end{align}

	By definition of the consensus point~$\conspointnoarg$ in~\eqref{eq:consensus_point}, it holds
	\begin{align}
		\label{eq:proof:lem:lln_consensuspoint:3}	
	\begin{aligned}
		\N{\conspoint{\widehat\mu^N_{k}} - \conspoint{\mu_k}}_2\overbar{\CI}^2_{M,k}
		&= \N{\sum_{i=1}^N Y_k^i \frac{\omegaa(Y_k^i)}{\sum_{j=1}^N \omegaa(Y_k^j)} - \int x \frac{\omegaa(x)}{\N{\omegaa}_{L_1(\mu_k)}} d\mu_{k}(x)}_2 \overbar{\CI}^2_{M,k} \\
		&\leq \big(\!\N{T_1}_2 + \N{T_2}_2\!\big) \,\overbar{\CI}^2_{M,k},
	\end{aligned}
	\end{align}
	where the terms~$T_1$ and $T_2$ are defined implicitly and bounded as follows.
	For the first term~$T_1$ we have
	\begin{align}
		\label{eq:proof:lem:lln_consensuspoint:5}
	\begin{aligned}
		\N{T_1}_2\overbar{\CI}^2_{M,k}
		&=  \N{\sum_{i=1}^N Y_k^i \frac{\omegaa(Y_k^i)}{\sum_{j=1}^N \omegaa(Y_k^j)} - \int x \frac{\omegaa(x)}{\frac{1}{N} \sum_{j=1}^N \omegaa(Y_k^j)} d\mu_{k}(x)}_2 \overbar{\CI}^2_{M,k} \\
		&=  \frac{\overbar{\CI}^2_{M,k}}{\frac{1}{N} \sum_{j=1}^N \omegaa(Y_k^j)}\N{\frac{1}{N}\sum_{i=1}^N Y_k^i \omegaa(Y_k^i) - \int x \omegaa(x) \,d\mu_{k}(x)}_2 \\
		&\leq c_Me^{\alpha\underbarscript{\CE}} \N{\frac{1}{N}\sum_{i=1}^N Y_k^i \omegaa(Y_k^i) - \int x \omegaa(x) \,d\mu_{k}(x)}_2,
	\end{aligned}
	\end{align}
	where we utilized \eqref{eq:proof:lem:lln_consensuspoint:1} in the last step.
	Similarly, for the second term~$T_2$ we have
	\begin{align}
		\label{eq:proof:lem:lln_consensuspoint:7}
	\begin{aligned}
		\N{T_2}_2\overbar{\CI}^2_{M,k}
		&=  \N{\int x \frac{\omegaa(x)}{\frac{1}{N} \sum_{j=1}^N \omegaa(Y_k^j)} d\mu_{k}(x) - \int x \frac{\omegaa(x)}{\N{\omegaa}_{L_1(\mu_k)}} d\mu_{k}(x)}_2 \overbar{\CI}^2_{M,k} \\
		&= \frac{\overbar{\CI}^2_{M,k}}{\frac{1}{N} \sum_{j=1}^N \omegaa(Y_k^j)} \N{\conspoint{\mu_k}}_2 \abs{\frac{1}{N} \sum_{j=1}^N \omegaa(Y_k^j) - \int \omegaa(x) \,d\mu_{k}(x)}_2\\
		&\leq c_Me^{\alpha\underbarscript{\CE}} \left(b_1+b_2M\right) \abs{\frac{1}{N} \sum_{j=1}^N \omegaa(Y_k^j) - \int \omegaa(x) \,d\mu_{k}(x)}_2,
	\end{aligned}
	\end{align}
	where the last step involved additionally Lemma~\ref{lem:boundedness_consensuspoint}.
	Let us now introduce the random variables
	\begin{align*}
		Z_k^i := Y_k^i \omegaa(Y_k^i) - \int x \omegaa(x)\,d\mu_{k}(x)
		\quad\text{ and }\quad
		z_k^i := \omegaa(Y_k^i) - \int \omegaa(x)\,d\mu_{k}(x),
	\end{align*}
	respectively, which have zero expectation, and are i.i.d.\@ for $i=1,\dots,N$.
	With these definitions as well as the bounds \eqref{eq:proof:lem:lln_consensuspoint:5} and \eqref{eq:proof:lem:lln_consensuspoint:7} we obtain
	\begin{align}
		\label{eq:proof:lem:lln_consensuspoint:9}
	\begin{aligned}
		\bbE\N{T_1}_2^2\overbar{\CI}^2_{M,k}
		&\leq c_M^2e^{2\alpha\underbarscript{\CE}} \bbE\N{\frac{1}{N}\sum_{i=1}^N Z_k^i}_2^2\overbar{\CI}^2_{M,k}
		= c_M^2e^{2\alpha\underbarscript{\CE}} \frac{1}{N^2} \bbE\sum_{i=1}^N\sum_{j=1}^N\big\langle Z_k^i, Z_k^j\big\rangle\,\overbar{\CI}^2_{M,k} \\
		&= c_M^2e^{2\alpha\underbarscript{\CE}} \frac{1}{N^2} \bbE\sum_{i=1}^N \N{Z_k^i}_2^2\overbar{\CI}^2_{M,k}
		\leq 4c_M^2 M^2 \frac{1}{N}
	\end{aligned}
	\end{align}
	and, analogously,
	\begin{align}
		\label{eq:proof:lem:lln_consensuspoint:11}
	\begin{aligned}
		\bbE\N{T_2}_2^2\overbar{\CI}^2_{M,k}
		&\leq c_M^2e^{2\alpha\underbarscript{\CE}} \left(b_1+b_2M\right)^2 \frac{1}{N^2} \bbE\sum_{i=1}^N \N{z_k^i}_2^2\overbar{\CI}^2_{M,k}
		\leq 4c_M^2 \left(b_1+b_2M\right)^2 \frac{1}{N}.
	\end{aligned}
	\end{align}
	The last inequalities of \eqref{eq:proof:lem:lln_consensuspoint:9} and \eqref{eq:proof:lem:lln_consensuspoint:11} are due to the estimates
	\begin{align*}
	\begin{aligned}
		\bbE\frac{1}{N} \sum_{i=1}^N \N{Z_k^i}_2^2\overbar{\CI}^2_{M,k}
		&\leq 2\bbE\frac{1}{N} \sum_{i=1}^N \N{Y_k^i \omegaa(Y_k^i)}_2^2 \overbar{\CI}^2_{M,k} + 2\bbE \N{\int x \omegaa(x)\,d\mu_{k}(x)}_2^2\overbar{\CI}^2_{M,k} \\
		&\leq 2e^{-2\alpha\underbarscript{\CE}} \bbE\frac{1}{N} \sum_{i=1}^N \N{Y_k^i}_2^2 \overbar{\CI}^2_{M,k} + 2e^{-2\alpha\underbarscript{\CE}}\bbE \int \N{x}_2^2 \,d\mu_{k}(x) \,\overbar{\CI}^2_{M,k} \\
		&\leq 4e^{-2\alpha\underbarscript{\CE}} M^2
	\end{aligned}
	\end{align*}
	and, similarly,
	\begin{align*}
	\begin{aligned}
		\bbE \abs{z_k^1}_2^2\overbar{\CI}^2_{M,k}
		\leq 4e^{-2\alpha\underbarscript{\CE}}.
	\end{aligned}
	\end{align*}
	Combining \eqref{eq:proof:lem:lln_consensuspoint:9} and \eqref{eq:proof:lem:lln_consensuspoint:11} concludes the proof.
\end{proof}

\begin{remark}
	Alternatively to the explicit computations of Lemma~\ref{lem:lln_consensuspoint},
	the stability estimate for the consensus point, Lemma~\ref{lem:stability_consensus_points}, would allow to obtain
	\begin{equation*}
		\max_{k=1,\dots,K}\bbE\N{\conspoint{\widehat\mu^N_{k}} - \conspoint{\mu_k}}_2^2\overbar{\CI}^2_{M,k}
		\leq c_0 \max_{k=1,\dots,K}\bbE W_2^2(\widehat\mu^N_{k},\mu_k)\,\overbar{\CI}^2_{M,k},
	\end{equation*}
	where $\bbE W_2^2(\widehat\mu^N_{k},\mu_k)$ can be controlled by employing \cite[Theorem~1]{fournier2015rate}.
	This, however, only gives a quantitative convergence rate of order $\CO(N^{-2/d})$, which is affected by the curse of dimensionality.
	The convergence rate $\CO(N^{-1})$ obtained in Lemma~\ref{lem:lln_consensuspoint} matches the one to be expected from Monte Carlo sampling.
\end{remark}

\section{Proof details for Proposition~\ref{prop:relaxation_CH_GF} and Theorem~\ref{thm:relaxation_CH_GF}} \label{sec:appendix:thm:relaxation_CH_GF}

Proposition~\ref{prop:relaxation_CH_GF} and Theorem~\ref{thm:relaxation_CH_GF} are centered around the observation that the CH scheme~\eqref{eq:CH} behaves gradient-like.
To establish this, Proposition~\ref{prop:relaxation_CH_GF} exploits, by using the quantitative nonasymptotic Laplace principle (see Section~\ref{subsec:appendix:proof:LaplacePrinciple} and in particular Proposition~\ref{prop:LaplacePrinciple} for a review of \cite[Proposition~4.5]{fornasier2021consensus}), that one step of the implicit CH scheme~\eqref{eq:CH_aux} can be recast into the computation of a consensus point~$\conspointnoargE{\widetilde\CE}$ for an objective function of the form $\widetilde\CE(x) = \frac{1}{2\tau}\Nnormal{\;\dummy\,-x}_2^2 + \CE(x)$.
To prove Theorem~\ref{thm:relaxation_CH_GF}, this is combined with a stability argument for the MMS~\eqref{eq:MMS}, which relies on the $\Lambda$-convexity of $\CE$ (Assumption~\ref{asm:lambda-convex}).


\subsection{A quantitative nonasymptotic Laplace principle}
	\label{subsec:appendix:proof:LaplacePrinciple}

The Laplace principle~\cite{dembo2009large,miller2006applied} asserts that for any absolutely continuous probability measure~$\varrho\in\CP(\bbR^d)$ it holds
\begin{equation*}
	\lim_{\alpha\rightarrow\infty} \left(-\frac{1}{\alpha}\log\left(\int\exp\big(\!-\!\alpha\widetilde\CE(x)\big)\, d\varrho(x)\right)\right)
	= \inf_{x\in\supp(\varrho)} \widetilde\CE(x).
\end{equation*}
This suggests that, as $\alpha\rightarrow\infty$, the Gibbs measure~$\eta_\alpha^{\widetilde\CE} = \omega_\alpha^{\widetilde\CE}\varrho/\Nnormal{\omega_\alpha^{\widetilde\CE}}_{L_1(\varrho)}$ converges to a discrete probability distribution (i.e., a convex combination of Dirac measures) supported on the set of global minimizers of $\widetilde\CE$.
However, even in the case that such minimizer is unique, it does not permit to quantify the proximity of $\conspointE{\widetilde\CE}{\varrho}=\int x\,d\eta_\alpha^{\widetilde\CE}$ (see also Equation~\eqref{eq:consensus_point}) to the minimizer of $\widetilde\CE$ without the following assumption (see also Remark~\ref{rem:ICP}).

\begin{definition}[{Inverse continuity property}]
	\label{def:ICP}
	A function~$\widetilde\CE\in\CC(\bbR^d)$ satisfies the $\ell^2$-inverse continuity property globally if there exist constants~$\eta,\nu>0$ such that
	\begin{align}
		\label{eq:ICP}
		\N{x-\widetilde{x}^*}_2
		\leq \frac{1}{\eta} \big(\widetilde\CE(x)-\underbar{\widetilde\CE}\,\big)^\nu
		\quad\text{ for all } x\in \bbR^d,
	\end{align}
	where $\widetilde{x}^*\in\bbR^d$ denotes the unique global minimizer of $\widetilde\CE$ with objective value $\underbar{\widetilde\CE} := \inf_{x\in\bbR^d} \widetilde\CE(x)$.
\end{definition}

As elaborated on in Remark~\ref{rem:ICP} for the ($\ell^\infty$-)inverse continuity property, it is usually sufficient if \eqref{eq:ICP} holds locally around the global minimizer~$\widetilde{x}^*$.
In the following Proposition~\ref{prop:LaplacePrinciple}, however, we recall the quantitative Laplace principle in the slightly more specific form, where the $\ell^2$-inverse continuity property holds globally as required by Definition~\ref{def:ICP}.
For the general version, namely in the case of functions which satisfy~\eqref{eq:ICP} only on an $\ell^2$-ball around~$\widetilde{x}^*$ (see~\cite[Definition~3.5~(A2)]{fornasier2021consensus} for the details), we refer to~\cite[Proposition~4.5]{fornasier2021consensus}.

\begin{proposition}[Quantitative Laplace principle]
	\label{prop:LaplacePrinciple}
	Let $\widetilde\CE\in\CC(\bbR^d)$ satisfy the $\ell^2$-inverse continuity property in form of Definition~\ref{def:ICP}.
	Moreover, let $\indivmeasure \in \CP(\bbR^d)$.
	For any $r > 0$ define $\widetilde\CE_{r} := \sup_{x \in B_{r}(\widetilde{x}^*)}\widetilde\CE(x)-\underbar{\widetilde\CE}$.
	Then, for fixed $\alpha  > 0$ it holds for any $r,q>0$ that
	\begin{align}
		\label{eq:prop:LaplacePrinciple}
		\Nbig{\conspointE{\widetilde\CE}{\indivmeasure} - \widetilde{x}^*}_2 \leq \frac{\big(q + \widetilde\CE_{r}\big)^\nu}{\eta} + \frac{\exp(-\alpha q)}{\indivmeasure(B_{r}(\widetilde{x}^*))}\int\N{x-\widetilde{x}^*}_2 d\indivmeasure(x).
	\end{align}
\end{proposition}

\begin{proof}
	W.l.o.g.\@~we may assume $\underbar{\widetilde\CE}=0$ since a constant offset to $\widetilde\CE$ neither affects the definition of the consensus point in~\eqref{eq:consensus_point} nor the quantities appearing on the right-hand side of~\eqref{eq:prop:LaplacePrinciple}.
	
	By Markov's inequality it holds $\Nnormal{\exp(-\alpha\widetilde\CE)}_{L_1(\indivmeasure)} \geq a \indivmeasure\big(\big\{x\in\bbR^d : \exp(-\alpha \widetilde\CE(x)) \geq a\big\}\big)$ for any $a > 0$.
	With the choice $a = \exp(-\alpha\widetilde\CE_{r})$ and noting that
	\begin{align*}
		\indivmeasure\left(\left\{x \in \bbR^d: \exp(-\alpha \widetilde\CE(x)) \geq \exp(-\alpha\widetilde\CE_{r})\right\}\right)
		&= \indivmeasure\left(\left\{ x \in \bbR^d: \widetilde\CE(x) \leq \widetilde\CE_{r} \right\}\right)
		\geq \indivmeasure(B_{r}(\widetilde{x}^*)),
	\end{align*}
	we obtain $\Nnormal{\exp(-\alpha\widetilde\CE)}_{L_1(\indivmeasure)} \geq \exp(-\alpha \widetilde\CE_{r})\indivmeasure(B_{r}(\widetilde{x}^*))$.
	Now let $\widetilde r \geq r > 0$.
	With the definition of the consensus point in~\eqref{eq:consensus_point} and by Jensen's inequality we can decompose 
	\begin{align*}
		\Nbig{\conspointE{\widetilde\CE}{\indivmeasure} - \widetilde{x}^*}_2
		&\leq \int_{B_{\widetilde r}(\widetilde{x}^*)} \N{x-\widetilde{x}^*}_2\frac{\exp\big(\!-\!\alpha\widetilde\CE(x)\big)}{\Nbig{\exp(-\alpha\widetilde\CE)}_{L_1(\indivmeasure)}}d\indivmeasure(x)\\
		&\quad\;\!+ \int_{\left(B_{\widetilde r}(\widetilde{x}^*)\right)^c} \N{x-\widetilde{x}^*}_2\frac{\exp\big(\!-\!\alpha\widetilde\CE(x)\big)}{\Nbig{\exp(-\alpha\widetilde\CE)}_{L_1(\indivmeasure)}}d\indivmeasure(x).
	\end{align*}
	The first term is bounded by $\widetilde r$ since $ \Nnormal{x-\widetilde{x}^*}_2\leq \widetilde r$ for all $x \in B_{\widetilde r}(\widetilde{x}^*)$.
	For the second term we use the formerly derived $\Nnormal{\exp(-\alpha\widetilde\CE)}_{L_1(\indivmeasure)} \geq \exp(-\alpha \widetilde\CE_{r})\indivmeasure(B_{r}(\widetilde{x}^*))$ to get
	\begin{align*}
		&\int_{\left(B_{\widetilde r}(\widetilde{x}^*)\right)^c} \N{x-\widetilde{x}^*}_2\frac{\exp\big(\!-\!\alpha\widetilde\CE(x)\big)}{\Nbig{\exp(-\alpha\widetilde\CE)}_{L_1(\indivmeasure)}}d\indivmeasure(x) \\
		&\qquad\qquad\qquad\leq \frac{1}{\exp(-\alpha \widetilde\CE_{r})\indivmeasure(B_{r}(\widetilde{x}^*))}\int_{(B_{\widetilde r}(\widetilde{x}^*))^c} \N{x-\widetilde{x}^*}_2 \exp\big(\!-\!\alpha\widetilde\CE(x)\big)\,d\indivmeasure(x)\\
		&\qquad\qquad\qquad\leq \frac{\exp\left(-\alpha \left(\inf_{x \in (B_{\widetilde r}(\widetilde{x}^*))^c}\widetilde\CE(x)-\widetilde\CE_{r}\right)\right)}{\indivmeasure(B_{r}(\widetilde{x}^*))} \int\N{x-\widetilde{x}^*}_2d\indivmeasure(x).
	\end{align*}	
	Thus, for any $\widetilde r \geq r > 0$ we obtain
	\begin{align} \label{eq:aux_laplace_1}
		\Nbig{\conspointE{\widetilde\CE}{\indivmeasure} - \widetilde{x}^*}_2
		\leq \widetilde r + \frac{\exp\left(-\alpha \left(\inf_{x \in (B_{\widetilde r}(\widetilde{x}^*))^c}\widetilde\CE(x)-\widetilde\CE_{r}\right)\right)}{\indivmeasure(B_{r}(\widetilde{x}^*))} \int\N{x-\widetilde{x}^*}_2d\indivmeasure(x).
	\end{align}		
	We now choose $\widetilde r = \big(q+\widetilde\CE_{r}\big)^{\nu}/\eta$, which satisfies $\widetilde r  \geq r$, since \eqref{eq:ICP} with $\widetilde\minobj = 0$ implies
	\begin{align*}
		\widetilde r
		= \frac{\big(q+\widetilde\CE_{r}\big)^{\nu}}{\eta}
		\geq \frac{\widetilde\CE_{r}^{\nu}}{\eta}
		= \frac{\left(\sup_{x \in B_{r}(\widetilde{x}^*)}\widetilde\CE(x)\right)^{\nu}}{\eta} 
		\geq \sup_{x \in B_{r}(\widetilde{x}^*)}\N{x-\widetilde{x}^*}_2
		= r.
	\end{align*}
	Using again \eqref{eq:ICP} with $\widetilde\minobj = 0$ we thus have
	\begin{align*}
		\inf_{x \in (B_{\widetilde r}(\widetilde{x}^*))^c}\widetilde\CE(x) - \widetilde\CE_{r}
		\geq (\eta \widetilde r)^{1/\nu} - \widetilde\CE_{r}
		= q + \widetilde\CE_r - \widetilde\CE_{r}
		= q.
	\end{align*}
	Inserting this and the definition of $\widetilde r$ into \eqref{eq:aux_laplace_1} gives the statement.
\end{proof}

\subsection{The auxiliary function~$\widetilde\CE_{k}$}
	\label{subsec:appendix:proof:CE_k}

Let us now show that the function
\begin{equation}
    \widetilde\CE_{k}(x) := \frac{1}{2\tau} \N{x^{\mathrm{CH}}_{k-1} - x}_2^2 + \CE(x),
\end{equation}
which appears later in the proofs of Proposition~\ref{prop:relaxation_CH_GF} and Theorem~\ref{thm:relaxation_CH_GF}, satisfies the $\ell^2$-inverse continuity property in form of Definition~\ref{def:ICP} if $\CE$ is $\Lambda$-convex and the parameter $\tau$ sufficiently small.
As we discuss in Remark~\ref{rem:ICPifconvex} below, the condition on the parameter~$\tau$ vanishes if $\CE$ is convex, i.e., $\Lambda\geq0$.

\begin{lemma}[{$\widetilde\CE_{k}$ satisfies the $\ell^2$-inverse continuity property}]
\label{lem:widetildeCEhasICP}
	Let $\widetilde\CE_{k}$ be defined as above with $\tau>0$ and with $\CE\in\CC(\bbR^d)$ satisfying \ref{asm:lambda-convex}.
	Moreover, if $\Lambda<0$, assume further that $\tau<1/(-\Lambda)$.
	Then, $\widetilde\CE_{k}$ satisfies the $\ell^2$-inverse continuity property~\eqref{eq:ICP} with parameters
		\begin{align*}
			\nu = \frac{1}{2}
			\quad\text{ and }\quad
			\eta = \sqrt{\frac{1}{2\tau} + \frac{\Lambda}{2}}.
		\end{align*}
		I.e., denoting the unique global minimizer of  $\widetilde\CE_{k}$ by $\widetilde{x}^{\mathrm{CH}}_{k}$, it holds
		\begin{align}
			\label{eq:ICP_widetildeCE}
			\N{x-\widetilde{x}^{\mathrm{CH}}_{k}}_2
			&\leq \frac{1}{\eta}\left(\widetilde\CE_{k}(x)-\widetilde\CE_{k}(\widetilde{x}^{\mathrm{CH}}_{k})\right)^{\nu} \quad \text{ for all } x \in \bbR^d.
		\end{align}	
\end{lemma}	
	
\begin{proof}
	We first notice that $\widetilde\CE_{k}$ is $2\eta^2\!=\!\left(\frac{1+\Lambda\tau}{\tau}\right)$-strongly convex ($2\eta^2>0$ by assumption), since
	\begin{align*}
		\widetilde\CE_{k}(x) - \frac{1}{2}\left(\frac{1+\Lambda\tau}{\tau}\right)\N{x}_2^2
		&= \frac{1}{2\tau} \left(\N{x^{\mathrm{CH}}_{k-1} - x}_2^2 - \N{x}_2^2\right) + \CE(x) - \frac{\Lambda}{2}\N{x}_2^2 \\
		&= \underbrace{\frac{1}{2\tau} \left(\N{x^{\mathrm{CH}}_{k-1}}_2^2 - 2\left\langle x^{\mathrm{CH}}_{k-1}, x\right\rangle\right)}_{\text{convex since linear}} + \underbrace{\CE(x) - \frac{\Lambda}{2}\N{x}_2^2}_{\text{convex by  \ref{asm:lambda-convex}}}
	\end{align*}
	is convex by being the sum of two convex functions.
	By strong convexity of $\widetilde\CE_{k}$, $\widetilde{x}^{\mathrm{CH}}_{k}$ exists, is unique and for all $\xi\in[0,1]$ it holds
	\begin{align*}
		\frac{1}{2}\left(\frac{1+\Lambda\tau}{\tau}\right)\xi(1-\xi)\N{x-\widetilde{x}^{\mathrm{CH}}_{k}}_2^2
		&\leq \xi\widetilde\CE_{k}(x) + (1-\xi)\widetilde\CE_{k}(\widetilde{x}^{\mathrm{CH}}_{k}) - \widetilde\CE_{k}(\xi x + (1-\xi)\widetilde{x}^{\mathrm{CH}}_{k}) \\
		&\leq \xi\left(\widetilde\CE_{k}(x) - \widetilde\CE_{k}(\widetilde{x}^{\mathrm{CH}}_{k})\right)\!,
	\end{align*}
	where we used in the last inequality that $\widetilde{x}^{\mathrm{CH}}_{k}$ minimizes $\widetilde\CE_{k}$.
	Dividing both sides by $\xi$, letting $\xi\rightarrow0$ and reordering the inequality gives the result.	
\end{proof}

\begin{remark}
	\label{rem:ICPifconvex}
	In the case that $\CE$ is $\Lambda$-convex with $\Lambda<0$ (i.e., potentially nonconvex), Lemma~\ref{lem:widetildeCEhasICP} requires that the parameter~$\tau$ is sufficiently small, in order to ensure that $\widetilde\CE_{k}$ is strongly convex and therefore has a unique global minimizer~$\widetilde{x}^{\mathrm{CH}}_{k}$.
	On the other hand, if $\CE$ is convex, i.e., $\Lambda\geq0$, $\widetilde\CE_{k}$ is strongly convex and therefore such constraint is not necessary, i.e., $\tau$ can be chosen arbitrarily.
\end{remark}

Next, let us provide rather technical estimates on the quantities~$(\widetilde\CE_{k})_r$, $\nu_k\big(B_r(\widetilde{x}^{\mathrm{CH}}_{k})\big)$ and $\int \N{x-\widetilde{x}^{\mathrm{CH}}_{k}}_2 d\nu_k(x)$, which appear when applying Proposition~\ref{prop:LaplacePrinciple} in the setting of the function~$\widetilde\CE_k$ and the probability measure~$\nu_k = \CN\!\left(x^{\mathrm{CH}}_{k-1},2\widetilde{\sigma}^2\Id\right)$.
This allows to keep the proof of Proposition~\ref{prop:relaxation_CH_GF} more concise.

\begin{lemma}
	\label{lem:auxiliary_estimates_Laplace}
	Let $\widetilde\CE_k\in\CC(\bbR^d)$ be as defined above with $\CE\in\CC(\bbR^d)$ satisfying \ref{asm:local_Lipschitz_quadratic_growth}.
	Then for the expressions~$(\widetilde\CE_{k})_r$, $\nu_k\big(B_r(\widetilde{x}^{\mathrm{CH}}_{k})\big)$ and $\int \N{x-\widetilde{x}^{\mathrm{CH}}_{k}}_2 d\nu_k(x)$ appearing in Equation~\eqref{eq:prop:LaplacePrinciple} the following estimates hold.
	Namely,
%
%
	\begin{align*}
		(\widetilde\CE_{k})_r
		&\leq \left(\frac{1}{2\tau} \Big(r + 4\tau C_1\left(\N{x^{\mathrm{CH}}_{k-1}}_2 + \Nbig{\widetilde{x}^{\mathrm{CH}}_{k}}_2\right)\!\Big) + C_1 \left(1+r+2\Nbig{\widetilde{x}^{\mathrm{CH}}_{k}}_2\right) \right)r, \\
		\nu_k\big(B_r(\widetilde{x}^{\mathrm{CH}}_{k})\big)
		&\geq \frac{1}{(2\widetilde\sigma)^{d}} \exp\left(-\frac{1}{2\widetilde\sigma^2}\left(r^2\!+\!12\tau^2 C_1^2\!\left(1\!+\!\N{x^{\mathrm{CH}}_{k-1}}_2^2 \!+\! \Nbig{\widetilde{x}^{\mathrm{CH}}_{k}}_2^2\right)\right)\!\right)\! \frac{r^d}{\Gamma\left(\frac{d}{2}\!+\!1\right)}, \\
		\int \N{x-\widetilde{x}^{\mathrm{CH}}_{k}}_2 d\nu_k(x)
		&\leq 2\tau C_1\left(1+\N{x^{\mathrm{CH}}_{k-1}}_2 + \Nbig{\widetilde{x}^{\mathrm{CH}}_{k}}_2\right) + \sqrt{2d}\widetilde{\sigma}.
	\end{align*}
\end{lemma}

\begin{proof}
	We investigate the expressions~$(\widetilde\CE_{k})_r$, $\nu_k\big(B_r(\widetilde{x}^{\mathrm{CH}}_{k})\big)$ and $\int \N{x-\widetilde{x}^{\mathrm{CH}}_{k}}_2 d\nu_k(x)$ individually.
	
	\textbf{Term~$(\widetilde\CE_{k})_r$:}
	By definition (see Proposition~\ref{prop:LaplacePrinciple}) and under \ref{asm:local_Lipschitz_quadratic_growth} it holds
	\begin{align*}
		(\widetilde\CE_{k})_r
		&= \sup_{x\in B_r(\widetilde{x}^{\mathrm{CH}}_{k})} \widetilde\CE_{k}(x) - \widetilde\CE_{k}(\widetilde{x}^{\mathrm{CH}}_{k}) \\
		&\leq \frac{1}{2\tau} \sup_{x\in B_r(\widetilde{x}^{\mathrm{CH}}_{k})} \left(\N{x^{\mathrm{CH}}_{k-1} - x}_2^2 - \N{x^{\mathrm{CH}}_{k-1} - \widetilde{x}^{\mathrm{CH}}_{k}}_2^2\right) + \sup_{x\in B_r(\widetilde{x}^{\mathrm{CH}}_{k})} \CE(x) - \CE(\widetilde{x}^{\mathrm{CH}}_{k}) \\
		&\leq \frac{1}{2\tau} \left(r + 2\N{x^{\mathrm{CH}}_{k-1} - \widetilde{x}^{\mathrm{CH}}_{k}}_2\right)r + C_1 \left(1+r+2\N{\widetilde{x}^{\mathrm{CH}}_{k}}_2\right) r \\
		&\leq \left(\frac{1}{2\tau} \left(r + 2\N{x^{\mathrm{CH}}_{k-1} - \widetilde{x}^{\mathrm{CH}}_{k}}_2\right) + C_1 \left(1+r+2\N{\widetilde{x}^{\mathrm{CH}}_{k}}_2\right) \right)r.
	\end{align*}

	\textbf{Term~$\nu_k\big(B_r(\widetilde{x}^{\mathrm{CH}}_{k})\big)$:}
	Using the density of the multivariate normal distribution~$\nu_k = \CN\!\left(x^{\mathrm{CH}}_{k-1},2\widetilde{\sigma}^2\Id\right)$ we can directly compute
	\begin{align*}
		\nu_k\big(B_r(\widetilde{x}^{\mathrm{CH}}_{k})\big)
		&= \frac{1}{(4\pi\widetilde\sigma^2)^{d/2}} \int_{B_r(\widetilde{x}^{\mathrm{CH}}_{k})} \exp\left(-\frac{1}{4\widetilde\sigma^2}\N{x-x^{\mathrm{CH}}_{k-1}}_2^2\right) d\lambda(x) \\
		&\geq \frac{1}{(4\pi\widetilde\sigma^2)^{d/2}} \int_{B_r(\widetilde{x}^{\mathrm{CH}}_{k})} \exp\left(-\frac{1}{2\widetilde\sigma^2}\left(\N{x-\widetilde{x}^{\mathrm{CH}}_{k}}_2^2+\N{\widetilde{x}^{\mathrm{CH}}_{k}-x^{\mathrm{CH}}_{k-1}}_2^2\right)\right) d\lambda(x) \\
		&\geq \frac{1}{(4\pi\widetilde\sigma^2)^{d/2}} \exp\left(-\frac{1}{2\widetilde\sigma^2}\left(r^2+\N{\widetilde{x}^{\mathrm{CH}}_{k}-x^{\mathrm{CH}}_{k-1}}_2^2\right)\right) \int_{B_r(\widetilde{x}^{\mathrm{CH}}_{k})} d\lambda(x) \\
		&= \frac{1}{(2\widetilde\sigma)^{d}} \exp\left(-\frac{1}{2\widetilde\sigma^2}\left(r^2+\N{\widetilde{x}^{\mathrm{CH}}_{k}-x^{\mathrm{CH}}_{k-1}}_2^2\right)\right) \frac{1}{\Gamma\left(\frac{d}{2}+1\right)}r^d,
	\end{align*}
	where we used in the last step that the volume of a $d$-dimensional unit ball is $\pi^{d/2}/\Gamma\left(\frac{d}{2}+1\right)$.
	Here, $\Gamma$ denotes Euler's gamma function. 
	We recall for the readers' convenience that by Stirling's approximation $\Gamma\left(x+1\right) \sim \sqrt{2\pi x}\left(x/e\right)^x$ as $x\rightarrow\infty$.
	
	\textbf{Term~$\int \N{x-\widetilde{x}^{\mathrm{CH}}_{k}}_2 d\nu_k(x)$:}
	A straightforward computation gives
	\begin{align*}
		\int \N{x-\widetilde{x}^{\mathrm{CH}}_{k}}_2 d\nu_k(x)
		&= \int \N{x - \widetilde{x}^{\mathrm{CH}}_{k}}_2 d\CN\!\left(x^{\mathrm{CH}}_{k-1},2\widetilde{\sigma}^2\Id\right)(x) \\
		&= \int \N{x + x^{\mathrm{CH}}_{k-1} - \widetilde{x}^{\mathrm{CH}}_{k}}_2 d\CN\!\left(0,2\widetilde{\sigma}^2\Id\right)(x) \\
		&\leq \N{x^{\mathrm{CH}}_{k-1}-\widetilde{x}^{\mathrm{CH}}_{k}}_2 + \int \N{x}_2 d\CN\!\left(0,2\widetilde{\sigma}^2\Id\right)(x) \\
		&\leq \N{x^{\mathrm{CH}}_{k-1}-\widetilde{x}^{\mathrm{CH}}_{k}}_2 + \sqrt{2d}\widetilde{\sigma}.
	\end{align*}
	
	\textbf{Concluding step:}
	To conclude the proof, we further observe that since $\widetilde{x}^{\mathrm{CH}}_{k}$ is the minimizer of $\widetilde\CE_{k}$, see \eqref{eq:CH_aux}, a comparison with $x^{\mathrm{CH}}_{k-1}$ yields
	\begin{align*}
		\frac{1}{2\tau} \N{x^{\mathrm{CH}}_{k-1} - \widetilde{x}^{\mathrm{CH}}_{k}}_2^2 + \CE(\widetilde{x}^{\mathrm{CH}}_{k})
		\leq \CE(x^{\mathrm{CH}}_{k-1}).
	\end{align*}
	With \ref{asm:local_Lipschitz_quadratic_growth} it therefore holds
	\begin{align*}
		\N{x^{\mathrm{CH}}_{k-1} - \widetilde{x}^{\mathrm{CH}}_{k}}_2^2
		&\leq 2\tau\left(\CE(x^{\mathrm{CH}}_{k-1}) - \CE(\widetilde{x}^{\mathrm{CH}}_{k})\right) 
		\leq 2\tau C_1\left(1+\N{x^{\mathrm{CH}}_{k-1}}_2 + \Nbig{\widetilde{x}^{\mathrm{CH}}_{k}}_2\right)\N{x^{\mathrm{CH}}_{k-1}-\widetilde{x}^{\mathrm{CH}}_{k}}_2,
	\end{align*}
	or rephrased
	\begin{align*}
		\N{x^{\mathrm{CH}}_{k-1} - \widetilde{x}^{\mathrm{CH}}_{k}}_2
		&\leq 2\tau C_1\left(1+\N{x^{\mathrm{CH}}_{k-1}}_2 + \Nbig{\widetilde{x}^{\mathrm{CH}}_{k}}_2\right).
	\end{align*}
	Exploiting this estimate in the former bounds, gives the statements.
\end{proof}

\subsection{Proof of Theorem~\ref{thm:relaxation_CH_GF}}
	\label{subsec:appendix:proof:thm:relaxation_CH_GF}
	
We now have all necessary tools at hand to present the detailed proof of Theorem~\ref{thm:relaxation_CH_GF}.

\begin{proof}[Proof of Theorem~\ref{thm:relaxation_CH_GF}]
	We combine in what follows Proposition~\ref{prop:relaxation_CH_GF} with a stability argument for the MMS~\eqref{eq:MMS}.
    
    To obtain the probabilistic formulation of the statement, let us denote, as in the proof of Proposition~\ref{prop:relaxation_CH_GF}, the underlying probability space by $(\Omega,\CF,\bbP)$ (note that we can use the same probability space as in Section~\ref{sec:appendix:thm:relaxation_CBO_CH} since the stochasticity of the three schemes~\eqref{eq:CH}, \eqref{eq:CH_aux} and \eqref{eq:MMS} is solely coming from the initialization) and introduce the subset~$\widetilde\Omega_M$ of $\Omega$ of suitably bounded random variables according to
	\begin{align*}
		\widetilde\Omega_M
		:= \left\{\omega\in\Omega : \max_{k=0,\dots,K} \max\left\{\N{x^{\mathrm{CH}}_{k}}_2,\Nbig{\widetilde{x}^{\mathrm{CH}}_{k}}_2\right\} \leq M \right\}.
	\end{align*}	
	For the associated cutoff function (random variable) we write $\mathbbm{1}_{\widetilde\Omega_M}$.
	
	We can decompose the expected squared discrepancy $\bbE\N{x^{\mathrm{MMS}}_{k}-x^{\mathrm{CH}}_{k}}_2^2\mathbbm{1}_{\widetilde\Omega_M}$ between the MMS~\eqref{eq:MMS} and the CH scheme~\eqref{eq:CH} for any $\vartheta\in(0,1)$ as
	\begin{align} \label{eq:proof:thm:error_decomposition_MMS_CH:3}
		\bbE\N{x^{\mathrm{MMS}}_{k}-x^{\mathrm{CH}}_{k}}_2^2\mathbbm{1}_{\widetilde\Omega_M}
		\leq
		(1+\vartheta)\,\bbE\N{x^{\mathrm{MMS}}_{k}-\widetilde{x}^{\mathrm{CH}}_{k}}_2^2\mathbbm{1}_{\widetilde\Omega_M} + (1+\vartheta^{-1})\,\bbE\N{\widetilde{x}^{\mathrm{CH}}_{k}-x^{\mathrm{CH}}_{k}}_2^2\mathbbm{1}_{\widetilde\Omega_M}.
	\end{align}
	In what follows we individually estimate the two terms on the right-hand side of~\eqref{eq:proof:thm:error_decomposition_MMS_CH:3}.
	
	\textbf{First term:}
	Let us first bound the term $\bbE\N{x^{\mathrm{MMS}}_{k}-\widetilde{x}^{\mathrm{CH}}_{k}}_2^2\mathbbm{1}_{\widetilde\Omega_M}$.
	By definition of $x^{\mathrm{MMS}}_{k}$ and $\widetilde{x}^{\mathrm{CH}}_{k}$ as minimizers of \eqref{eq:MMS} and \eqref{eq:CH_aux}, respectively,
	and with the definition $\CE_\Lambda(x) := \CE(x) - \frac{\Lambda}{2}\!\N{x}_2^2$
	it holds
	\begin{align*}
		\frac{(1+\tau\Lambda)x^{\mathrm{MMS}}_{k} - x^{\mathrm{MMS}}_{k-1}}{\tau} \in - \partial\CE_\Lambda(x^{\mathrm{MMS}}_{k})
		\quad\text{ and }\quad
		\frac{(1+\tau\Lambda)\widetilde{x}^{\mathrm{CH}}_{k} - x^\mathrm{CH}_{k-1}}{\tau} \in - \partial\CE_\Lambda(\widetilde{x}^{\mathrm{CH}}_{k}).
	\end{align*}
	Since $\CE_\Lambda$ is convex due to \ref{asm:lambda-convex} and as consequence of the properties of the subdifferential we have
	\begin{align*}
		\left\langle -\frac{(1+\tau\Lambda)x^{\mathrm{MMS}}_{k} - x^{\mathrm{MMS}}_{k-1}}{\tau} + \frac{(1+\tau\Lambda)\widetilde{x}^{\mathrm{CH}}_{k} - x^\mathrm{CH}_{k-1}}{\tau}, x^{\mathrm{MMS}}_{k}-\widetilde{x}^{\mathrm{CH}}_{k}\right\rangle \geq 0,
	\end{align*}
	which allows to obtain by means of Cauchy-Schwarz inequality
	\begin{align*}
		(1+\tau\Lambda)\N{x^{\mathrm{MMS}}_{k}\!-\!\widetilde{x}^{\mathrm{CH}}_{k}}_2^2
		&\leq \left\langle x^{\mathrm{MMS}}_{k-1}\!-\!x^\mathrm{CH}_{k-1},x^{\mathrm{MMS}}_{k}\!-\!\widetilde{x}^{\mathrm{CH}}_{k}\right\rangle 
		\leq \N{x^{\mathrm{MMS}}_{k-1}\!-\!x^\mathrm{CH}_{k-1}}_2 \N{x^{\mathrm{MMS}}_{k}\!-\!\widetilde{x}^{\mathrm{CH}}_{k}}_2
	\end{align*}
	or, equivalently,
	\begin{align} \label{eq:proof:thm:error_decomposition_MMS_CH:9}
		\N{x^{\mathrm{MMS}}_{k}-\widetilde{x}^{\mathrm{CH}}_{k}}_2
		\leq \frac{1}{1+\tau\Lambda} \N{x^{\mathrm{MMS}}_{k-1}-x^\mathrm{CH}_{k-1}}_2.
	\end{align}
	
	\textbf{Second term:}
	For the term~$\bbE\N{\widetilde{x}^{\mathrm{CH}}_{k}-x^{\mathrm{CH}}_{k}}_2^2\mathbbm{1}_{\widetilde\Omega_M}$ we obtained in \eqref{eq:proof:thm:error_decomposition_MMS_CH:33} in the proof of Proposition~\ref{prop:relaxation_CH_GF}, for suitable choices of $\widetilde\sigma$ and $\alpha$, the bound
     \begin{align}
		\label{eq:proof:thm:error_decomposition_MMS_CH:33_aux}
	\begin{aligned}
		\bbE\N{x^{\mathrm{CH}}_{k} - \widetilde{x}^{\mathrm{CH}}_{k}}_2^2\mathbbm{1}_{\widetilde\Omega_M}
        &\leq c\tau^2
	\end{aligned}
	\end{align}
    with a constant $c=c(C_1,M)$.
	
	\textbf{Concluding step:}
	Combining this with the estimate~\eqref{eq:proof:thm:error_decomposition_MMS_CH:9} yields for \eqref{eq:proof:thm:error_decomposition_MMS_CH:3} the bound
	\begin{align}
	\begin{aligned}
		\bbE\N{x^{\mathrm{MMS}}_{k}-x^{\mathrm{CH}}_{k}}_2^2\mathbbm{1}_{\widetilde\Omega_M}
		&\leq
		\frac{1+\vartheta}{\left(1+\tau\Lambda\right)^2} \bbE\N{x^{\mathrm{MMS}}_{k-1}-x^\mathrm{CH}_{k-1}}_2^2 \mathbbm{1}_{\widetilde\Omega_M} + c(1+\vartheta^{-1})\,\tau^2.
	\end{aligned}
	\end{align}
	An application of the discrete variant of Gr\"onwall's inequality~\eqref{eq:gronwall_inequality} shows that
	\begin{align}
		\label{eq:proof:thm:error_decomposition_MMS_CH:35}
	\begin{aligned}
		\bbE\N{x^{\mathrm{MMS}}_{k}-x^{\mathrm{CH}}_{k}}_2^2\mathbbm{1}_{\widetilde\Omega_M}
		&\leq
		c(1+\vartheta^{-1})\,\tau^2 \sum_{\ell=0}^{k-1} \left(\frac{1+\vartheta}{\left(1+\tau\Lambda\right)^2}\right)^\ell
	\end{aligned}
	\end{align}
	for all $k=1,\dots,K$, where we used that both schemes are initialized by the same~$x_0$.
	
	\textbf{Probabilistic formulation:}
	We first note that with Markov's inequality we have the estimate
	\begin{align*}
	\begin{aligned}
		\bbP\big(\widetilde\Omega_M^c\big)
		&= \bbP\left(\max_{k=0,\dots,K} \max\left\{\N{x^{\mathrm{CH}}_{k}}_2,\Nbig{\widetilde{x}^{\mathrm{CH}}_{k}}_2\right\}   > M \right) \\
		&\leq \frac{1}{M^4} \left(\bbE\max_{k=0,\dots,K} \N{x^{\mathrm{CH}}_{k}}_2^4+\bbE\max_{k=0,\dots,K}\Nbig{\widetilde{x}^{\mathrm{CH}}_{k}}_2^4\right)\\ 
		&\leq \frac{1}{M^4} \big(\CM^{\mathrm{CH}}+\widetilde{\CM}^{\mathrm{CH}}\;\!\big),
	\end{aligned}
	\end{align*}
	where the last inequality is due to Lemmas~\ref{lem:boundedness_CH} and~\ref{lem:boundedness_auxiliaryMMS}.
	Thus, for any $\delta\in(0,1/2)$, a sufficiently large choice $M = M(\delta^{-1},\CM^{\mathrm{CH}},\widetilde{\CM}^{\mathrm{CH}})$ allows to ensure $\bbP\big(\widetilde\Omega_M^c\big)\leq \delta$.
	To conclude the proof, let us denote by $\widetilde K_\varepsilon\subset\Omega$ the set, where~\eqref{eq:bound_CH_GF} does not hold and abbreviate
	\begin{align*}
		\epsilon
		= \varepsilon^{-1} c(1+\vartheta^{-1})\,\tau^2 \sum_{\ell=0}^{k-1} \left(\frac{1+\vartheta}{\left(1+\tau\Lambda\right)^2}\right)^\ell.
	\end{align*}
	For the probability of this set we can estimate
	\begin{align}
	\begin{aligned}
		\bbP\big(\widetilde K_\varepsilon\big)
		&= \bbP\big(\widetilde K_\varepsilon\cap\widetilde\Omega_M\big) + \bbP\big(\widetilde K_\varepsilon\cap\widetilde\Omega_M^c\big)
		\leq \bbP\big(\widetilde K_\varepsilon\,\big|\,\widetilde\Omega_M\big)\,\bbP\big(\widetilde\Omega_M\big) + \bbP\big(\widetilde\Omega_M^c\big) \\
		&\leq \bbP\big(\widetilde K_\varepsilon\,\big|\,\widetilde\Omega_M\big) + \delta
		\leq \epsilon^{-1}\,\bbE\left[\N{x^{\mathrm{MMS}}_{k}-x^{\mathrm{CH}}_{k}}_2^2 \,\Big|\, \widetilde\Omega_M\right] + \delta,
	\end{aligned}
	\end{align}
	where the last step is due to Markov's inequality.
    By definition of the conditional expectation we further have
	\begin{align*}
	\begin{aligned}
		\bbE\left[\N{x^{\mathrm{MMS}}_{k}-x^{\mathrm{CH}}_{k}}_2^2 \,\Big|\, \widetilde\Omega_M\right]
		\leq \frac{1}{\bbP\big(\widetilde\Omega_M\big)} \bbE\N{x^{\mathrm{MMS}}_{k}-x^{\mathrm{CH}}_{k}}_2^2\mathbbm{1}_{\widetilde\Omega_M}
		\leq 2\bbE\N{x^{\mathrm{MMS}}_{k}-x^{\mathrm{CH}}_{k}}_2^2\mathbbm{1}_{\widetilde\Omega_M}.
	\end{aligned}
	\end{align*}
	Inserting now the expression from \eqref{eq:proof:thm:error_decomposition_MMS_CH:35} concludes the proof.
\end{proof}

\newpage
\section{Additional numerical experiments} \label{sec:appendix:additional_numerics}

\subsection{Comparison of the CH scheme~\eqref{eq:CH} for different sampling widths~$\widetilde\sigma$} 

To complement Figure~\ref{fig:CH_GrandCanyon3noisy}, we visualize in Figure~\ref{fig:comparison_CH} the influence of the sampling width~$\widetilde\sigma$ on the behavior of the CH scheme~\eqref{eq:CH}.

\begin{figure}[ht]
	\centering
	\subcaptionbox{The CH scheme~\eqref{eq:CH} with sampling width~$\widetilde\sigma=0.4$ gets stuck in a local minimum of $\CE$. \label{fig:comparison_CH_1}}{\includegraphics[trim=28 209 31 200,clip,width=0.29\textwidth]{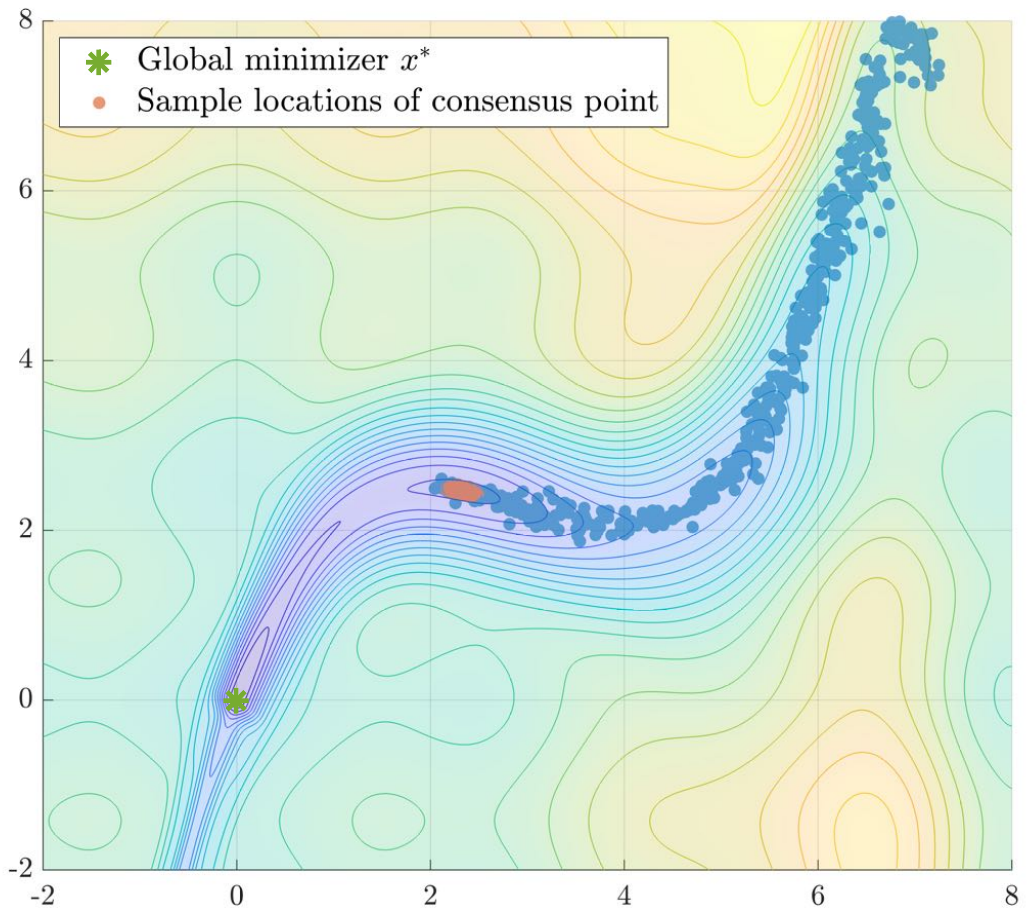}}
	\hspace{0.8em}
	\subcaptionbox{The CH scheme~\eqref{eq:CH} with sampling width~$\widetilde\sigma=0.6$ can occasionally escape local minima of $\CE$. \label{fig:comparison_CH_2}}{\includegraphics[trim=28 209 31 200,clip,width=0.29\textwidth]{img/low_resolution/CH/CH_GrandCanyon3noisy_sigma0_6}}
	\hspace{0.8em}
	\subcaptionbox{The CH scheme~\eqref{eq:CH} with sampling width~$\widetilde\sigma=0.7$ can escape local minima of $\CE$. \label{fig:comparison_CH_3}}{\includegraphics[trim=28 209 31 200,clip,width=0.29\textwidth]{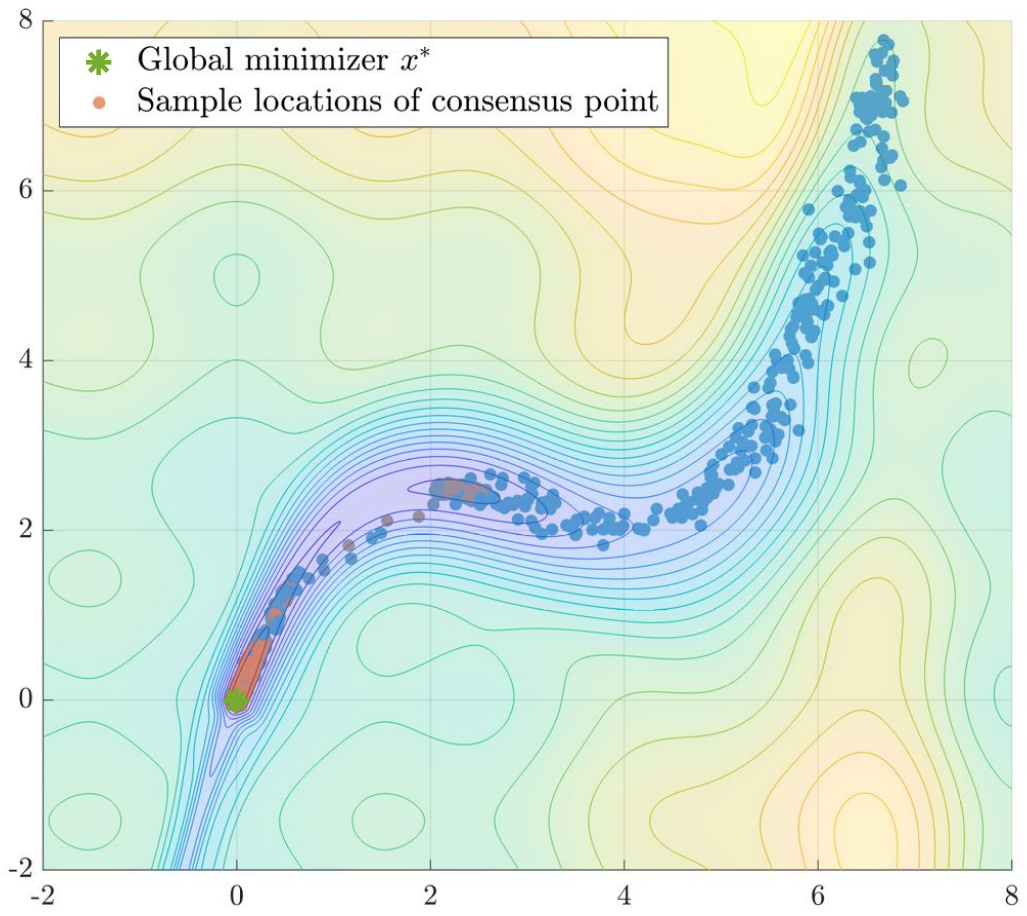}}
	\caption{A visual comparison of the CH scheme~\eqref{eq:CH} for different sampling widths~$\widetilde\sigma$.
	We depict the positions of the consensus hopping scheme~\eqref{eq:CH} for different values of $\widetilde\sigma$ ($0.4$ in (a), $0.6$ in (b) and $0.7$ in (c)) in the setting of Figure~\ref{fig:CH_GrandCanyon3noisy}.
	While for small $\widetilde\sigma$ the numerical scheme gets stuck in a local minimum of the objective, the ability to escape such critical points improves with larger $\widetilde\sigma$.
	Notice that (b) coincides with Figure~\ref{fig:CH_GrandCanyon3noisy}.}
	\label{fig:comparison_CH}
\end{figure}

\newpage
\subsection{The numerical experiments of Figures~\ref{fig:intuitionGiAyN} and~\ref{fig:comparison_algorithms} for a different objective} 

In the style of Figures~\ref{fig:intuitionGiAyN} and~\ref{fig:comparison_algorithms} we provide in Figure~\ref{fig:additional_experiments} an additional set of illustrations of the behavior of the different algorithms analyzed in this work for a noisy Canyon function with a valley shaped as a second degree polynomial.
\begin{figure}[h!]
	\centering
	\subcaptionbox{A noisy Canyon function~$\CE$ with a valley shaped as a second degree polynomial~\label{fig:GrandCanyon2noisy}}{\includegraphics[trim=91 251 79 250,clip,width=0.45\textwidth]{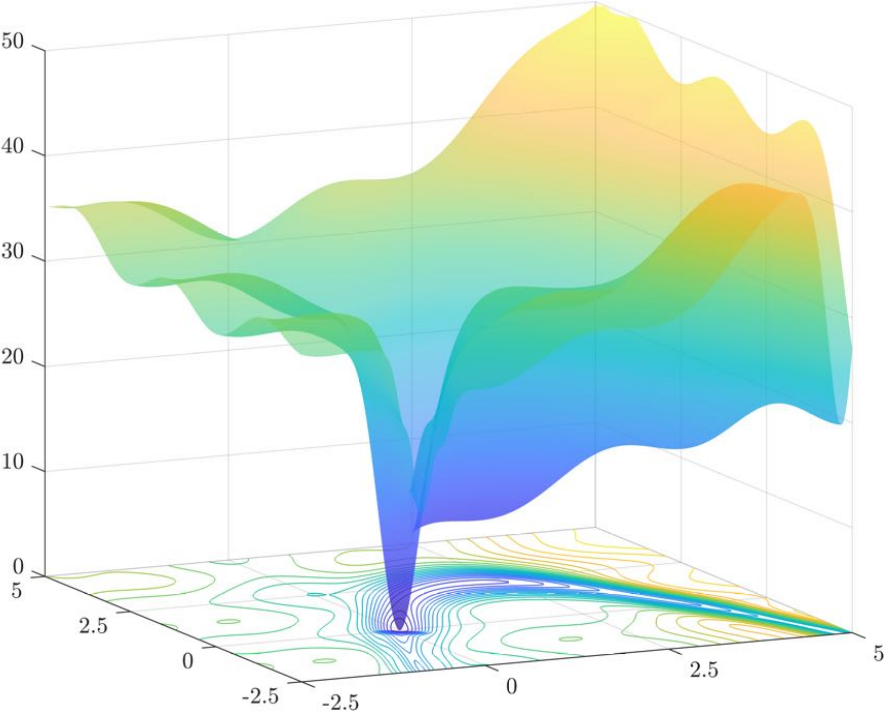}}
	\hspace{1.5em}
	\subcaptionbox{The CBO scheme~\eqref{eq:CBO} (sampled over several runs) follows on average the valley while passing over local minima. \label{fig:CBO_GrandCanyon2noisy}}{\includegraphics[trim=28 209 31 200,clip,width=0.45\textwidth]{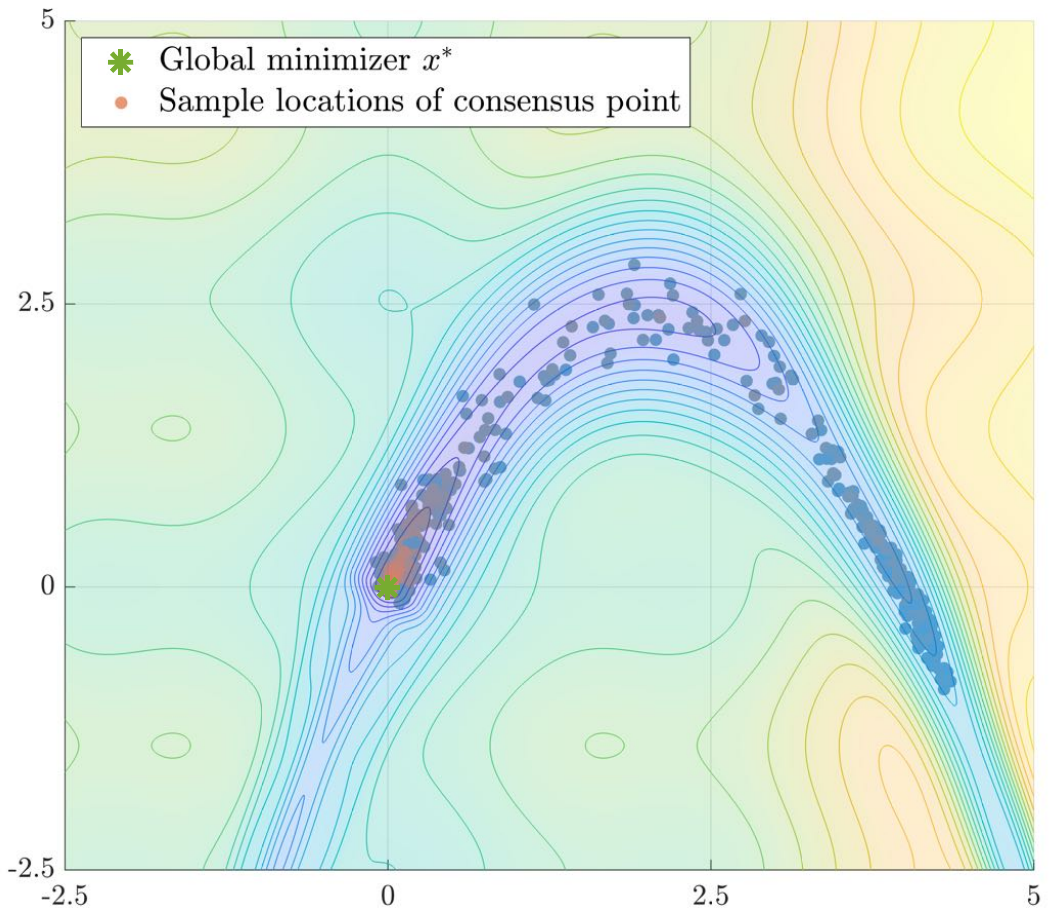}}
	\\ \vspace{1em}
	\subcaptionbox{The CH scheme~\eqref{eq:CH} (sampled over several runs) follows on average the valley of $\CE$ and can occasionally escape local minima. \label{fig:CH_GrandCanyon2noisy}}{\includegraphics[trim=28 209 31 200,clip,width=0.29\textwidth]{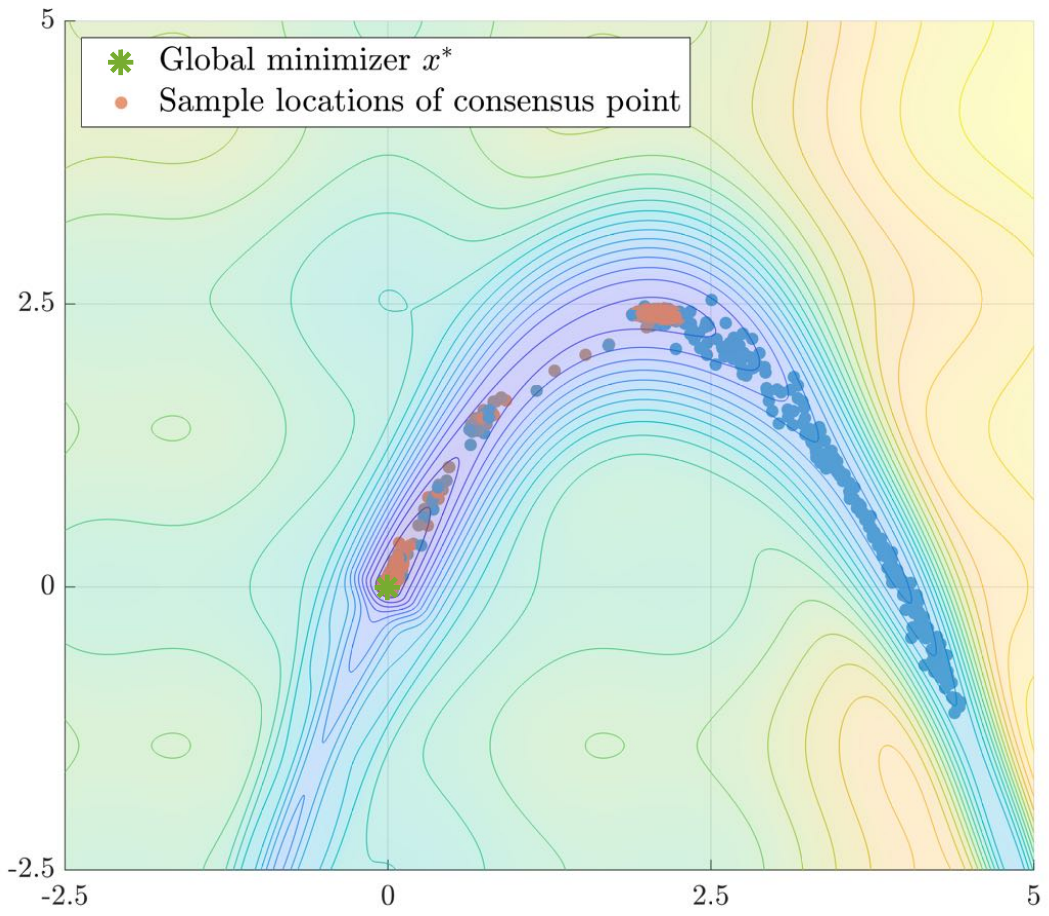}}
	\hspace{0.75em}
	\subcaptionbox{GD gets stuck in a local minimum of $\CE$. \label{fig:GD_GrandCanyon2noisy}}{\includegraphics[trim=28 209 31 200,clip,width=0.29\textwidth]{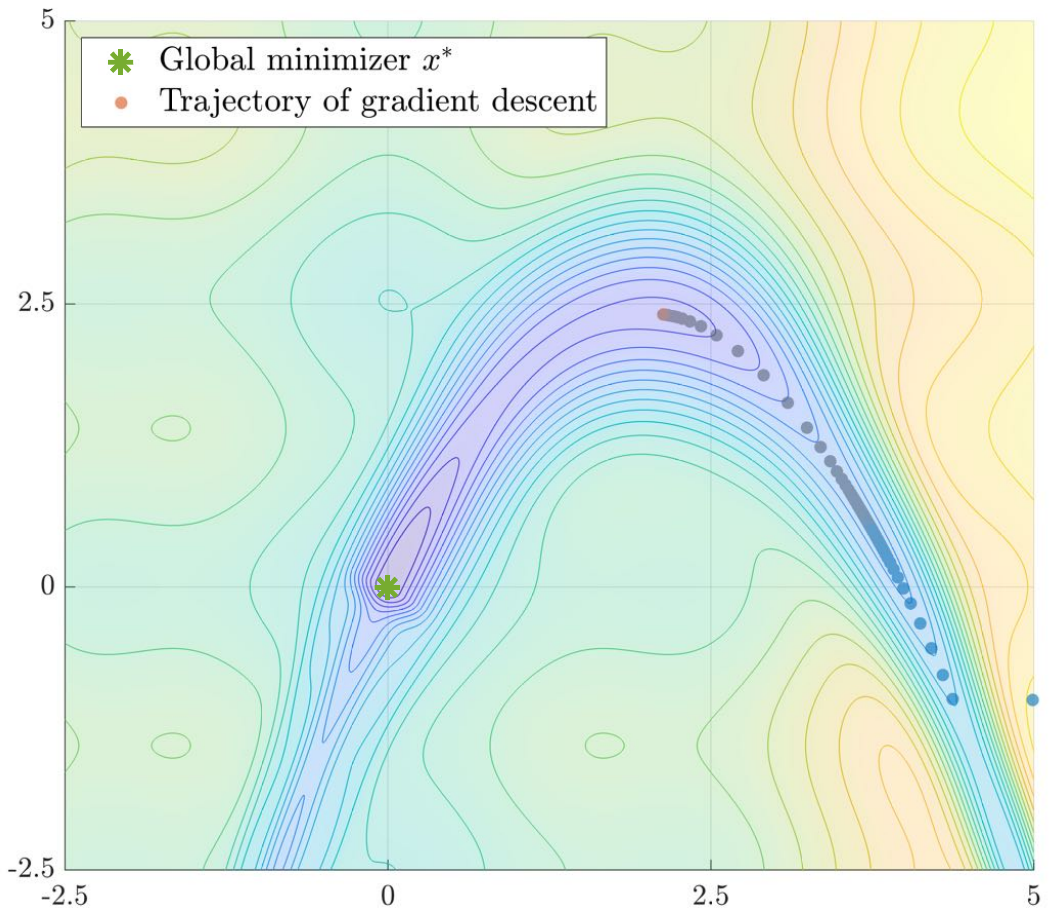}}
	\hspace{0.75em}
	\subcaptionbox{The Langevin dynamics (sampled over several runs) follows on average the valley of $\CE$ and escapes local minima. \label{fig:Langevin_GrandCanyon2noisy}}{\includegraphics[trim=28 209 31 200,clip,width=0.29\textwidth]{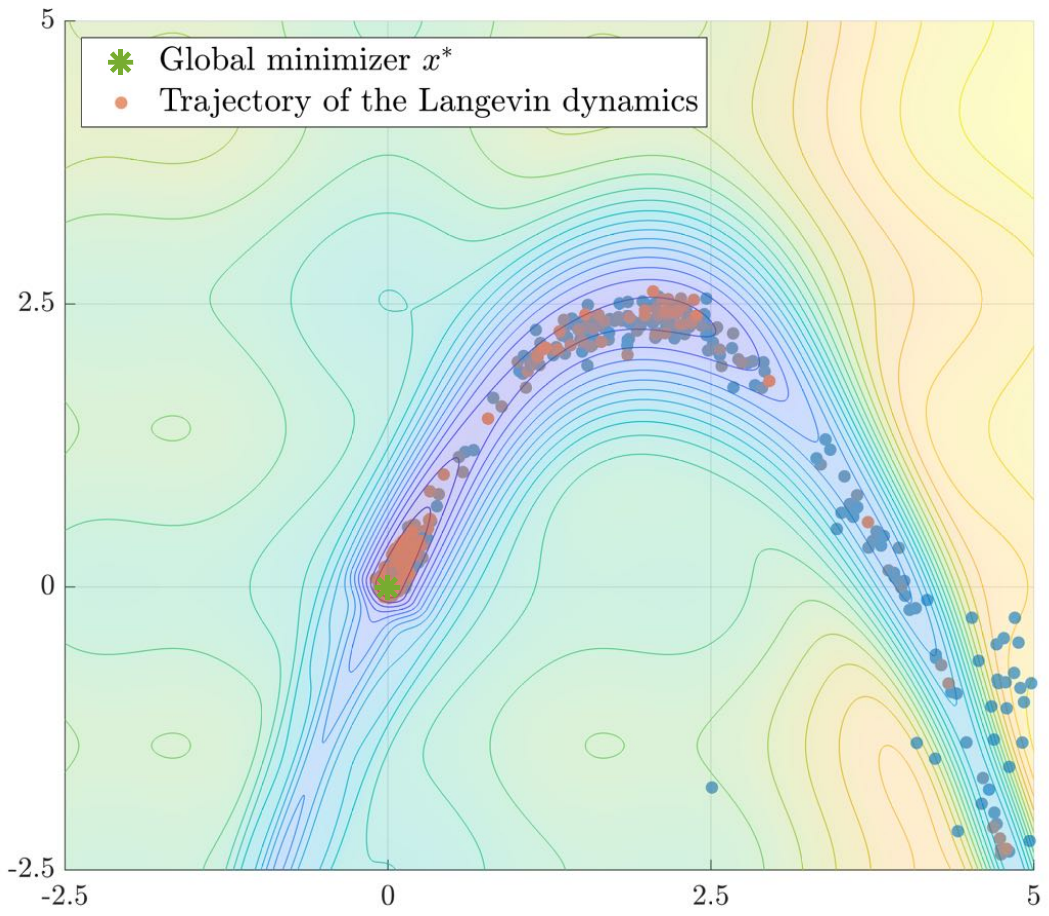}}
	\caption{An additional numerical experiment illustrating the behavior of the CBO scheme~\eqref{eq:CBO} (see~(b)), the consensus hopping scheme~\eqref{eq:CH} (see~(c)), GD (see~(d)) and the overdamped Langevin dynamics (see~(e)) in search of the global minimizer~$\globmin$ of the nonconvex objective function~$\CE$ depicted in (a).
	The experimental setting is the one of Figures~\ref{fig:intuitionGiAyN} and~\ref{fig:comparison_algorithms} with the only difference of the particles being initialized around~$(5,-1)$.}
	\label{fig:additional_experiments}
\end{figure}


\end{document}